\documentclass[acmsmall,nonacm]{acmart}

\usepackage{amsmath}
\usepackage{algorithm}
\usepackage{algpseudocode}
\usepackage{booktabs}
\usepackage{multirow}
\usepackage{tabularx}
\usepackage{subcaption}
\usepackage{needspace}
\usepackage{wrapfig}
\usepackage[most]{tcolorbox}

\acmJournal{TKDD}
\setcopyright{none}
\acmDOI{}
\theoremstyle{acmplain}
\newtheorem{proposition}{Proposition}
\newtheorem{lemma}{Lemma}
\theoremstyle{acmdefinition}

\tcolorboxenvironment{proposition}{
  breakable,
  colback=blue!5,
  colframe=blue!50!black,
  boxrule=0.4pt,
  arc=3pt,
  left=3pt, right=3pt,
  top=2pt, bottom=2pt,
  before skip=6pt, after skip=9pt
}

\newcommand{\method}{ProtoGDA}
\begin{document}

\raggedbottom
\setlength{\emergencystretch}{1em}

\title[ProtoGDA: Target-Conditioned Graph Prototype Distillation for Graph Domain Adaptation]{\method{}: Target-Conditioned Graph Prototype Distillation for Graph Domain Adaptation}
\author{Yingxu Wang}
\authornotemark[1]
\email{yingxv.wang@gmail.com}
\affiliation{%
  \institution{The Chinese University of Hong Kong}
  \department{Department of Computer Science and Engineering}
  \city{Hong Kong}
  \country{China}}

\author{Kunyu Zhang}
\authornotemark[1]
\email{kuzh330@outlook.com}
\affiliation{%
  \institution{The Education University of Hong Kong}
  \department{Centre for Learning, Teaching and Technology and Department of Mathematics and Information Technology}
  \city{Hong Kong}
  \country{China}}

\author{Xinwang Liu}
\email{xinwangliu@nudt.edu.cn}
\affiliation{%
  \institution{National University of Defense Technology}
  \department{College of Computer}
  \city{Changsha}
  \postcode{410073}
  \country{China}}

\author{Chang Tang}
\email{tangchang@hust.edu.cn}
\affiliation{%
  \institution{Huazhong University of Science and Technology}
  \department{School of Software Engineering}
  \city{Wuhan}
  \country{China}}

\author{Mengzhu Wang}
\email{dreamkily@gmail.com}
\affiliation{%
  \institution{Hebei University of Technology}
  \department{School of Artificial Intelligence}
  \city{Tianjin}
  \country{China}}

\author{Joey Tianyi Zhou}
\email{zhouty@a-star.edu.sg}
\affiliation{%
  \institution{Centre for Frontier AI Research, Agency for Science, Technology and Research (A*STAR)}
  \country{Singapore}}
\affiliation{%
  \institution{Singapore Management University}
  \country{Singapore}}

\author{Nan Yin}
\authornotemark[2]
\email{yinnan8911@gmail.com}
\affiliation{%
  \institution{The Education University of Hong Kong}
  \department{Centre for Learning, Teaching and Technology and Department of Mathematics and Information Technology}
  \city{Hong Kong}
  \country{China}}

\renewcommand{\shortauthors}{Wang et al.}

\begin{abstract}
Graph domain adaptation (GDA) transfers a predictor from a labeled source domain to an unlabeled target domain under distribution shift. Existing GDA methods mainly adapt representations or message passing while treating the observed graphs as fixed inputs. However, under structural shift, these methods still train on source graphs whose connectivity differs from that of the target graphs, so adapting how the model processes them cannot change the structural patterns it learns from. In this paper, we propose Target-Conditioned Graph \textbf{Proto}type Distillation for \textbf{G}raph \textbf{D}omain \textbf{A}daptation (\method{}), a data-centric framework that makes graph structure an adaptation variable. \method{} parameterizes a compact set of graph prototypes with probabilistic adjacency matrices, learnable node features, and soft labels, and distills source-label supervision into them by differentiating the source loss of a proxy graph neural network (GNN) through its finite-step training on the prototypes. It conditions the same prototypes on unlabeled target graphs by matching the collection means of permutation-invariant geometric descriptors and Dirichlet energy. The optimized prototypes are then frozen to train an independently initialized GNN for target prediction. Extensive experiments on graph and image benchmarks under diverse domain shifts demonstrate the effectiveness of \method{}, with improvements over representative state-of-the-art methods in most settings.
\end{abstract}



\maketitle

{\renewcommand{\thefootnote}{\fnsymbol{footnote}}%
\footnotetext[1]{Both authors contributed equally to this research.}%
\footnotetext[2]{Corresponding author.}}

\section{Introduction}
\label{sec:intro}

Graph domain adaptation (GDA) transfers knowledge from a labeled source domain to an unlabeled target domain with a different graph distribution, enabling graph neural networks (GNNs) to generalize across domains~\cite{wu2020unsupervised,liu2024revisiting,cai2024graph,yao2023improving}. This setting arises when labeled graphs are available only for an existing collection while newly emerging collections, such as molecules from new chemical spaces or networks from new platforms, remain unlabeled and costly to annotate~\cite{hu2020open,qiao2023semi,zhang2026survey,zhao2025spot}.

Existing GDA methods transfer source knowledge in two main ways. Model-centric methods keep the source graphs fixed as training data and adapt the model to the target domain. They learn domain-invariant representations through adversarial training and statistical distribution matching~\cite{wu2020unsupervised,dai2022graph,wang2026sgac,wang2026nested}, and further exploit graph structure through edge reweighting, propagation refinement, smoothness regularization, and spectral alignment~\cite{liu2023structural,liu2024rethinking,liu2024pairwise,chen2025smoothness,di2025graph}. Data-centric methods instead modify the training data, either by generating a small transferable graph aligned with the target domain~\cite{huang2024can} or by condensing a large training set into a compact set of learned graphs~\cite{jin2022graph,jin2022condensing,wang2024self,xu2023kernel,wang2025protomol}. However, structural shifts change the connectivity of graphs across domains, and model-centric adaptation leaves this mismatch in the training data untouched~\cite{levie2021transferability}. Existing data-centric methods either adapt a single large graph or optimize the learned graphs to reproduce the source distribution. When the training data form a collection of variable-sized graphs without node correspondence to the target graphs, the learned graphs must move toward the target domain while remaining consistent with the source classes. This raises a key question of \textbf{\textit{how can a compact set of labeled graphs be optimized toward an unlabeled target domain while preserving their discriminative information?}}

\begin{wrapfigure}{r}{0.52\linewidth}
  \vspace{-12pt}
  \centering
  \includegraphics[width=\linewidth]{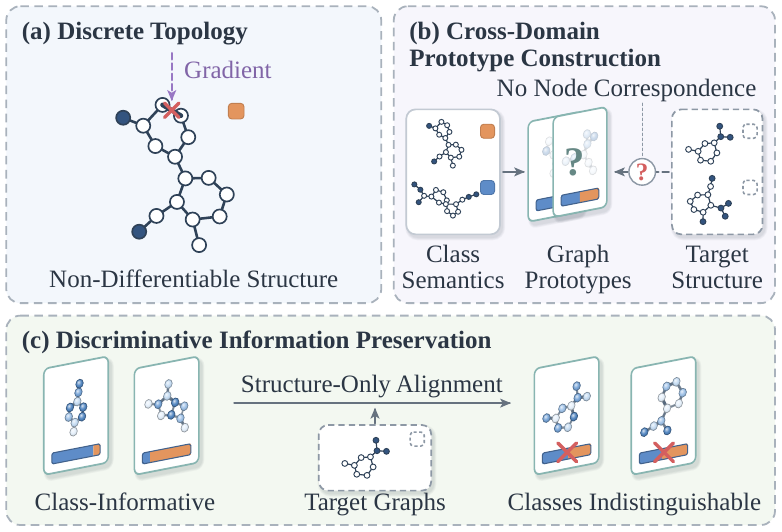}
  \vspace{-0.5cm}
  \caption{Challenges of adapting graph structure. (a) Discrete topology blocks gradient-based optimization. (b) Target structure offers no node correspondence. (c) Structure-only alignment blurs class boundaries.}
  \label{fig:challenge}
  \vspace{-8pt}
\end{wrapfigure}

In this paper, we aim to develop a GDA framework that adapts the training graphs themselves to the target domain while preserving their discriminative information. As shown in Fig.~\ref{fig:challenge}, achieving this objective involves three challenges.
\textbf{\textit{(1) How can the structure of training graphs be optimized when graph topology is discrete?}}
Graph structure determines how node features propagate and which substructures a GNN can recognize, so adapting the training graphs requires changing their topology. Edges are discrete, however, and the training objective provides no gradient with respect to their presence. Existing structural methods reweight observed edges or relax synthetic structures within a single domain~\cite{liu2023structural,jin2022condensing,wang2026disrfm,wang2026usbd}, and neither yields structures that respond to signals from another domain. Adapting the training graphs requires a differentiable structural parameterization that both source and target signals can update.
\textbf{\textit{(2) How can unlabeled target graphs guide the training graphs without node correspondence?}}
Target graphs reveal how the structure changes across domains, but they differ from the training graphs in size and are not isomorphic to them. Node- and edge-level correspondences are undefined, so the target signal cannot be expressed as a graph-to-graph matching~\cite{huang2024can,gao2025graph,wang2026cross}. Guiding variable-sized training graphs with unpaired target graphs requires a signal defined over graph collections and invariant to node ordering.
\textbf{\textit{(3) How can the training graphs move toward the target domain while remaining class-discriminative?}}
Target graphs carry no labels, so the target signal indicates how the structure should change but not how the changes relate to individual classes. Training graphs updated solely toward the target structure can resemble the target graphs while no longer separating the source classes~\cite{jin2022condensing,gao2025graph}. Since the updated graphs remain training examples, the updates must preserve the class information supervised by the source labels.

To address these challenges, we propose Target-Conditioned Graph \textbf{Proto}type Distillation for \textbf{G}raph \textbf{D}omain \textbf{A}daptation (\method{}), a data-centric framework that adapts the training graphs instead of the model. \method{} represents the training data as a compact set of learnable graph prototypes. \textit{Differentiable Semantic Prototype Distillation} parameterizes each prototype with a probabilistic adjacency matrix, learnable node features, and a soft label, and distills source supervision into the prototypes through bilevel optimization with a proxy GNN, which makes their structure differentiable and grounds them in the source classes. Building on this parameterization, \textit{Correspondence-Free Geometric Alignment} and \textit{Dirichlet-Energy Statistic Alignment} summarize each graph with permutation-invariant descriptors of connectivity and feature variation, and match the descriptor means of the prototypes and the target graphs without node correspondence. The alignment terms update only the structure and features of the prototypes, while their labels remain governed by source supervision, so the prototypes move toward the target domain without losing class information. \textit{Deployed Predictor Construction} then trains an independently initialized GNN on the frozen prototypes for target prediction. We further show that the target risk of the deployed predictor is bounded by its fitting error on the prototypes, the alignment loss, and a label-consistency term informed only by source supervision. Extensive experiments on graph and image benchmarks under diverse domain shifts demonstrate the effectiveness of \method{}, with improvements over representative state-of-the-art methods in most settings.

Our contributions are summarized as follows: (1) We study GDA from a data-centric perspective and identify the training graphs as an adaptation variable, which shifts the focus from adapting the model to adapting the data it is trained on. (2) We propose \method{}, which distills source supervision into differentiable graph prototypes and aligns them with unlabeled target graphs through permutation-invariant descriptors of connectivity and feature variation. We further establish a target risk bound that connects each component to the generalization of the deployed predictor. (3) Extensive experiments on graph and image benchmarks under diverse domain shifts demonstrate the effectiveness of \method{} over representative state-of-the-art methods, and visualizations show that the learned prototypes are structurally diverse and follow the target statistics.

\section{Related Work}
\label{sec:related}

\textbf{Graph Domain Adaptation (GDA).}
GDA transfers predictive knowledge from a labeled source domain to an unlabeled target domain under distribution shifts in graph-structured data~\cite{wu2024graph,liu2024revisiting,cai2024graph,wang2026safe}. Existing methods fall into two families according to how they handle the shift. Alignment-based methods extend adversarial training and distribution matching~\cite{ganin2015unsupervised} to graph encoders~\cite{wu2020unsupervised,dai2022graph,wang2026learning} and learn representations whose source and target distributions are indistinguishable, and they further incorporate adversarial perturbation~\cite{yin2022deal}, contrastive learning~\cite{yin2023coco}, and pseudo-labeling~\cite{qiao2023semi} to exploit unlabeled target graphs. These methods treat graph structure as an input to the encoder and assume that aligning pooled representations removes the shift, so structural differences between domains are absorbed into the representation space without being modeled explicitly. Structure-aware methods build graph-specific inductive biases into adaptation, by reweighting source edges~\cite{liu2023structural} or labels~\cite{liu2024pairwise}, refining message propagation~\cite{liu2024rethinking}, exploiting node attributes~\cite{fang2025benefits}, enforcing representation smoothness~\cite{chen2025smoothness}, or augmenting~\cite{pang2023sa}, aligning~\cite{xiao2023spa}, and regularizing~\cite{you2023graph,yang2025disentangled} graph spectra across domains. These methods model topology discrepancies more explicitly but act on how a model processes the observed source graphs, and their benefit depends on whether the shift matches the encoded bias. Both families keep the source graphs fixed as training data. Differences in connectivity and in how node features vary over the topology then persist in the data that supervise the model, which limits the transferability of graph filters calibrated on the source domain~\cite{levie2021transferability}. \method{} departs from both families by optimizing the training graphs themselves, so that the structural mismatch is reduced in the data before a predictor is trained on it.

\noindent\textbf{Graph Distillation.}
Dataset distillation synthesizes a small training set on which a model attains performance comparable to training on the full data~\cite{wang2018dataset}. Its methods share a bilevel structure, in which a model is trained on synthetic data in an inner loop and an outer objective on real data updates the synthetic data. They differ mainly in how the outer gradient is obtained, either by differentiating through the unrolled inner optimization~\cite{wang2018dataset,franceschi2018bilevel} or by matching surrogate quantities such as gradients~\cite{zhao2021dataset} and training trajectories~\cite{cazenavette2022dataset} between synthetic and real data, which trades fidelity to the inner problem for lower computational cost. Graph condensation adapts this framework to graph data~\cite{gao2025graph}, where the central design choice is how to make discrete structure optimizable. Existing methods relax adjacency into probabilistic variables~\cite{jin2022condensing}, generate structure from learnable node features~\cite{jin2022graph}, or replace the inner GNN with kernel~\cite{xu2023kernel}, self-supervised~\cite{wang2024self}, and clustering-based~\cite{lai2025simple} surrogates, which reduces the cost of bilevel optimization and the dependence on a specific architecture. Graph knowledge distillation compresses models instead of data, transferring knowledge from a teacher to a student~\cite{hinton2015distilling} through representation-level~\cite{joshi2022representation}, label-aware~\cite{hong2024label}, structure-aware~\cite{yao2024mugsi}, teacher-free~\cite{wu2024teacher}, or GNN-to-MLP~\cite{lu2024adagmlp} matching. Both lines pursue fidelity to a single data distribution, so the synthetic graphs or student models they produce inherit the source statistics and carry the source--target mismatch into downstream training. \method{} retains the bilevel structure and the relaxed graph parameterization but conditions the distilled prototypes on the statistics of unlabeled target graphs, extending dataset distillation from in-domain compression to cross-domain adaptation.
\section{Methodology}
\label{sec:method}

\begin{figure}[t]
  \centering
  \includegraphics[width=\linewidth]{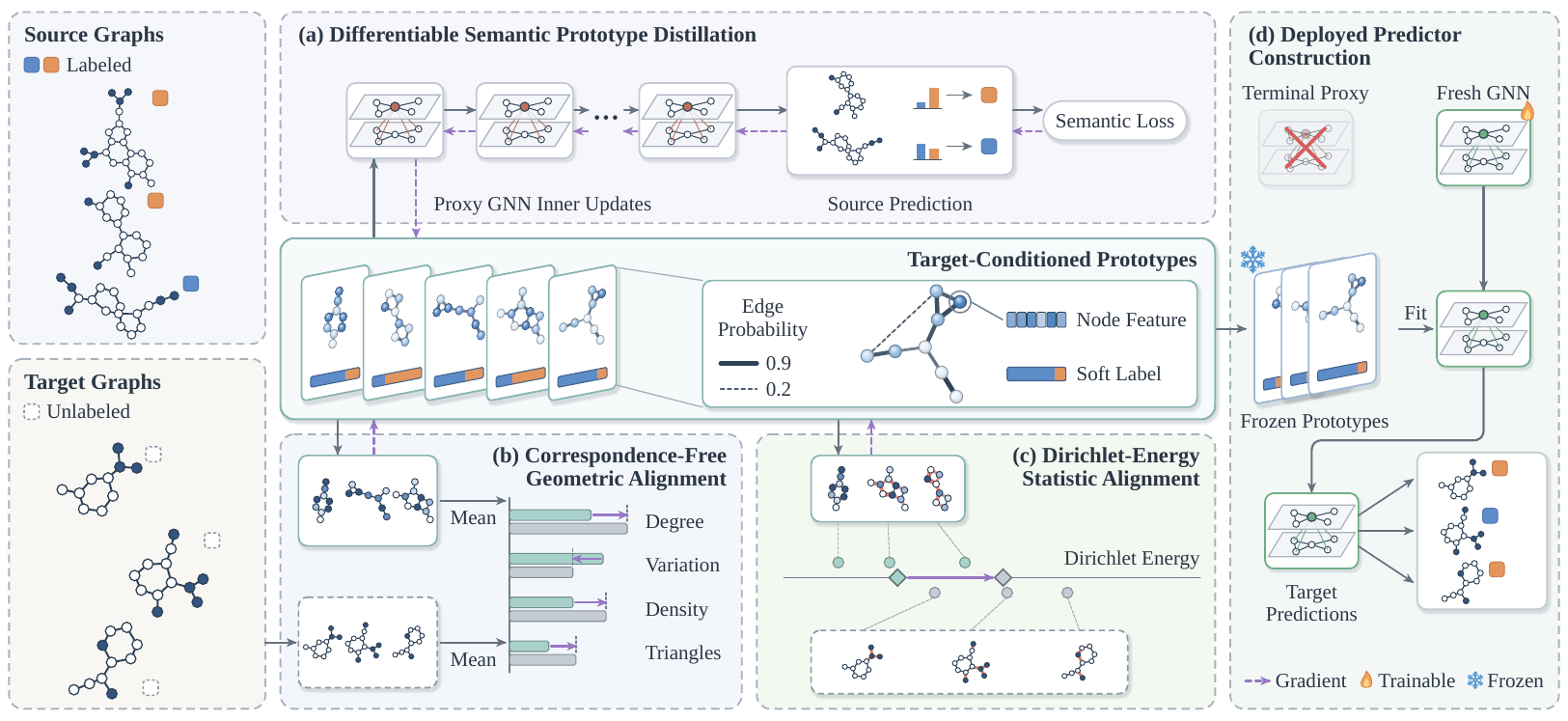}
  \vspace{-0.5cm}
  \caption{Overview of \method{}. (a) Semantic prototype distillation back-propagates the source classification loss of a proxy GNN to the prototypes. (b) Geometric alignment matches the mean degree, density, and triangle statistics of the prototypes and target graphs. (c) Dirichlet-energy alignment matches their mean Dirichlet energy. (d) A fresh GNN is trained on the frozen prototypes for target prediction.}
  \label{fig:framework}
  \vspace{-0.3cm}
\end{figure}

\subsection{Problem Setup.}
We study unsupervised GDA for graph classification. Let $\mathcal D_S=\{(G_i^s,y_i^s)\}_{i=1}^{N_S}$ denote a labeled source domain and $\mathcal D_T=\{G_j^t\}_{j=1}^{N_T}$ denote an unlabeled target domain, drawn from joint distributions $P_S$ and $P_T$ over attributed graphs and labels. Each graph $G=(\mathbf A,\mathbf X)$ has $N$ nodes, a symmetric adjacency matrix $\mathbf A\in\{0,1\}^{N\times N}$ with zero diagonal, and a node feature matrix $\mathbf X\in\mathbb R^{N\times d}$, and each label lies in $\mathcal Y=\{1,\ldots,C\}$. The two domains share the label space and the node feature space, whereas $P_S(G,y)\neq P_T(G,y)$, and the shift can affect graph size and connectivity, node attributes, or the correlation between label-determining substructures and their context. Given $\mathcal D_S$ and $\mathcal D_T$, we learn a classifier $f$ that maps each graph to $C$ class scores and minimizes the expected target risk
\begin{equation}
\mathcal{R}_T(f)
=
\mathbb{E}_{(G, y) \sim P_T}
\big[ \ell(f(G), y) \big],
\label{eq:target_risk}
\end{equation}
where $\ell$ denotes the classification loss and $\ell_{\mathrm{CE}}$ denotes the cross-entropy between a predicted probability vector and a hard or soft label. Target labels are unavailable during training, so $\mathcal R_T(f)$ cannot be optimized directly. Appendix~\ref{app:notation} summarizes the notation.

\subsection{Overview.}
As shown in Fig.~\ref{fig:framework}, \method{} adapts the training graphs instead of the model by learning a compact set of graph prototypes. (i) \textbf{Differentiable Semantic Prototype Distillation} parameterizes each prototype with a probabilistic adjacency matrix, learnable node features, and a soft label, and distills source supervision into the prototypes through bilevel optimization with a proxy GNN. (ii) \textbf{Correspondence-Free Geometric Alignment} summarizes each graph with permutation-invariant connectivity descriptors and matches their means between the prototypes and the unlabeled target graphs. (iii) \textbf{Dirichlet-Energy Statistic Alignment} complements geometric alignment by matching the mean Dirichlet energy, a statistic that couples node features with topology. (iv) \textbf{Deployed Predictor Construction} trains an independently initialized GNN on the frozen prototypes for target prediction.

\subsection{Differentiable Semantic Prototype Distillation}
\label{sec:semantic}

Using graph structure as an adaptation variable requires prototypes whose topology admits gradients and whose labels remain grounded in the source classes. Graph condensation relaxes discrete structures into continuous variables and optimizes synthetic graphs by gradient matching or self-supervision~\cite{jin2022condensing,wang2024self,gao2025graph}, but it is designed for in-domain compression and leaves the synthetic graphs unconditioned on any other domain. We parameterize the prototypes with relaxed adjacency matrices and distill source supervision into them through finite-step bilevel optimization, which keeps every prototype variable differentiable and open to signals beyond the source domain.

\textbf{(i) Probabilistic Prototype Parameterization.}
We denote the prototype set by $\mathcal{B}=\{(\widetilde{G}_k,\widetilde{\mathbf{y}}_k)\}_{k=1}^{K}$, where the $k$-th prototype $\widetilde{G}_k=(\widetilde{\mathbf{A}}_k,\widetilde{\mathbf{X}}_k)$ consists of a probabilistic adjacency matrix $\widetilde{\mathbf{A}}_k\in[0,1]^{\widetilde{N}_k\times \widetilde{N}_k}$ and a node feature matrix $\widetilde{\mathbf{X}}_k\in\mathbb{R}^{\widetilde{N}_k\times d}$, and $\widetilde{\mathbf{y}}_k\in\Delta^{C-1}$ is its soft label over $C$ classes. We construct $\widetilde{\mathbf{A}}_k$ from a learnable score matrix $\mathbf{Z}_k$ as
\begin{equation}
\mathbf{P}_k=\sigma(\mathbf{Z}_k),\quad
\mathbf{S}_k=\frac{\mathbf{P}_k+\mathbf{P}_k^{\top}}{2},\quad
\widetilde{\mathbf{A}}_k
=(\mathbf{1}\mathbf{1}^{\top}-\mathbf{I})\odot\mathbf{S}_k,
\label{eq:adj_param}
\end{equation}
where $\sigma(\cdot)$ is the sigmoid function and $\odot$ denotes elementwise multiplication. Symmetrization and the diagonal mask yield an undirected graph without self-loops, while the sigmoid relaxation makes edge selection continuous. The entries of $\widetilde{\mathbf A}_k$ act as edge weights in a differentiable message-passing encoder, so any loss on the prototypes propagates gradients to $\mathbf Z_k$ and $\widetilde{\mathbf X}_k$.

\textbf{(ii) Finite-Step Bilevel Distillation.}
A differentiable topology alone does not make the prototypes class-discriminative, which calls for supervision from the labeled source graphs. Let $f_{\theta}$ denote a proxy GNN and $\operatorname{Opt}_{\theta}^{T_{\mathrm{in}}}$ denote $T_{\mathrm{in}}$ steps of an inner optimizer. Starting from the proxy state $\theta^{(0)}$, we formulate the outer and inner levels as
\begin{equation}
\begin{aligned}
\min_{\mathcal B}\quad
\mathcal{L}_{\mathrm{total}}(\mathcal B)
&=
\mathcal{L}_{\mathrm{sem}}(\widehat\theta_{\mathcal B};\mathcal{D}_S)
+
\mathcal{L}_{\mathrm{tgt}}(\mathcal{B};\mathcal{D}_T),
\\[-1pt]
\text{s.t.}\quad
\widehat\theta_{\mathcal B}
&=
\operatorname{Opt}_{\theta}^{T_{\mathrm{in}}}
\big(
\theta^{(0)};
\mathcal L_{\mathrm{in}}(\theta;\mathcal B)
\big),
\\[-1pt]
\mathcal{L}_{\mathrm{in}}(\theta;\mathcal{B})
&=
\frac{1}{K}
\sum_{k=1}^{K}
\ell_{\mathrm{CE}}
\big(
f_{\theta}(\widetilde{G}_k),
\widetilde{\mathbf{y}}_k
\big),
\end{aligned}
\label{eq:semantic_bilevel}
\end{equation}
where $\mathcal{L}_{\mathrm{sem}}$ is a semantic loss on the labeled source graphs and $\mathcal{L}_{\mathrm{tgt}}$ is a label-free term computed from the unlabeled target graphs. The inner level trains only the proxy parameters $\theta$ on the fixed prototypes, and we instantiate $\operatorname{Opt}_{\theta}^{T_{\mathrm{in}}}$ with Adam, so $\widehat\theta_{\mathcal B}$ is the proxy state after $T_{\mathrm{in}}$ updates. Since $\mathcal L_{\mathrm{in}}$ depends on $\mathcal B$, the trained proxy $\widehat\theta_{\mathcal B}$ is itself a differentiable function of the prototype variables. The semantic loss evaluates this proxy on the labeled source graphs,
\begin{equation}
\mathcal{L}_{\mathrm{sem}}(\widehat\theta_{\mathcal B};\mathcal{D}_S)
=
\frac{1}{N_S}
\sum_{i=1}^{N_S}
\ell_{\mathrm{CE}}
\big(
f_{\widehat\theta_{\mathcal B}}(G_i^s),y_i^s
\big).
\label{eq:semantic_loss}
\end{equation}
Back-propagating the semantic loss through the unrolled inner updates transfers source supervision to the prototype adjacencies, features, and soft labels, so that a GNN trained on the prototypes classifies the source graphs correctly. This objective shapes the prototypes only through source-domain predictive performance and carries no information about the shifted target distribution. The label-free term $\mathcal{L}_{\mathrm{tgt}}$ supplies this missing signal from statistics of the unlabeled target graphs and acts directly on the prototype variables.

\subsection{Correspondence-Free Geometric Alignment}
\label{sec:geometric}

Conditioning the prototypes on the target domain requires comparing graphs that differ in size and share no canonical node correspondence, so direct structural matching between prototypes and target graphs is ill-posed. Structure-aware GDA methods reweight source edges or modify message propagation~\cite{liu2023structural,liu2024rethinking}, which acts on the model or the observed source graphs and offers no objective for optimizing variable-sized prototypes. We instead map every graph to a fixed-dimensional, permutation-invariant descriptor and align the prototypes with the target graphs at the level of collection means.

For an adjacency matrix $\mathbf{A}\in[0,1]^{N\times N}$ with degree vector $\mathbf{d}=\mathbf{A}\mathbf{1}$, we define the geometric descriptor as
\begin{equation}
\boldsymbol{\psi}_{\mathrm{geo}}(\mathbf{A})
=
\big[
\mu_d(\mathbf{A}),
\sigma_d(\mathbf{A}),
\rho(\mathbf{A}),
\tau(\mathbf{A})
\big]^{\top},
\label{eq:geo_descriptor}
\end{equation}
whose entries are the mean degree, the degree standard deviation, the graph density, and the per-node weighted triangle intensity,
\begin{equation}
\begin{aligned}
\mu_d(\mathbf{A})
&=
\frac{1}{N}\mathbf{1}^{\top}\mathbf{d},
&
\sigma_d(\mathbf{A})
&=
\left(
\frac{1}{N}
\left\|
\mathbf{d}-\mu_d(\mathbf{A})\mathbf{1}
\right\|_2^2
+\epsilon
\right)^{1/2},
\\[2pt]
\rho(\mathbf{A})
&=
\frac{\mathbf{1}^{\top}\mathbf{A}\mathbf{1}}
{N(N-1)+\epsilon},
&
\tau(\mathbf{A})
&=
\frac{\operatorname{Tr}(\mathbf{A}^{3})}
{6N+\epsilon},
\end{aligned}
\label{eq:geo_components}
\end{equation}
with a small constant $\epsilon>0$ for numerical stability. The first two entries summarize the degree distribution, $\rho$ measures global sparsity, and $\tau$ captures local clustering. Each entry is a smooth function of the adjacency entries, so the descriptor remains differentiable when evaluated on the relaxed adjacency matrices of the prototypes. We average the descriptor over the prototype set and over the target domain,
\begin{equation}
\bar{\boldsymbol{\psi}}_{\mathrm{geo}}^{B}
=
\frac{1}{K}
\sum_{k=1}^{K}
\boldsymbol{\psi}_{\mathrm{geo}}(\widetilde{\mathbf{A}}_k),\quad
\bar{\boldsymbol{\psi}}_{\mathrm{geo}}^{T}
=
\frac{1}{N_T}
\sum_{j=1}^{N_T}
\boldsymbol{\psi}_{\mathrm{geo}}(\mathbf{A}_j^t),
\label{eq:geo_stats}
\end{equation}
where the target mean does not depend on $\mathcal B$ and is computed once before optimization. The geometric alignment loss is defined as
\begin{equation}
\mathcal{L}_{\mathrm{geo}}(\mathcal{B};\mathcal{D}_T)
=
\left\|
\bar{\boldsymbol{\psi}}_{\mathrm{geo}}^{B}
-
\bar{\boldsymbol{\psi}}_{\mathrm{geo}}^{T}
\right\|_2^2 ,
\label{eq:geo_loss}
\end{equation}
which transfers the connectivity statistics of the target domain to the prototype adjacencies without node- or edge-level correspondence.

\subsection{Dirichlet-Energy Statistic Alignment}
\label{sec:energy}

The geometric descriptors capture connectivity but not how node features vary over it, so two graphs with identical degree, density, and triangle statistics can still differ in feature--topology coupling. Spectral GDA methods align graph eigenspaces or regularize the adapted model in the spectral domain~\cite{xiao2023spa,you2023graph}, and these objectives act on model representations without providing a compact feature-dependent statistic for both variable-sized target graphs and learnable prototypes. We adopt the Dirichlet energy, which measures the Laplacian-weighted variation of node features, and align its collection mean between the prototypes and the target graphs.

Given a graph $G=(\mathbf A,\mathbf X)$, we construct the self-looped normalized Laplacian
\begin{equation}
\widehat{\mathbf{A}}
=
\mathbf{A}+\mathbf{I}, \quad
\widehat{\mathbf{D}}
=
\operatorname{diag}(\widehat{\mathbf{A}}\mathbf{1}), \quad
\widehat{\mathbf{L}}
=
\mathbf{I}
-
\widehat{\mathbf{D}}^{-1/2}
\widehat{\mathbf{A}}
\widehat{\mathbf{D}}^{-1/2},
\label{eq:norm_lap}
\end{equation}
and define its Dirichlet energy as
\begin{equation}
\Omega(G)
\equiv \Omega(\mathbf A,\mathbf X)
=
\operatorname{Tr}
\left(
\mathbf{X}^{\top}
\widehat{\mathbf{L}}
\mathbf{X}
\right).
\label{eq:dirichlet_energy}
\end{equation}
For a prototype, the continuous entries of $\widetilde{\mathbf A}_k$ serve as edge weights in its normalized Laplacian, and the self-loops keep every degree positive, so $\Omega(\widetilde G_k)$ is differentiable with respect to both $\widetilde{\mathbf A}_k$ and $\widetilde{\mathbf X}_k$. This distinguishes it from the geometric descriptor, which depends on the adjacency alone. As a quadratic form in $\mathbf X$, $\Omega$ is homogeneous of degree two,
\begin{equation}
\Omega(\mathbf A,c\mathbf X)
=c^2\Omega(\mathbf A,\mathbf X),
\qquad c\in\mathbb R ,
\label{eq:energy_scale}
\end{equation}
and since the Laplacian of a disjoint union is block diagonal, $\Omega$ is additive over disjoint components, giving $\Omega(G^{\sqcup r})=r\,\Omega(G)$ for the union of $r$ identical attributed graphs. The Dirichlet energy reflects feature magnitude and graph size together with feature--topology coupling, and matching its mean aligns the overall level of feature variation of the prototypes with that of the target graphs. We compute the mean energy over the prototype set and over the target domain,
\begin{equation}
\bar{\Omega}^{B}
=
\frac{1}{K}
\sum_{k=1}^{K}
\Omega(\widetilde{G}_k),\quad
\bar{\Omega}^{T}
=
\frac{1}{N_T}
\sum_{j=1}^{N_T}
\Omega(G_j^t),
\label{eq:spec_stats}
\end{equation}
and define the Dirichlet-energy alignment loss as
\begin{equation}
\mathcal{L}_{\mathrm{spec}}(\mathcal{B};\mathcal{D}_T)
=
\big(
\bar{\Omega}^{B}
-
\bar{\Omega}^{T}
\big)^2 .
\label{eq:spec_loss}
\end{equation}

The geometric loss constrains the prototypes through connectivity statistics, and the energy loss constrains feature variation over the learned topology. Together they form the label-free target term of the bilevel objective, which yields the complete outer objective
\begin{equation}
\begin{aligned}
\mathcal L_{\mathrm{tgt}}(\mathcal B;\mathcal D_T)
&=
\lambda_{\mathrm{geo}}
\mathcal L_{\mathrm{geo}}(\mathcal B;\mathcal D_T)
+
\lambda_{\mathrm{spec}}
\mathcal L_{\mathrm{spec}}(\mathcal B;\mathcal D_T),\\
\mathcal{L}_{\mathrm{total}}(\mathcal{B})
&=
\mathcal{L}_{\mathrm{sem}}(\widehat\theta_{\mathcal B};\mathcal{D}_S)
+
\mathcal L_{\mathrm{tgt}}(\mathcal B;\mathcal D_T),
\end{aligned}
\label{eq:outer_meta}
\end{equation}
where $\lambda_{\mathrm{geo}}$ and $\lambda_{\mathrm{spec}}$ balance the two alignment terms. When computing $\nabla_{\mathcal B}\mathcal L_{\mathrm{total}}$, we retain the dependence of $\widehat\theta_{\mathcal B}$ on $\mathcal B$, so the semantic loss propagates source supervision through the inner optimization, while the two alignment losses act directly on the prototype adjacencies and features. The following result characterizes how these two signals reach each prototype.

\begin{proposition}[Correspondence-Free and Label-Free Target Conditioning]
\label{prop:conditioning}
For an attributed graph $G=(\mathbf A,\mathbf X)$, let $\Phi(G)=[\boldsymbol\psi_{\mathrm{geo}}(\mathbf A)^{\top},\Omega(\mathbf A,\mathbf X)]^{\top}\in\mathbb R^{5}$ collect its geometric descriptor and Dirichlet energy, let $\mathbf M=\operatorname{diag}(\lambda_{\mathrm{geo}}\mathbf I_4,\lambda_{\mathrm{spec}})$ with $\lambda_{\mathrm{geo}},\lambda_{\mathrm{spec}}\ge0$, and let $\bar\Phi^{B}$ and $\bar\Phi^{T}$ denote the means of $\Phi$ over the prototype set $\mathcal B$ and the target collection $\mathcal D_T$. Then $\Phi(\boldsymbol\Pi\mathbf A\boldsymbol\Pi^{\top},\boldsymbol\Pi\mathbf X)=\Phi(\mathbf A,\mathbf X)$ for every permutation matrix $\boldsymbol\Pi$ of matching size, and
\begin{equation}
\mathcal L_{\mathrm{tgt}}(\mathcal B;\mathcal D_T)=\big\|\mathbf M^{1/2}\big(\bar\Phi^{B}-\bar\Phi^{T}\big)\big\|_2^2 .
\label{eq:prop_distance}
\end{equation}
Moreover, if the inner optimizer is differentiable, then for the $k$-th prototype with topology and feature variables $\mathbf b_k=(\mathbf Z_k,\widetilde{\mathbf X}_k)$ and soft label $\widetilde{\mathbf y}_k$, the gradients of the outer objective decompose as
\begin{equation}
\begin{aligned}
\nabla_{\mathbf b_k}\mathcal L_{\mathrm{total}}
&=\mathbf J_{\theta,k}^{\top}\nabla_{\widehat\theta}\mathcal L_{\mathrm{sem}}
+\frac{2}{K}\mathbf J_{\Phi,k}^{\top}\mathbf M\big(\bar\Phi^{B}-\bar\Phi^{T}\big),\\
\nabla_{\widetilde{\mathbf y}_k}\mathcal L_{\mathrm{total}}
&=\Big(\frac{\partial\widehat\theta_{\mathcal B}}{\partial\widetilde{\mathbf y}_k}\Big)^{\top}\nabla_{\widehat\theta}\mathcal L_{\mathrm{sem}},
\end{aligned}
\label{eq:prototype_gradient}
\end{equation}
where $\mathbf J_{\theta,k}=\partial\widehat\theta_{\mathcal B}/\partial\mathbf b_k$ and $\mathbf J_{\Phi,k}=\partial\Phi(\widetilde G_k)/\partial\mathbf b_k$.
\end{proposition}

Proposition~\ref{prop:conditioning} shows that the target term is a weighted squared distance between mean descriptors, which equals the squared empirical maximum mean discrepancy between the two collections under the linear kernel $k_\Phi(G,G')=\Phi(G)^{\top}\mathbf M\,\Phi(G')$~\cite{gretton2012kernel}. Since $\Phi$ is permutation invariant and maps graphs of any size to $\mathbb R^{5}$, prototypes and target graphs enter the same discrepancy without node correspondence or equal node counts. The gradient decomposition further shows that topology and features receive both the semantic signal and the target residual, whereas soft labels receive only the semantic signal. Target statistics reshape the prototypes without directly determining their classes, so $\mathcal L_{\mathrm{sem}}$ remains the only source of class supervision. Because $\boldsymbol\psi_{\mathrm{geo}}$ depends only on the adjacency, node features receive target information only through the energy coordinate of $\Phi$. Finally, the shared residual $\bar\Phi^{B}-\bar\Phi^{T}$ reaches each prototype through its own Jacobian $\mathbf J_{\Phi,k}$, so prototypes with different structures are updated in different directions instead of being pushed toward a common template.

\subsection{Deployed Predictor Construction}
\label{sec:deploy}

The proxy GNN transmits source supervision to the prototypes during bilevel optimization and is not trained for deployment. Since the prototypes change after every outer update, the terminal proxy is fitted to the prototypes before the final update and for only a few inner steps, so its parameters do not match the final prototypes. We separate distillation from deployment by freezing the optimized prototypes, discarding the proxy, and training an independently initialized GNN on the frozen prototypes, which gives the deployed predictor a stationary training objective.

Let $\mathcal{B}^{*}=\{(\widetilde{G}_k^{*},\widetilde{\mathbf{y}}_k^{*})\}_{k=1}^{K}$ denote the optimized prototype set. Starting from an initialization $\phi^{(0)}$ independent of the proxy, the deployed predictor $f_{\phi}$ is trained on $\mathcal B^*$ by
\begin{equation}
\begin{aligned}
\mathcal{L}_{\mathrm{inf}}(\phi;\mathcal{B}^{*})
&=
\frac{1}{K}
\sum_{k=1}^{K}
\ell_{\mathrm{CE}}
\big(
f_{\phi}(\widetilde{G}_k^{*}),
\widetilde{\mathbf{y}}_k^{*}
\big),\\[-1pt]
\phi^*
&=
\operatorname{Opt}_{\phi}^{T_{\mathrm{inf}}}
\big(
\phi^{(0)};
\mathcal L_{\mathrm{inf}}(\phi;\mathcal B^*)
\big),
\end{aligned}
\label{eq:inference_train}
\end{equation}
where $\operatorname{Opt}_{\phi}^{T_{\mathrm{inf}}}$ denotes $T_{\mathrm{inf}}$ Adam updates with learning rate $\eta_{\mathrm{inf}}$. The trained classifier predicts the label of each target graph $G_j^t\in\mathcal D_T$ as
\begin{equation}
\hat{y}_j^t
=
\arg\max_{c\in\{1,\ldots,C\}}
\left[f_{\phi^{*}}(G_j^t)\right]_c .
\label{eq:target_prediction}
\end{equation}
The deployed predictor is trained only on the frozen prototypes, so it receives information from both the source and the target domains exclusively through them. Although its training loss has the same form as the inner objective of distillation, it defines a single-level problem on fixed prototypes, and no gradient flows back to the prototypes.

To analyze how the frozen prototypes govern target prediction, we measure errors with the total-variation loss $\ell_{\mathrm{TV}}(\mathbf p,\mathbf q)=\frac12\|\mathbf p-\mathbf q\|_1$. It is a metric on the probability simplex $\Delta^{C-1}$, is bounded by one, and equals $1-[\mathbf p]_y$ when $\mathbf q$ is the one-hot label $\mathbf e_y$.

\begin{proposition}[Target Risk of the Deployed Predictor]
\label{prop:risk}
Let $\mathbb G$ be the space of attributed graphs with at most $N_{\max}$ nodes, symmetric adjacency matrices in $[0,1]^{N\times N}$ with zero diagonal, and node features of $\ell_2$ norm at most $x_{\max}$, and let $\mathcal H$ be a class of predictors $h\colon\mathbb G\to\Delta^{C-1}$. For a prototype set $\mathcal B\subset\mathbb G$ and $h\in\mathcal H$, define the target risk $R_T(h)=\mathbb E_{(G,y)\sim P_T}[\ell_{\mathrm{TV}}(h(G),\mathbf e_y)]$, the prototype risk $R_{\mathcal B}(h)=\frac1K\sum_{k=1}^{K}\ell_{\mathrm{TV}}(h(\widetilde G_k),\widetilde{\mathbf y}_k)$, and the joint optimal risk $\gamma_{\mathcal B}=\inf_{h'\in\mathcal H}[R_{\mathcal B}(h')+R_T(h')]$. With $\Phi(G)=[\boldsymbol\psi_{\mathrm{geo}}(\mathbf A)^{\top},\Omega(\mathbf A,\mathbf X)]^{\top}$ and $\mathbf M=\operatorname{diag}(\lambda_{\mathrm{geo}}\mathbf I_4,\lambda_{\mathrm{spec}})$ for $\lambda_{\mathrm{geo}},\lambda_{\mathrm{spec}}>0$, suppose that (i) for every $h,h'\in\mathcal H$, the disagreement $g(G)=\ell_{\mathrm{TV}}(h(G),h'(G))$ admits $\mathbf w\in\mathbb R^{5}$ and $c\in\mathbb R$ with $\|\mathbf M^{-1/2}\mathbf w\|_2\le\kappa$ and $\sup_{G\in\mathbb G}|g(G)-\mathbf w^{\top}\Phi(G)-c|\le\varepsilon_\Phi$, (ii) $\mathcal D_T$ is drawn i.i.d.\ from the target marginal $P_T^{\mathbb G}$, and (iii) $\|\mathbf M^{1/2}\Phi(G)\|_2\le R$ holds $P_T^{\mathbb G}$-almost surely. Then for any $\delta\in(0,1)$, with probability at least $1-\delta$, every prototype set $\mathcal B\subset\mathbb G$, including one optimized on $\mathcal D_T$, and every $h\in\mathcal H$ satisfy
\begin{equation}
R_T(h)\le R_{\mathcal B}(h)+\kappa\sqrt{\mathcal L_{\mathrm{tgt}}(\mathcal B;\mathcal D_T)}+\gamma_{\mathcal B}+2\varepsilon_\Phi+\frac{\kappa R\big(1+\sqrt{2\ln(1/\delta)}\big)}{\sqrt{N_T}}.
\label{eq:main_bound}
\end{equation}
Moreover, on the same event, every predictor $f_{\phi}\in\mathcal H$ evaluated on the frozen prototypes $\mathcal B^*$ with the cross-entropy loss $\mathcal L_{\mathrm{inf}}(\phi;\mathcal B^*)=\frac1K\sum_{k=1}^{K}\ell_{\mathrm{CE}}(f_{\phi}(\widetilde G_k^*),\widetilde{\mathbf y}_k^*)$ satisfies
\begin{equation}
R_T(f_{\phi})\le\sqrt{\tfrac12\mathcal L_{\mathrm{inf}}(\phi;\mathcal B^*)}+\kappa\sqrt{\mathcal L_{\mathrm{tgt}}(\mathcal B^*;\mathcal D_T)}+\gamma_{\mathcal B^*}+2\varepsilon_\Phi+\frac{\kappa R\big(1+\sqrt{2\ln(1/\delta)}\big)}{\sqrt{N_T}} .
\label{eq:deployed_bound}
\end{equation}
\end{proposition}

Proposition~\ref{prop:risk} follows the decomposition of \citet{ben2010theory}, with the source sample replaced by the learnable prototype set and the divergence replaced by an integral probability metric over descriptor-linear functions~\cite{muller1997integral}, and it ties each term of the target risk to one component of \method{}. The first term of the deployed bound is the fitting error on the frozen prototypes, which deployed predictor construction minimizes directly, whereas the terminal proxy is fitted to the prototypes before the final outer update. The second term is proportional to the square root of the alignment loss, so the correspondence-free objectives control the domain discrepancy through collection means alone. The estimation error decays as $\mathcal O(N_T^{-1/2})$ independently of $K$, even when the prototypes are optimized on the same target graphs. The third term depends on the soft labels jointly with the target concept, and since the alignment loss carries no label information, the semantic objective is the only term that steers $\gamma_{\mathcal B}$. The bound depends on $\mathbf M$ only through the ratio of the two weights, because rescaling $\mathbf M$ scales $\sqrt{\mathcal L_{\mathrm{tgt}}}$ and the admissible $\kappa$ in opposite directions. The overall magnitude of the weights instead governs optimization, where large weights let statistic matching dominate the semantic objective. Finally, for a fixed $\kappa$, adding $\Omega$ to $\boldsymbol\psi_{\mathrm{geo}}$ can only reduce $\varepsilon_\Phi$, since any linear function of $\boldsymbol\psi_{\mathrm{geo}}$ is a linear function of $\Phi$ with zero weight on $\Omega$, which motivates feature-dependent alignment in addition to geometric alignment.

\subsection{Learning Framework}
\label{sec:learning_framework}

\method{} is trained in two stages, summarized in Algorithm~\ref{alg:overall}. Before training, the target statistics $\bar{\boldsymbol\psi}_{\mathrm{geo}}^{T}$ and $\bar\Omega^{T}$ are computed in a single pass over the target graphs and kept fixed afterwards. In the prototype distillation stage, each outer iteration first trains the proxy GNN from the initialization $\theta^{(0)}$ for $T_{\mathrm{in}}$ Adam steps on the prototypes with $\mathcal B$ held fixed. The prototypes then receive one Adam step on the total outer objective, where the semantic gradient is back-propagated through the inner updates and the alignment gradient is computed directly from the prototype statistics. After $T_{\mathrm{out}}$ outer iterations, the deployment stage freezes the prototypes as $\mathcal B^*$, discards the proxy GNN, trains the predictor $f_{\phi}$ on $\mathcal B^*$, and assigns each target graph the class with the highest predicted probability. The cost of prototype distillation scales with $K$, the prototype size, $T_{\mathrm{in}}$, and the GNN depth, and does not depend on the dataset sizes. The cost of evaluating the semantic loss grows linearly with the number of source graphs, as in standard GNN training, and the target graphs contribute only a one-time statistic extraction and the final inference.

\section{Experiments}
\label{sec:experiments}

\begin{table}[t]
    \centering
    \caption{Statistics of the experimental datasets.}
    \resizebox{0.6\textwidth}{!}{
    \begin{tabular}{lcccc}
        \toprule
        Datasets      & Graphs & Avg. Nodes & Avg. Edges & Classes \\
        \midrule
        NCI1  & 4,110   & 29.87     & 32.30     & 2       \\
        Mutagenicity  & 4,337   & 30.32      & 30.77      & 2       \\
        FRANKENSTEIN  & 4,337   & 16.9       & 17.88      & 2       \\
        ogbg-molhiv & 41,127 & 25.5 & 27.5 & 2 \\
        CIFAR-10 & 60,000 & 117.6 & 941.2 & 10 \\

        MNIST & 70,000 & 70.6 & 564.5 & 10 \\
        \midrule
        Spurious-Motif & 3,000 & 87.78 & 124.76 & 3 \\
        Spurious-Motif\_bias & 3,000 & 18.67 & 27.83 & 3 \\
        
        \midrule
         PROTEINS  & 1,113   & 39.1      & 72.8      & 2       \\
        DD & 1,178 & 284.32 & 715.66 & 2 \\
        COX2 & 467 & 41.22 & 43.45& 2 \\
        COX2\_MD & 303 & 26.28 &	335.12 & 2\\
        BZR & 405 & 35.75 & 38.36 & 2 \\
        BZR\_MD & 306 & 21.30 &	225.06 & 2 \\
        \bottomrule
    \end{tabular}}
    \label{tab:dataset}
\end{table}

\textbf{Datasets.}
We evaluate \method{} on graph and image benchmarks covering three types of domain shift.
(1) \textit{Structural shifts.} Following~\cite{yin2023coco,yin2025dream}, we split Mutagenicity~\cite{kazius2005derivation}, NCI1~\cite{wale2008comparison}, FRANKENSTEIN~\cite{orsini2015graph}, and ogbg-molhiv~\cite{hu2020open} into four domains by node density or edge density. For MNIST~\cite{lecun2002gradient} and CIFAR-10~\cite{krizhevsky2009learning}, three domains are formed from superpixel $k$-nearest-neighbor graphs with different neighborhood sizes, so that the domains share node features and differ only in connectivity.
(2) \textit{Feature shifts.} We pair datasets with shared label semantics and different node attribute distributions, including PROTEINS/DD, COX2/COX2\_MD, and BZR/BZR\_MD, where the feature shift is accompanied by differences in graph size and density.
(3) \textit{Correlation shifts.} On Spurious-Motif~\cite{wu2022discovering}, the biased and unbiased variants differ in how strongly the label-determining motifs correlate with the base graphs, and we evaluate both transfer directions.
Table~\ref{tab:dataset} reports the dataset statistics, and Appendix~\ref{app:datasets} provides further details.

\textbf{Baselines.}
We compare \method{} with competitive baselines in three groups.
(1) General GNNs, including GCN~\cite{kipf2016semi}, GIN~\cite{xu2018how}, and GMT~\cite{baek2021accurate}.
(2) Graph distillation methods, including G-CRD~\cite{joshi2022representation}, MuGSI~\cite{yao2024mugsi}, TGS~\cite{wu2024teacher}, LAD-GNN~\cite{hong2024label}, AdaGMLP~\cite{lu2024adagmlp}, and ClustGDD~\cite{lai2025simple}.
(3) GDA methods, including SGDA~\cite{qiao2023semi}, StruRW~\cite{liu2023structural}, A2GNN~\cite{liu2024rethinking}, PA-BOTH~\cite{liu2024pairwise}, GAA~\cite{fang2025benefits}, and TDSS~\cite{chen2025smoothness}. Details of the baselines and implementation are in Appendices~\ref{app:baselines} and~\ref{app:implementation}.

\begin{table*}[t]
    \small
    \centering
    \caption{Graph classification results (in \%) under node and edge density domain shifts on the FRANKENSTEIN dataset, and 
    under correlation shifts (Corr. Shift) and feature shifts (source$\rightarrow$target). S, SB, P, D, C, CM, B, and BM denote Spurious-Motif, Spurious-Motif\_bias, PROTEINS, DD, COX2, COX2\_MD, BZR, and BZR\_MD. \textbf{Bold} indicates the best performance.}
    % \vspace{-0.1cm}
    \resizebox{1.0\textwidth}{!}{
    \begin{tabular}{l|c|c|c|c|c|c|c|c|c|c|c|c|c|c}
    \toprule
    \multirow{2}{*}{\textbf{Methods}}
    & \multicolumn{3}{c|}{\textbf{Node Shift}}
    & \multicolumn{3}{c|}{\textbf{Edge Shift}}
    & \multicolumn{2}{c|}{\textbf{Corr. Shift}} 
    & \multicolumn{6}{c}{\textbf{Feature Shift}} \\
    \cmidrule(lr){2-4} \cmidrule(lr){5-7} \cmidrule(lr){8-9} \cmidrule(lr){10-15}
    & F0$\rightarrow$F1 & F0$\rightarrow$F2 & F0$\rightarrow$F3
    & F0$\rightarrow$F1 & F0$\rightarrow$F2 & F0$\rightarrow$F3
    & S$\rightarrow$SB & SB$\rightarrow$S & P$\rightarrow$D & D$\rightarrow$P & C$\rightarrow$CM & CM$\rightarrow$C & B$\rightarrow$BM & BM$\rightarrow$B \\
    \midrule
    GCN & 32.7\scriptsize{$\pm$2.0} & 42.2\scriptsize{$\pm$1.7} & 54.5\scriptsize{$\pm$1.8} & 33.5\scriptsize{$\pm$2.1} & 44.6\scriptsize{$\pm$1.3} & 57.4\scriptsize{$\pm$2.0} & 33.8\scriptsize{$\pm$1.4} & 36.9\scriptsize{$\pm$1.0} & 48.9\scriptsize{$\pm$2.0} & 60.9\scriptsize{$\pm$2.3} & 51.0\scriptsize{$\pm$1.8} & 66.9\scriptsize{$\pm$1.8} & 48.7\scriptsize{$\pm$2.0} & 78.8\scriptsize{$\pm$1.7} \\
    GIN & 39.7\scriptsize{$\pm$1.8} & 54.7\scriptsize{$\pm$1.3} & 58.2\scriptsize{$\pm$1.9} & 40.1\scriptsize{$\pm$1.2} & 52.2\scriptsize{$\pm$2.2} & 58.3\scriptsize{$\pm$1.5} & 45.0\scriptsize{$\pm$2.4} & 48.2\scriptsize{$\pm$1.9} & 57.3\scriptsize{$\pm$2.2} & 61.9\scriptsize{$\pm$1.9} & 53.8\scriptsize{$\pm$2.5} & 55.6\scriptsize{$\pm$2.0} & 49.9\scriptsize{$\pm$2.4} & 79.2\scriptsize{$\pm$2.8} \\
    GMT & 47.3\scriptsize{$\pm$2.2} & 57.8\scriptsize{$\pm$1.4} & 39.5\scriptsize{$\pm$2.3} & 51.5\scriptsize{$\pm$1.0} & 55.4\scriptsize{$\pm$1.6} & 39.6\scriptsize{$\pm$1.6} & 35.7\scriptsize{$\pm$2.0} & 43.1\scriptsize{$\pm$2.0} & 59.5\scriptsize{$\pm$2.5} & 50.7\scriptsize{$\pm$2.2} & 49.3\scriptsize{$\pm$1.8} & 58.2\scriptsize{$\pm$2.0} & 50.2\scriptsize{$\pm$2.3} & 74.4\scriptsize{$\pm$1.8} \\
    
    \midrule
    G-CRD & 57.6{\scriptsize$\pm 1.7$} & 56.2{\scriptsize$\pm 0.9$} & 55.5{\scriptsize$\pm 1.9$} & 71.0{\scriptsize$\pm 0.8$} & 57.6{\scriptsize$\pm 2.6$} & 40.1{\scriptsize$\pm 2.5$} & 52.8{\scriptsize$\pm 2.4$} & 53.1{\scriptsize$\pm 2.5$} & 60.7{\scriptsize$\pm 1.4$} & 71.4{\scriptsize$\pm 1.2$} & 58.7{\scriptsize$\pm 1.5$} & 71.8{\scriptsize$\pm 1.6$} & 63.8{\scriptsize$\pm 1.9$} & 83.9{\scriptsize$\pm 1.1$} \\
    MuGSI & 54.3{\scriptsize$\pm 2.4$} & 51.5{\scriptsize$\pm 2.8$} & 53.8{\scriptsize$\pm 2.3$} & 48.2{\scriptsize$\pm 2.4$} & 52.3{\scriptsize$\pm 1.9$} & 56.4{\scriptsize$\pm 2.8$} & 49.5{\scriptsize$\pm 3.2$} & 50.2{\scriptsize$\pm 2.8$} & 57.5{\scriptsize$\pm 2.2$} & 70.1{\scriptsize$\pm 1.9$} & 55.9{\scriptsize$\pm 2.4$} & 69.4{\scriptsize$\pm 1.8$} & 61.8{\scriptsize$\pm 2.1$} & 80.5{\scriptsize$\pm 2.0$} \\
    TGS & 59.1{\scriptsize$\pm 1.7$} & 64.2{\scriptsize$\pm 1.6$} & 50.4{\scriptsize$\pm 2.5$} & 44.1{\scriptsize$\pm 1.7$} & 45.6{\scriptsize$\pm 2.0$} & 54.1{\scriptsize$\pm 1.6$} & 50.4{\scriptsize$\pm 2.6$} & 52.5{\scriptsize$\pm 2.4$} & 59.1{\scriptsize$\pm 1.6$} & 68.7{\scriptsize$\pm 1.5$} & 57.2{\scriptsize$\pm 1.9$} & 68.1{\scriptsize$\pm 2.3$} & 65.2{\scriptsize$\pm 1.8$} & 82.1{\scriptsize$\pm 1.4$} \\
    LAD-GNN & 60.5{\scriptsize$\pm 1.6$} & 65.4{\scriptsize$\pm 1.5$} & 54.1{\scriptsize$\pm 2.2$} & 68.1{\scriptsize$\pm 1.6$} & 43.8{\scriptsize$\pm 1.8$} & 53.8{\scriptsize$\pm 1.5$} & 51.9{\scriptsize$\pm 2.8$} & 54.0{\scriptsize$\pm 2.2$} & 59.5{\scriptsize$\pm 1.8$} & 71.8{\scriptsize$\pm 1.4$} & 56.5{\scriptsize$\pm 2.1$} & 72.3{\scriptsize$\pm 1.5$} & 64.9{\scriptsize$\pm 2.0$} & 81.8{\scriptsize$\pm 1.7$} \\
    AdaGMLP & 56.8{\scriptsize$\pm 2.0$} & 59.7{\scriptsize$\pm 2.1$} & 49.5{\scriptsize$\pm 2.6$} & 42.5{\scriptsize$\pm 2.0$} & 55.2{\scriptsize$\pm 2.4$} & 57.5{\scriptsize$\pm 2.1$} & 49.2{\scriptsize$\pm 2.5$} & 50.9{\scriptsize$\pm 2.1$} & 55.2{\scriptsize$\pm 2.3$} & 66.9{\scriptsize$\pm 2.1$} & 53.8{\scriptsize$\pm 2.5$} & 67.5{\scriptsize$\pm 2.0$} & 60.1{\scriptsize$\pm 2.4$} & 78.6{\scriptsize$\pm 2.2$} \\
    ClustGDD & 52.4{\scriptsize$\pm 2.5$} & 55.3{\scriptsize$\pm 2.4$} & 55.7{\scriptsize$\pm 2.0$} & 50.1{\scriptsize$\pm 2.5$} & 49.9{\scriptsize$\pm 2.0$} & 52.6{\scriptsize$\pm 2.4$} & 47.5{\scriptsize$\pm 3.5$} & 48.1{\scriptsize$\pm 3.4$} & 56.8{\scriptsize$\pm 2.0$} & 68.3{\scriptsize$\pm 1.8$} & 54.1{\scriptsize$\pm 2.2$} & 68.9{\scriptsize$\pm 2.4$} & 62.3{\scriptsize$\pm 1.6$} & 79.2{\scriptsize$\pm 2.5$} \\
    \midrule
    SGDA & 42.4{\scriptsize$\pm 1.6$} & 43.8{\scriptsize$\pm 2.4$} & 55.7{\scriptsize$\pm 1.9$} & 70.5{\scriptsize$\pm 1.9$} & 57.6{\scriptsize$\pm 2.5$} & 51.6{\scriptsize$\pm 2.1$} & 50.1{\scriptsize$\pm 2.9$} & 51.3{\scriptsize$\pm 2.7$} & 58.2{\scriptsize$\pm 1.7$} & 69.5{\scriptsize$\pm 2.0$} & 56.4{\scriptsize$\pm 1.8$} & 70.2{\scriptsize$\pm 2.1$} & 64.5{\scriptsize$\pm 2.3$} & 81.2{\scriptsize$\pm 1.9$} \\
    StruRW & 61.3{\scriptsize$\pm 1.2$} & 66.8{\scriptsize$\pm 1.5$} & 56.2{\scriptsize$\pm 2.3$} & 42.5{\scriptsize$\pm 1.2$} & 42.4{\scriptsize$\pm 2.1$} & 50.7{\scriptsize$\pm 1.5$} & 49.5{\scriptsize$\pm 2.7$} & 52.7{\scriptsize$\pm 2.5$} & 53.2{\scriptsize$\pm 2.1$} & 65.4{\scriptsize$\pm 1.8$} & 52.8{\scriptsize$\pm 2.3$} & 64.5{\scriptsize$\pm 1.9$} & 61.5{\scriptsize$\pm 2.4$} & 78.2{\scriptsize$\pm 1.5$} \\
    A2GNN & 57.6{\scriptsize$\pm 1.8$} & 44.2{\scriptsize$\pm 2.4$} & 56.2{\scriptsize$\pm 1.9$} & 41.8{\scriptsize$\pm 1.8$} & 44.2{\scriptsize$\pm 1.5$} & 58.1{\scriptsize$\pm 2.4$} & 44.7{\scriptsize$\pm 2.8$} & 46.5{\scriptsize$\pm 3.5$} & 55.8{\scriptsize$\pm 1.9$} & 66.1{\scriptsize$\pm 2.2$} & 51.4{\scriptsize$\pm 1.7$} & 66.2{\scriptsize$\pm 2.5$} & 59.8{\scriptsize$\pm 1.8$} & 77.6{\scriptsize$\pm 2.1$} \\
    PA-BOTH & 57.6{\scriptsize$\pm 2.0$} & 66.8{\scriptsize$\pm 1.3$} & 42.9{\scriptsize$\pm 3.0$} & 70.5{\scriptsize$\pm 2.0$} & 42.4{\scriptsize$\pm 1.7$} & 53.9{\scriptsize$\pm 1.3$} & 51.2{\scriptsize$\pm 2.3$} & 52.9{\scriptsize$\pm 2.1$} & 54.1{\scriptsize$\pm 2.5$} & 63.8{\scriptsize$\pm 1.6$} & 53.9{\scriptsize$\pm 2.0$} & 65.7{\scriptsize$\pm 1.4$} & 62.4{\scriptsize$\pm 2.2$} & 79.1{\scriptsize$\pm 1.8$} \\
    GAA & 57.6{\scriptsize$\pm 1.9$} & 66.8{\scriptsize$\pm 1.4$} & 52.5{\scriptsize$\pm 2.2$} & 71.0{\scriptsize$\pm 1.9$} & 57.6{\scriptsize$\pm 3.2$} & 55.8{\scriptsize$\pm 1.4$} & 47.8{\scriptsize$\pm 2.1$} & 49.3{\scriptsize$\pm 2.4$} & 52.7{\scriptsize$\pm 1.8$} & 64.5{\scriptsize$\pm 2.4$} & 50.6{\scriptsize$\pm 2.1$} & 63.9{\scriptsize$\pm 2.8$} & 60.1{\scriptsize$\pm 2.5$} & 76.8{\scriptsize$\pm 2.3$} \\
    TDSS & 55.8{\scriptsize$\pm 2.1$} & 66.8{\scriptsize$\pm 1.5$} & 56.2{\scriptsize$\pm 2.0$} & 41.3{\scriptsize$\pm 2.1$} & 58.1{\scriptsize$\pm 1.6$} & 58.1{\scriptsize$\pm 1.5$} & 52.1{\scriptsize$\pm 2.5$} & 53.4{\scriptsize$\pm 2.4$} & 56.4{\scriptsize$\pm 2.0$} & 67.2{\scriptsize$\pm 1.5$} & 54.2{\scriptsize$\pm 1.9$} & 67.1{\scriptsize$\pm 2.2$} & 63.5{\scriptsize$\pm 1.7$} & 80.4{\scriptsize$\pm 1.6$} \\
    \midrule
    \method {} & \textbf{64.5}{\scriptsize$\pm 1.4$} & \textbf{70.2}{\scriptsize$\pm 1.2$} & \textbf{59.4}{\scriptsize$\pm 1.8$} & \textbf{73.4}{\scriptsize$\pm 1.4$} & \textbf{61.2}{\scriptsize$\pm 1.6$} & \textbf{60.5}{\scriptsize$\pm 1.2$} & \textbf{54.3}{\scriptsize$\pm 2.1$} & \textbf{55.6}{\scriptsize$\pm 1.8$} & \textbf{62.4}{\scriptsize$\pm 1.5$} & \textbf{73.5}{\scriptsize$\pm 1.1$} & \textbf{60.2}{\scriptsize$\pm 1.4$} & \textbf{74.1}{\scriptsize$\pm 1.3$} & \textbf{69.5}{\scriptsize$\pm 1.6$} & \textbf{85.4}{\scriptsize$\pm 1.0$} \\
    \bottomrule
    \end{tabular}}

    \label{tab:combined_shifts}
    \end{table*}

\begin{table*}[t]
    \small
    \centering
    \caption{Image classification results (in \%) on MNIST and CIFAR-10 under edge density domain shifts. \textbf{Bold} results indicate the best performance.}
    \resizebox{1.0\textwidth}{!}{
    \begin{tabular}{l|c|c|c|c|c|c|c|c|c|c|c|c}
    \toprule
    \textbf{Methods}
    & \multicolumn{6}{c|}{\textbf{MNIST}}
    & \multicolumn{6}{c}{\textbf{CIFAR-10}} \\
    \cmidrule(lr){2-7} \cmidrule(lr){8-13}
    & S0$\rightarrow$S1 & S0$\rightarrow$S2 & S1$\rightarrow$S0 & S1$\rightarrow$S2 & S2$\rightarrow$S0 & S2$\rightarrow$S1
    & C0$\rightarrow$C1 & C0$\rightarrow$C2 & C1$\rightarrow$C0 & C1$\rightarrow$C2 & C2$\rightarrow$C0 & C2$\rightarrow$C1 \\
    \midrule
    GCN
& 82.4{\scriptsize$\pm 0.3$} & 84.7{\scriptsize$\pm 0.3$}
& 60.9{\scriptsize$\pm 1.9$} & 70.5{\scriptsize$\pm 4.0$}
& 81.5{\scriptsize$\pm 1.0$} & 83.2{\scriptsize$\pm 0.9$}
& 44.9{\scriptsize$\pm 0.4$} & 46.5{\scriptsize$\pm 0.2$}
& 42.0{\scriptsize$\pm 1.2$} & 43.0{\scriptsize$\pm 1.8$}
& 44.1{\scriptsize$\pm 0.3$} & 45.2{\scriptsize$\pm 0.3$} \\

    GIN
& 86.2{\scriptsize$\pm 1.5$} & 79.8{\scriptsize$\pm 2.3$}
& 81.4{\scriptsize$\pm 2.6$} & 83.9{\scriptsize$\pm 2.1$}
& 83.2{\scriptsize$\pm 1.6$} & 84.8{\scriptsize$\pm 2.2$}
& 44.4{\scriptsize$\pm 2.2$} & 43.1{\scriptsize$\pm 1.1$}
& 43.3{\scriptsize$\pm 2.1$} & 44.7{\scriptsize$\pm 1.7$}
& 42.4{\scriptsize$\pm 1.1$} & 42.0{\scriptsize$\pm 2.8$} \\

    GMT
& 83.3{\scriptsize$\pm 1.3$} & 83.5{\scriptsize$\pm 1.4$}
& 58.2{\scriptsize$\pm 0.8$} & 68.3{\scriptsize$\pm 2.4$}
& 80.9{\scriptsize$\pm 1.4$} & 84.6{\scriptsize$\pm 1.1$}
& 52.2{\scriptsize$\pm 0.4$} & 53.6{\scriptsize$\pm 0.1$}
& 44.1{\scriptsize$\pm 2.7$} & 50.0{\scriptsize$\pm 7.6$}
& 49.6{\scriptsize$\pm 0.7$} & 50.0{\scriptsize$\pm 0.9$} \\
    \midrule
    G-CRD       & 83.7{\scriptsize$\pm 0.3$} & 84.1{\scriptsize$\pm 1.2$} & 83.3{\scriptsize$\pm 0.8$} & 83.5{\scriptsize$\pm 1.5$} & 84.1{\scriptsize$\pm 0.7$} & 83.1{\scriptsize$\pm 1.1$} & 46.1{\scriptsize$\pm 0.4$} & 45.4{\scriptsize$\pm 1.2$} & 46.7{\scriptsize$\pm 0.2$} & 46.4{\scriptsize$\pm 1.8$} & 46.5{\scriptsize$\pm 0.3$} & 47.3{\scriptsize$\pm 0.9$} \\
    MuGSI       & 86.0{\scriptsize$\pm 1.5$} & 86.6{\scriptsize$\pm 2.6$} & 85.1{\scriptsize$\pm 2.3$} & 83.1{\scriptsize$\pm 1.6$} & 82.9{\scriptsize$\pm 2.1$} & 86.5{\scriptsize$\pm 2.2$} & 50.3{\scriptsize$\pm 2.2$} & 49.4{\scriptsize$\pm 2.1$} & 51.2{\scriptsize$\pm 1.1$} & 49.1{\scriptsize$\pm 1.7$} & 51.0{\scriptsize$\pm 1.1$} & 49.1{\scriptsize$\pm 2.8$} \\
    TGS         & 87.5{\scriptsize$\pm 1.3$} & 88.9{\scriptsize$\pm 0.8$} & 88.0{\scriptsize$\pm 1.4$} & 87.1{\scriptsize$\pm 1.4$} & 88.0{\scriptsize$\pm 2.4$} & 88.0{\scriptsize$\pm 1.1$} & 50.9{\scriptsize$\pm 0.4$} & 52.6{\scriptsize$\pm 2.7$} & 50.8{\scriptsize$\pm 0.1$} & 52.4{\scriptsize$\pm 1.6$} & 51.8{\scriptsize$\pm 0.7$} & 52.0{\scriptsize$\pm 0.9$} \\
    LAD-GNN     & 84.8{\scriptsize$\pm 2.0$} & 84.6{\scriptsize$\pm 1.5$} & 83.5{\scriptsize$\pm 1.8$} & 84.3{\scriptsize$\pm 1.7$} & 84.2{\scriptsize$\pm 1.5$} & 84.1{\scriptsize$\pm 2.3$} & 49.0{\scriptsize$\pm 0.5$} & 46.6{\scriptsize$\pm 3.0$} & 49.5{\scriptsize$\pm 1.9$} & 47.9{\scriptsize$\pm 0.8$} & 49.2{\scriptsize$\pm 0.4$} & 47.2{\scriptsize$\pm 0.6$} \\
    AdaGMLP     & 80.2{\scriptsize$\pm 2.8$} & 83.1{\scriptsize$\pm 2.0$} & 82.4{\scriptsize$\pm 1.4$} & 80.7{\scriptsize$\pm 1.7$} & 82.0{\scriptsize$\pm 1.6$} & 80.1{\scriptsize$\pm 2.5$} & 48.7{\scriptsize$\pm 2.0$} & 49.8{\scriptsize$\pm 2.3$} & 49.5{\scriptsize$\pm 1.5$} & 50.4{\scriptsize$\pm 1.6$} & 48.3{\scriptsize$\pm 2.1$} & 48.8{\scriptsize$\pm 2.0$} \\
    ClustGDD    & 79.3{\scriptsize$\pm 0.3$} & 80.3{\scriptsize$\pm 4.1$} & 81.0{\scriptsize$\pm 1.5$} & 80.0{\scriptsize$\pm 1.5$} & 81.7{\scriptsize$\pm 2.0$} & 80.7{\scriptsize$\pm 1.3$} & 43.5{\scriptsize$\pm 3.4$} & 41.8{\scriptsize$\pm 2.0$} & 44.4{\scriptsize$\pm 1.2$} & 42.3{\scriptsize$\pm 1.5$} & 44.7{\scriptsize$\pm 0.9$} & 42.2{\scriptsize$\pm 1.5$} \\
    \midrule
    SGDA        & 77.9{\scriptsize$\pm 1.0$} & 77.0{\scriptsize$\pm 0.7$} & 76.8{\scriptsize$\pm 1.3$} & 75.4{\scriptsize$\pm 1.2$} & 79.6{\scriptsize$\pm 1.3$} & 76.9{\scriptsize$\pm 1.1$} & 44.3{\scriptsize$\pm 1.4$} & 44.8{\scriptsize$\pm 0.8$} & 45.6{\scriptsize$\pm 1.3$} & 45.7{\scriptsize$\pm 1.4$} & 43.8{\scriptsize$\pm 1.4$} & 45.5{\scriptsize$\pm 1.8$} \\
    StruRW      & 76.6{\scriptsize$\pm 1.3$} & 78.3{\scriptsize$\pm 4.8$} & 80.6{\scriptsize$\pm 1.0$} & 79.1{\scriptsize$\pm 1.2$} & 80.6{\scriptsize$\pm 1.9$} & 78.3{\scriptsize$\pm 1.5$} & 45.4{\scriptsize$\pm 1.5$} & 44.6{\scriptsize$\pm 1.4$} & 46.1{\scriptsize$\pm 1.2$} & 42.8{\scriptsize$\pm 2.8$} & 44.5{\scriptsize$\pm 1.3$} & 43.6{\scriptsize$\pm 1.0$} \\
    A2GNN       & 85.1{\scriptsize$\pm 1.7$} & 86.2{\scriptsize$\pm 0.9$} & 84.7{\scriptsize$\pm 0.8$} & 85.8{\scriptsize$\pm 0.7$} & 84.9{\scriptsize$\pm 2.2$} & 84.9{\scriptsize$\pm 1.4$} & 47.6{\scriptsize$\pm 1.7$} & 47.8{\scriptsize$\pm 1.6$} & 48.7{\scriptsize$\pm 2.2$} & 48.7{\scriptsize$\pm 1.5$} & 48.6{\scriptsize$\pm 0.8$} & 49.0{\scriptsize$\pm 1.3$} \\
    PA-BOTH     & 75.9{\scriptsize$\pm 1.1$} & 77.5{\scriptsize$\pm 1.3$} & 80.4{\scriptsize$\pm 1.7$} & 77.7{\scriptsize$\pm 1.1$} & 81.4{\scriptsize$\pm 2.0$} & 77.1{\scriptsize$\pm 1.0$} & 46.9{\scriptsize$\pm 1.1$} & 45.1{\scriptsize$\pm 1.7$} & 46.5{\scriptsize$\pm 1.3$} & 45.3{\scriptsize$\pm 1.4$} & 45.4{\scriptsize$\pm 0.9$} & 43.3{\scriptsize$\pm 2.2$} \\
    GAA         & 91.4{\scriptsize$\pm 1.0$} & 91.9{\scriptsize$\pm 0.9$} & 91.7{\scriptsize$\pm 2.1$} & 91.1{\scriptsize$\pm 2.0$} & 91.5{\scriptsize$\pm 1.7$} & 91.2{\scriptsize$\pm 1.4$} & 51.1{\scriptsize$\pm 0.9$} & 52.6{\scriptsize$\pm 1.5$} & 53.4{\scriptsize$\pm 0.8$} & 54.0{\scriptsize$\pm 2.2$} & 52.9{\scriptsize$\pm 1.1$} & 55.2{\scriptsize$\pm 1.3$} \\
    TDSS        & 85.7{\scriptsize$\pm 1.7$} & 85.6{\scriptsize$\pm 1.1$} & 85.0{\scriptsize$\pm 0.9$} & 85.2{\scriptsize$\pm 0.8$} & 85.1{\scriptsize$\pm 0.6$} & 84.9{\scriptsize$\pm 0.7$} & 47.9{\scriptsize$\pm 1.2$} & 48.4{\scriptsize$\pm 0.7$} & 49.5{\scriptsize$\pm 0.8$} & 47.2{\scriptsize$\pm 2.2$} & 48.2{\scriptsize$\pm 1.3$} & 48.4{\scriptsize$\pm 1.5$} \\
    \midrule
    \method{}   & \textbf{93.0}{\scriptsize$\pm 0.8$} & \textbf{93.3}{\scriptsize$\pm 1.1$} & \textbf{92.8}{\scriptsize$\pm 0.9$} & \textbf{92.3}{\scriptsize$\pm 1.3$} & \textbf{93.1}{\scriptsize$\pm 0.7$} & \textbf{92.8}{\scriptsize$\pm 1.0$} & \textbf{53.3}{\scriptsize$\pm 1.2$} & \textbf{54.5}{\scriptsize$\pm 0.8$} & \textbf{55.4}{\scriptsize$\pm 1.1$} & \textbf{55.7}{\scriptsize$\pm 0.9$} & \textbf{54.7}{\scriptsize$\pm 1.4$} & \textbf{56.9}{\scriptsize$\pm 0.7$} \\
    \bottomrule
    \end{tabular}}
    \label{tab:mnist_cifar10_edge}
\end{table*}

\subsection{Performance Comparison}
\label{sec:performance_main}

We present the results of \method{} and all baselines under structural, correlation, and feature shifts in Table~\ref{tab:combined_shifts} and on the image-derived benchmarks in Table~\ref{tab:mnist_cifar10_edge}. From these tables, we make three observations.
(1) Graph distillation methods outperform general GNNs in most cases, but both groups remain limited by their reliance on the source domain. General GNNs yield the lowest performance overall and degrade most severely under correlation shift, since training on source graphs alone captures domain-specific connectivity and spurious correlations that do not hold in the target domain. Knowledge distillation methods form the strongest baselines under feature and correlation shifts, indicating that distilled supervision regularizes the learned representations against source-specific noise. Their gains are uneven across structural shifts, however, because the distilled knowledge is still extracted from the source domain alone. Dataset distillation without target information brings only modest improvements over general GNNs, showing that compressing source data does not by itself reduce the domain gap.
(2) GDA methods benefit from unlabeled target graphs, but their effectiveness depends on whether the shift matches the inductive bias of their adaptation mechanism. Structure-oriented methods perform well when the shift lies mainly in graph density, yet their performance varies sharply across benchmarks with different structural characteristics, and the strongest GDA baseline changes from one benchmark to another. Attribute-driven alignment is effective when node attributes carry rich class information, but loses its advantage when the attribute distributions themselves differ across domains. Under feature shifts in particular, most GDA methods trail the knowledge distillation methods, since a fixed model-level alignment mechanism cannot compensate for attribute discrepancies, and no single GDA method remains competitive across all shift types.
(3) \method{} achieves the best performance across all types of shift and on both image benchmarks, demonstrating its effectiveness under diverse domain shifts. The superior performance can be attributed to three factors. First, semantic distillation keeps the prototypes class-discriminative and retains the robustness that distilled supervision provides. Second, the correspondence-free geometric alignment moves the prototype topology toward the target domain through collection-level statistics, so the training graphs adapt to whichever structural shift is present without relying on a fixed alignment mechanism. Third, the Dirichlet-energy alignment couples node features with topology and allows target node features to shape the prototypes, which accounts for the clear gains under feature shifts where structure-oriented GDA methods fall short. Moreover, conditioning the prototypes on the target domain weakens the reliance on source-specific spurious correlations, leading to consistent improvements under correlation shift. More results on other datasets can be found in Appendix~\ref{app:full_results}.
\begin{figure}[t]
    \centering
    \begin{subfigure}[b]{0.24\linewidth}
        \centering
        \includegraphics[width=\linewidth]{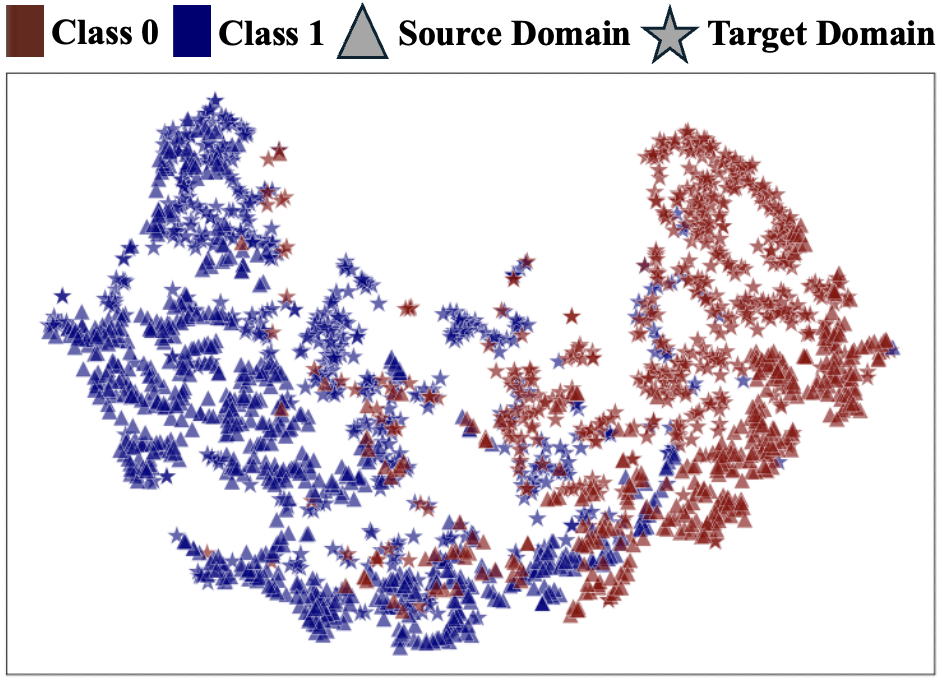}
        \caption{\method{}}
    \end{subfigure}\hfill
    \begin{subfigure}[b]{0.24\linewidth}
        \centering
        \includegraphics[width=\linewidth]{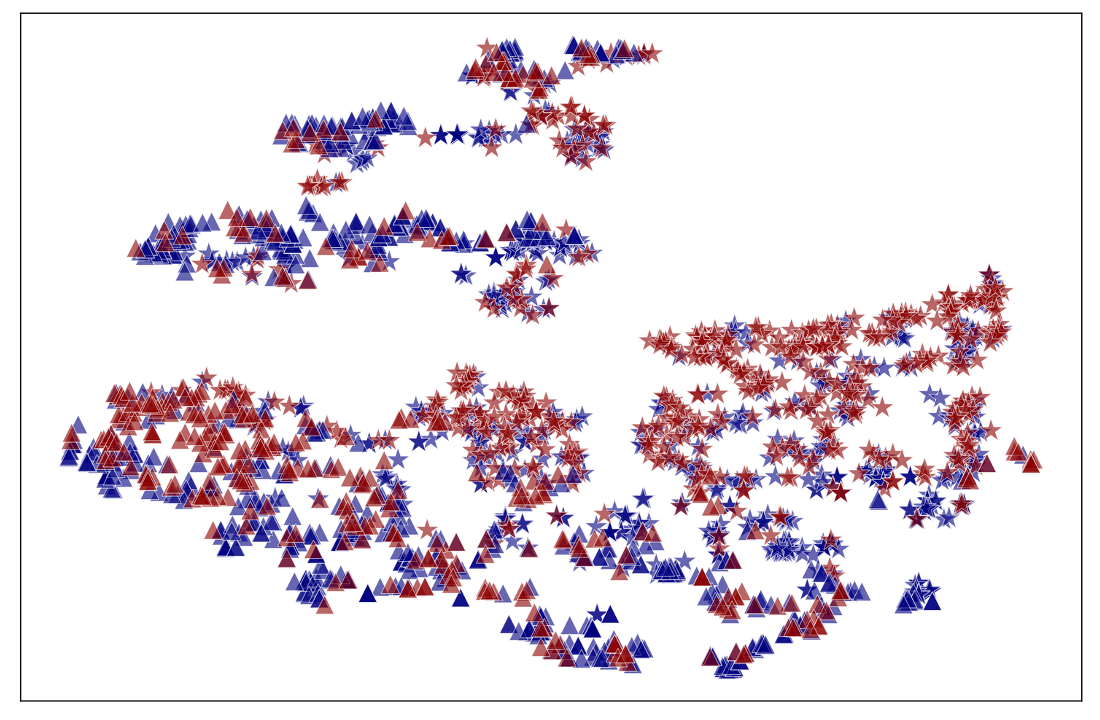}
        \caption{ClustGDD}
    \end{subfigure}\hfill
    \begin{subfigure}[b]{0.24\linewidth}
        \centering
        \includegraphics[width=\linewidth]{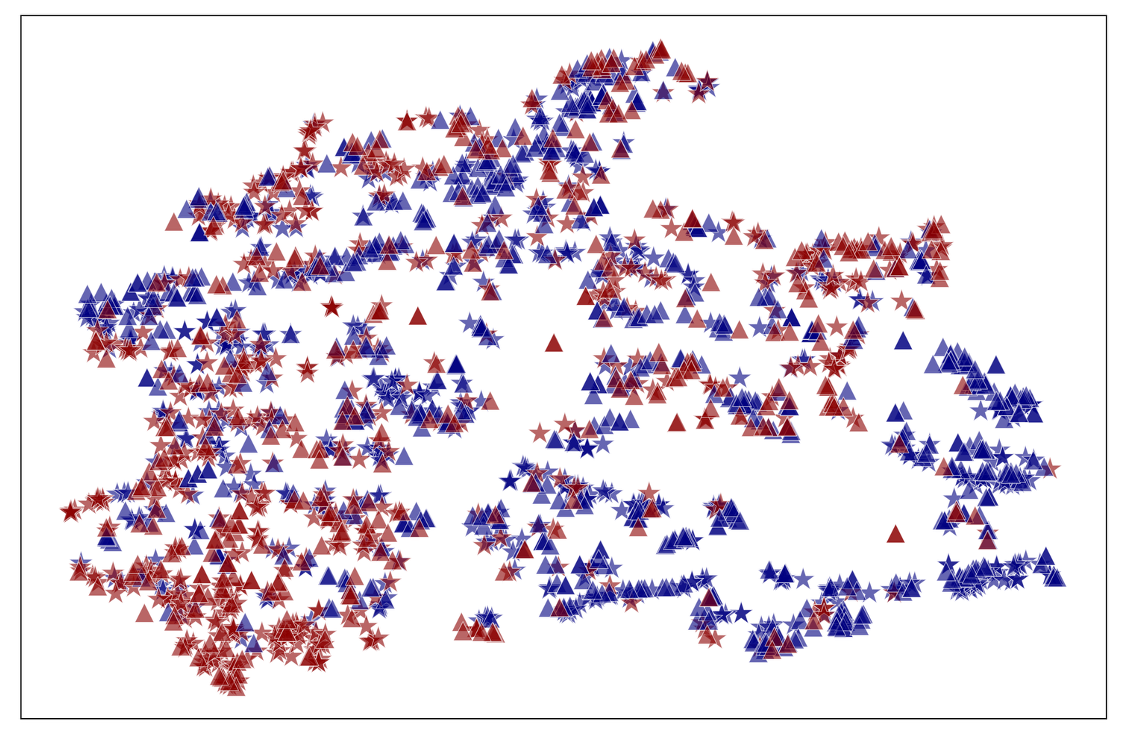}
        \caption{GAA}
    \end{subfigure}\hfill
    \begin{subfigure}[b]{0.24\linewidth}
        \centering
        \includegraphics[width=\linewidth]{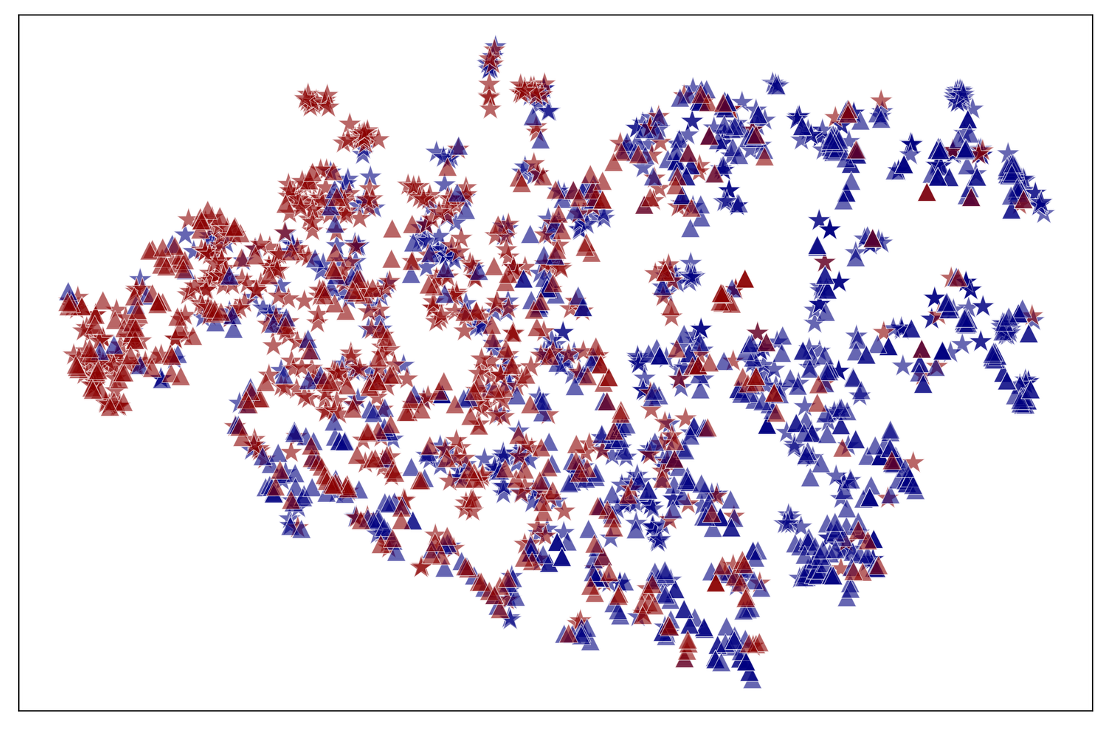}
        \caption{TDSS}
    \end{subfigure}

    \caption{t-SNE visualization of the graph representations learned on Mutagenicity.}
    \label{fig:tsne}
    
\end{figure}

We further visualize the graph representations learned on Mutagenicity with t-SNE in Fig.~\ref{fig:tsne}. \method{} separates the two classes into compact and distinct regions, and source and target graphs interleave within each region. In comparison, the source-only distillation method ClustGDD fragments the embedding space into small clusters that each mix both classes, the attribute-driven GAA mixes the classes within elongated clusters, and the smoothness-based TDSS preserves only a coarse class ordering around a broad mixed region. These patterns indicate that semantic distillation keeps the classes apart, while target conditioning moves the training graphs toward the target domain and draws same-class source and target graphs together.

\subsection{Ablation Study}
\label{sec:ablation}
\begin{table*}[t]
\small
\centering
\caption{Ablation results on FRANKENSTEIN (source $\rightarrow$ target). \textbf{Bold} indicates the highest reported mean.}
\resizebox{1.0\textwidth}{!}{
\begin{tabular}{l|c|c|c|c|c|c|c|c|c|c|c|c}
\toprule
\textbf{Methods}
& F0$\rightarrow$F1 & F1$\rightarrow$F0
& F0$\rightarrow$F2 & F2$\rightarrow$F0
& F0$\rightarrow$F3 & F3$\rightarrow$F0
& F1$\rightarrow$F2 & F2$\rightarrow$F1
& F1$\rightarrow$F3 & F3$\rightarrow$F1
& F2$\rightarrow$F3 & F3$\rightarrow$F2 \\
\midrule 
\method{} w/o SE
& 58.2 & 44.1
& 49.5 & 45.3
& 46.8 & 42.7
& 50.2 & 59.1
& 48.6 & 60.5
& 51.4 & 54.2 \\

\method{} w/o SP
& 67.5 & 52.4
& 56.1 & 52.6
& 53.2 & 50.1
& 58.7 & 65.4
& 56.8 & 68.3
& 59.5 & 61.7 \\

\method{} w/o GE
& 63.8 & 49.7
& 53.4 & 51.0
& 54.9 & 47.9
& 55.6 & 64.2
& 54.1 & 65.2
& 55.8 & 59.4 \\

\method{} w/o TG
& 70.1 & 54.5
& 58.6 & 55.4
& 57.8 & 49.3
& 60.9 & 69.8
& 59.2 & 71.6
& 61.2 & 64.5 \\

\method{}
& \textbf{73.4} & \textbf{56.8}
& \textbf{61.2} & \textbf{57.2}
& \textbf{60.5} & \textbf{54.6}
& \textbf{63.1} & \textbf{72.5}
& \textbf{61.4} & \textbf{74.1}
& \textbf{63.8} & \textbf{66.9} \\
\bottomrule
\end{tabular}
}
\label{tab:ablation_frankenstein}
\end{table*}

We compare \method{} with four ablation variants. (1) \method{} w/o SE removes the semantic loss $\mathcal L_{\mathrm{sem}}$, so the prototypes are optimized by the target alignment terms alone. (2) \method{} w/o SP removes the Dirichlet-energy alignment loss $\mathcal L_{\mathrm{spec}}$. (3) \method{} w/o GE removes the geometric alignment loss $\mathcal L_{\mathrm{geo}}$. (4) \method{} w/o TG replaces the deployed predictor with the terminal proxy GNN. Table~\ref{tab:ablation_frankenstein} reports the results on FRANKENSTEIN, where the full model outperforms all variants across transfer tasks. Removing the semantic loss causes the most severe degradation, since the alignment terms carry no label information and target statistics alone cannot keep the prototypes class-discriminative. Removing either alignment term also degrades performance, showing that connectivity statistics and feature-dependent statistics capture complementary aspects of the target domain. Geometric alignment generally contributes more than energy alignment, which matches the nature of edge density shifts that primarily alter connectivity, while energy alignment supplies the feature-dependent signal that connectivity statistics cannot express. Replacing the deployed predictor with the terminal proxy leads to a smaller but consistent decline, because the proxy is fitted to the prototypes before the last outer update, whereas the deployed predictor is trained on the final prototypes. Ablation results on the other datasets are reported in Appendix~\ref{app:ablation_study}.

\begin{figure}[t]
    \centering
    \begin{subfigure}[b]{0.24\linewidth}
        \centering
        \includegraphics[width=\linewidth]{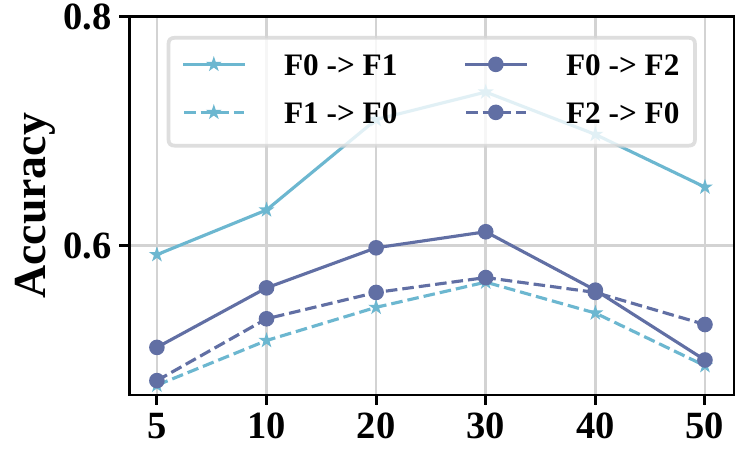}
        \caption{FRANKENSTEIN}
    \end{subfigure}\hfill
    \begin{subfigure}[b]{0.24\linewidth}
        \centering
        \includegraphics[width=\linewidth]{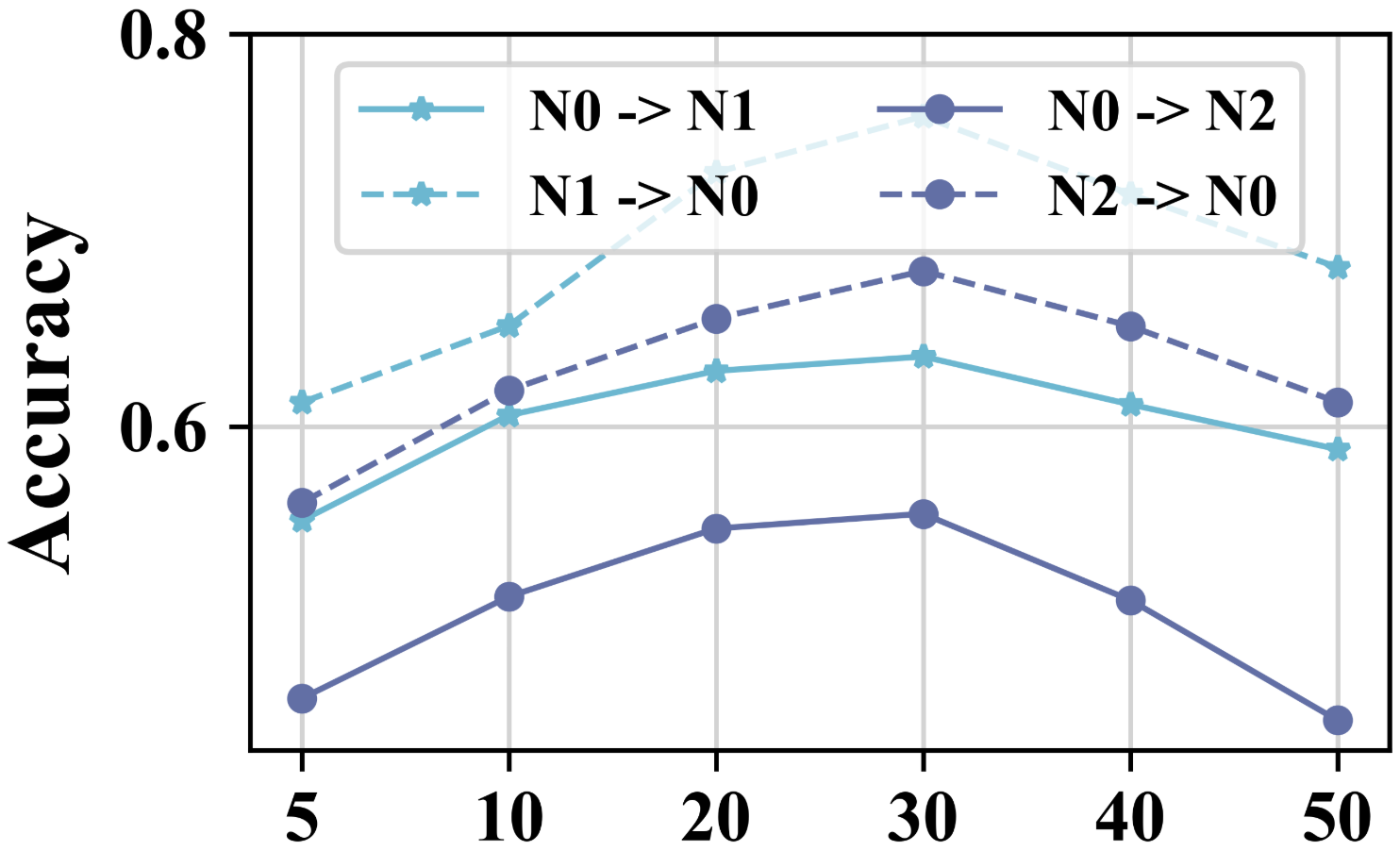}
        \caption{NCI1}
    \end{subfigure}\hfill
    \begin{subfigure}[b]{0.24\linewidth}
        \centering
        \includegraphics[width=\linewidth]{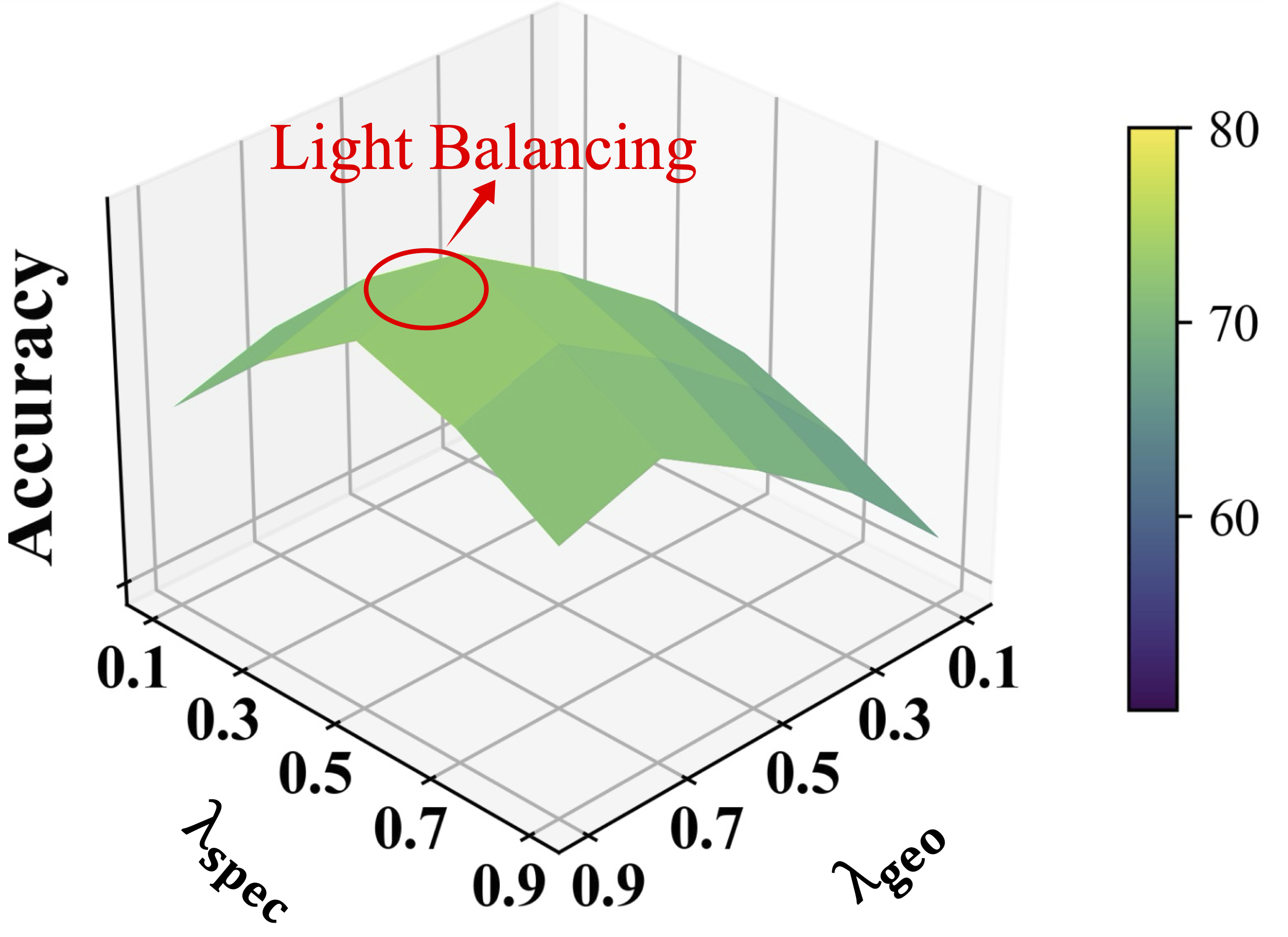}
        \caption{FRANKENSTEIN}
    \end{subfigure}\hfill
    \begin{subfigure}[b]{0.24\linewidth}
        \centering
        \includegraphics[width=\linewidth]{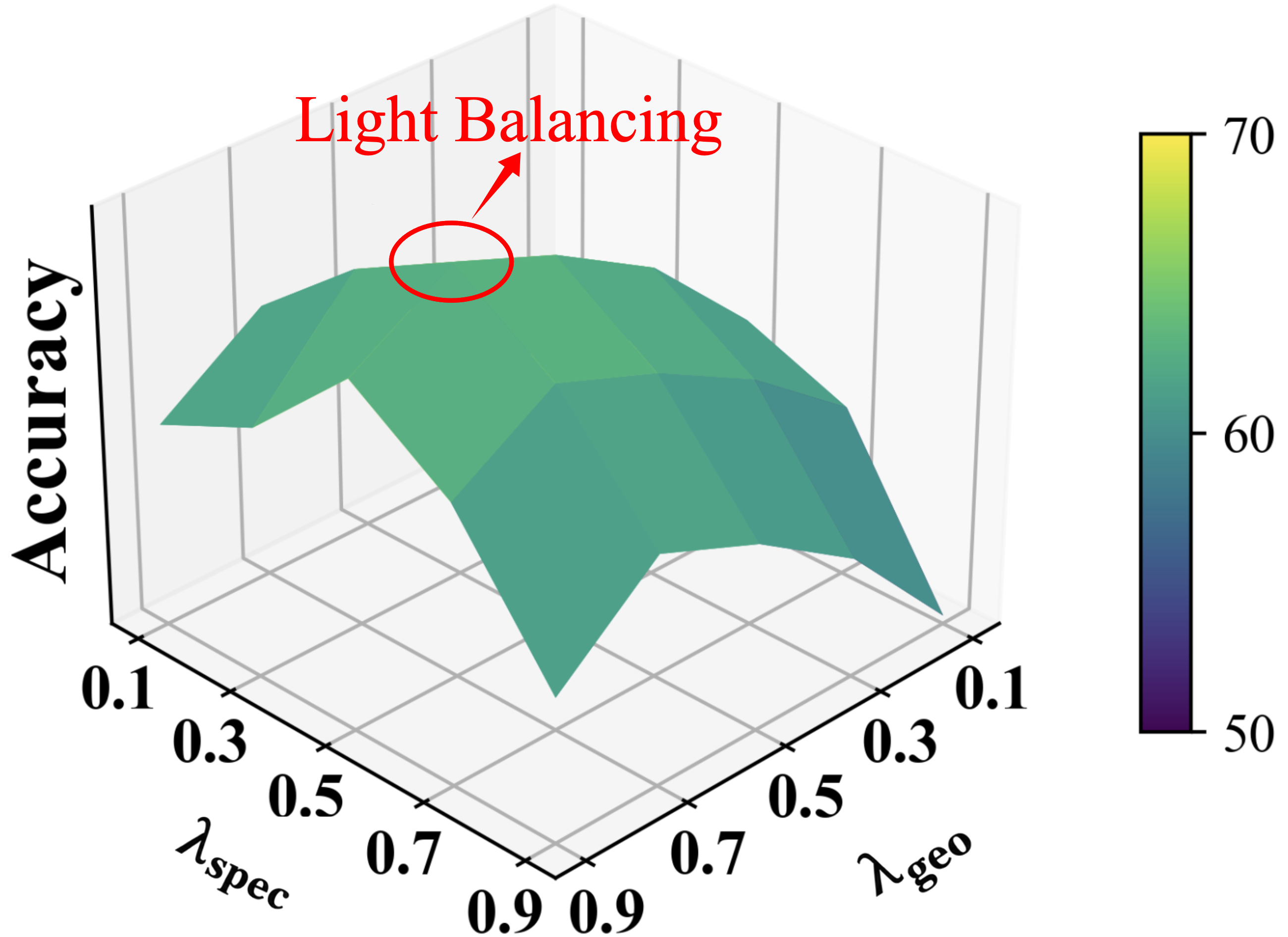}
        \caption{NCI1}
    \end{subfigure}

    \caption{Sensitivity to the number of prototypes $K$ (a,b) and to the alignment weights $\lambda_{\mathrm{geo}}$ and $\lambda_{\mathrm{spec}}$ (c,d) on FRANKENSTEIN and NCI1.}
    \label{fig:sensitivity}
    \vspace{-0.3cm}
\end{figure}

\subsection{Sensitivity Study}
\label{sec:sensitivity}

We study the sensitivity of \method{} to the number of prototypes $K$ and the alignment weights $\lambda_{\mathrm{geo}}$ and $\lambda_{\mathrm{spec}}$. As shown in Fig.~\ref{fig:sensitivity}(a,b), performance on FRANKENSTEIN and NCI1 first rises as $K$ increases, peaks at a moderate size, and then declines, with a similar trend across transfers. A small prototype set lacks the capacity to represent both the source classes and the target statistics. A large prototype set is harder for the proxy GNN to fit within the fixed inner budget, which weakens the semantic signal received by each prototype and also increases the computational cost. Fig.~\ref{fig:sensitivity}(c,d) shows that moderate alignment weights yield the best performance, and performance degrades toward the boundary of the grid. Small weights leave the prototypes weakly conditioned on the target domain, while large weights let statistic matching dominate the semantic objective, which is the only source of class information. Performance remains stable over a broad range of moderate weights, so \method{} does not require precise tuning. We set $K=30$, $\lambda_{\mathrm{geo}}=0.7$, and $\lambda_{\mathrm{spec}}=0.5$ by default, and Appendix~\ref{app:sensitive_analysis} reports results on the other datasets.

\subsection{Visualization of Learned Prototypes}
\label{sec:vis_prototypes}
\begin{figure}[t]
    \centering
    \begin{subfigure}[b]{0.155\linewidth}
        \centering
        \includegraphics[width=\linewidth]{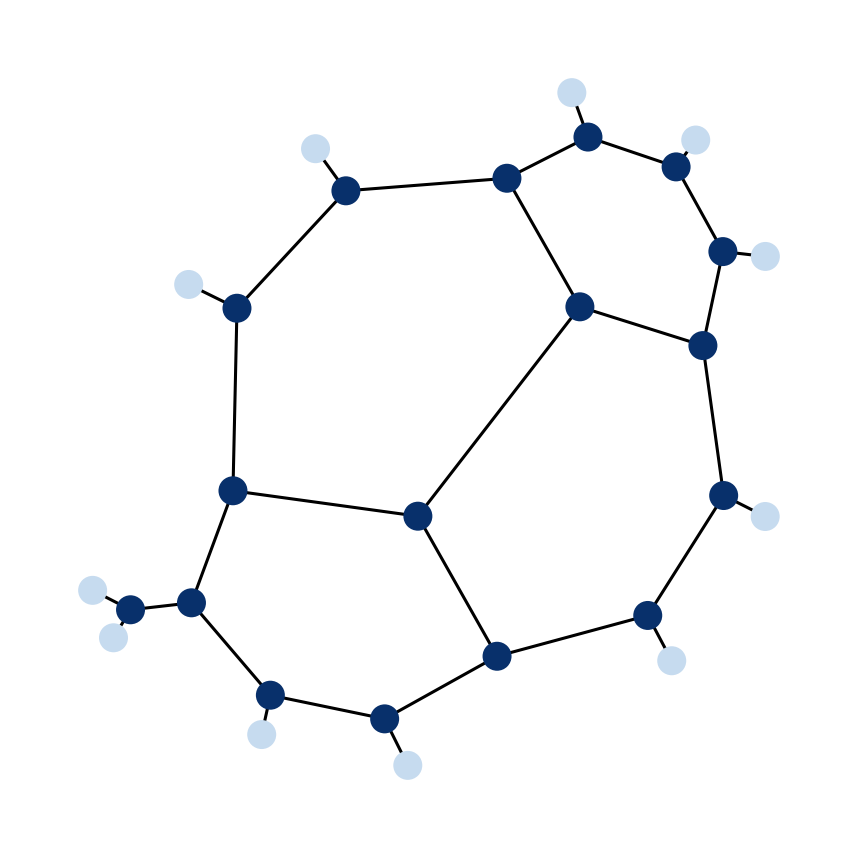}
        \caption{$\Omega=13.23$}
    \end{subfigure}\hfill
    \begin{subfigure}[b]{0.155\linewidth}
        \centering
        \includegraphics[width=\linewidth]{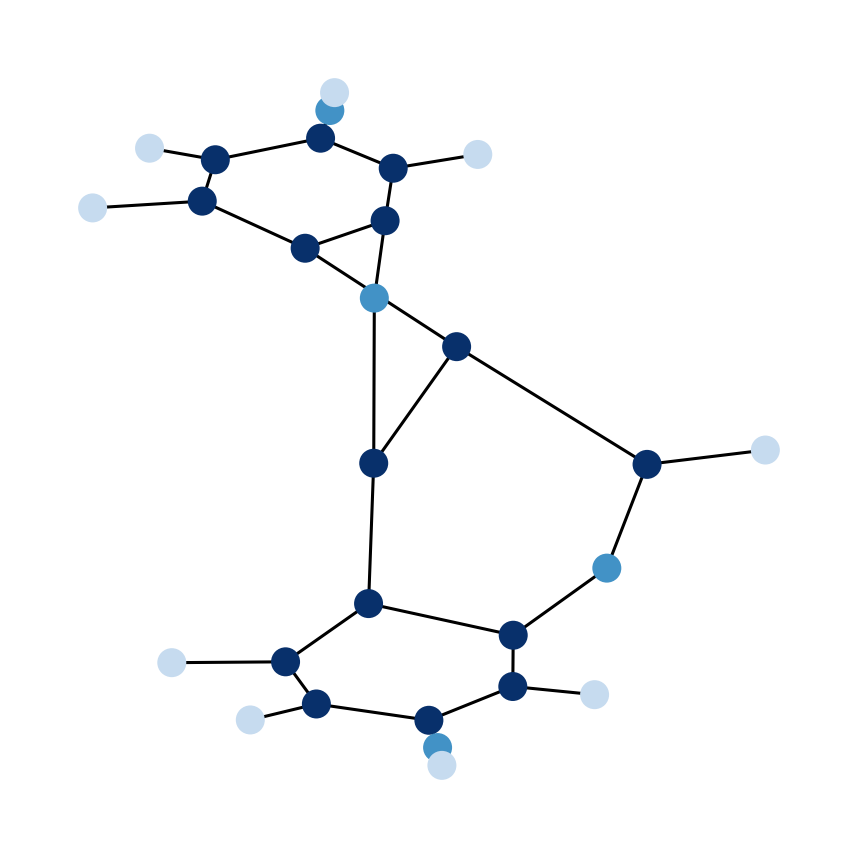}
        \caption{$\Omega=17.33$}
    \end{subfigure}\hfill
    \begin{subfigure}[b]{0.155\linewidth}
        \centering
        \includegraphics[width=\linewidth]{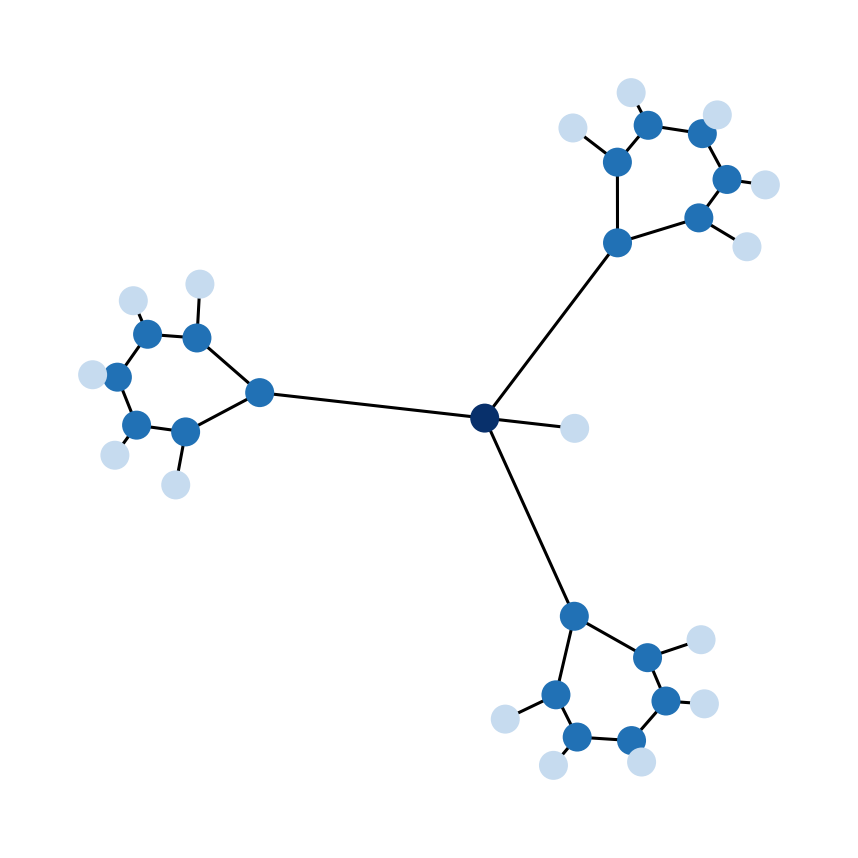}
        \caption{$\Omega=21.27$}
    \end{subfigure}\hfill
    \begin{subfigure}[b]{0.155\linewidth}
        \centering
        \includegraphics[width=\linewidth]{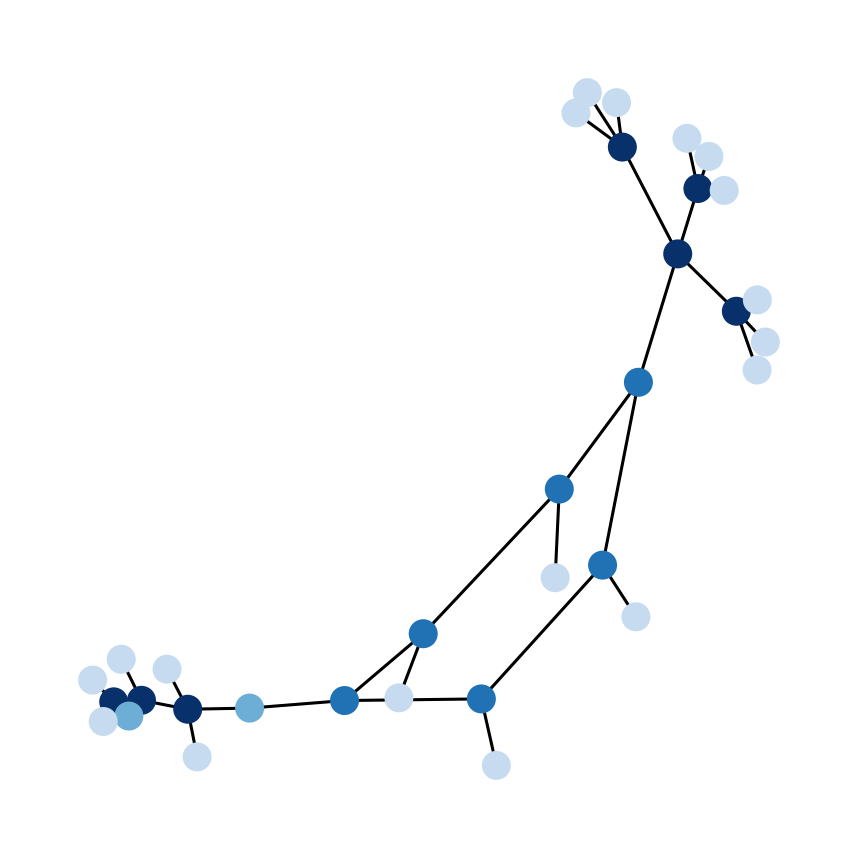}
        \caption{$\Omega=25.92$}
    \end{subfigure}\hfill
    \begin{subfigure}[b]{0.155\linewidth}
        \centering
        \includegraphics[width=\linewidth]{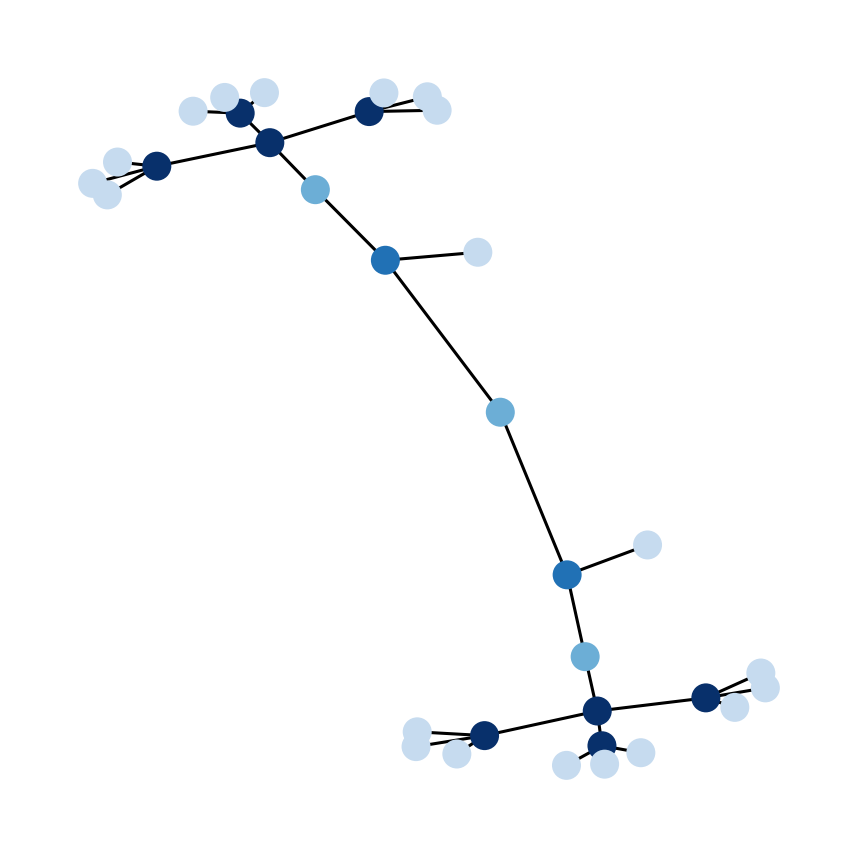}
        \caption{$\Omega=30.42$}
    \end{subfigure}\hfill
    \begin{subfigure}[b]{0.155\linewidth}
        \centering
        \includegraphics[width=\linewidth]{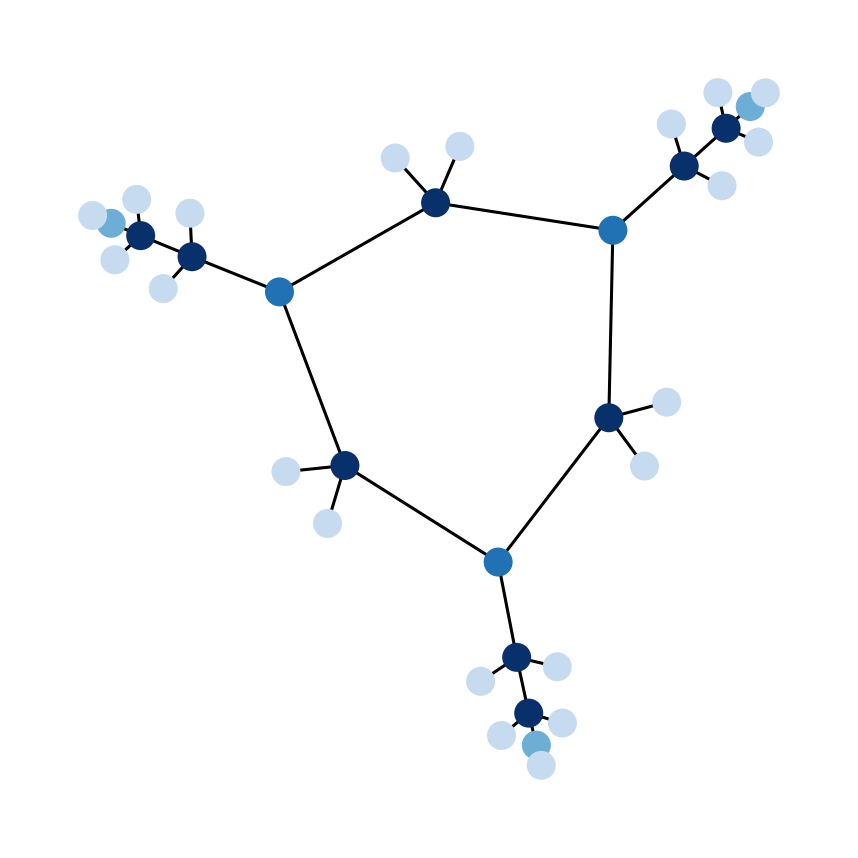}
        \caption{$\Omega=34.53$}
    \end{subfigure}

    \caption{The learned prototypes spanning the observed range of Dirichlet energy $\Omega$.}
    \label{fig:prototypes}
    \vspace{-0.3cm}
\end{figure}

Since the alignment objectives constrain the prototypes only as a collection, a natural question is whether they collapse to a single common template. To examine this, Fig.~\ref{fig:prototypes} visualizes six prototypes that span the observed range of Dirichlet energy. The lowest-energy prototype forms a compact multi-ring structure, prototypes with intermediate energy connect cyclic substructures through short chains or arms, and higher-energy prototypes turn into bent or elongated chains or a ring with peripheral branches. This diversity in ring count, chain length, and branching pattern arises because the alignment objectives match only the collection means and leave each prototype free to encode a different class-relevant substructure. The emerging cyclic and chain motifs are characteristic of molecular graphs, suggesting that the semantic objective grounds individual prototypes in the structure of the source classes, while the alignment terms adjust their statistics as a whole. Appendix~\ref{app:vis} presents the complete set of learned prototypes.

\subsection{Target Distribution Alignment}
\label{sec:distribution}

\begin{wrapfigure}{r}{0.5\linewidth}
    \vspace{-12pt}
    \centering
    \includegraphics[width=\linewidth]{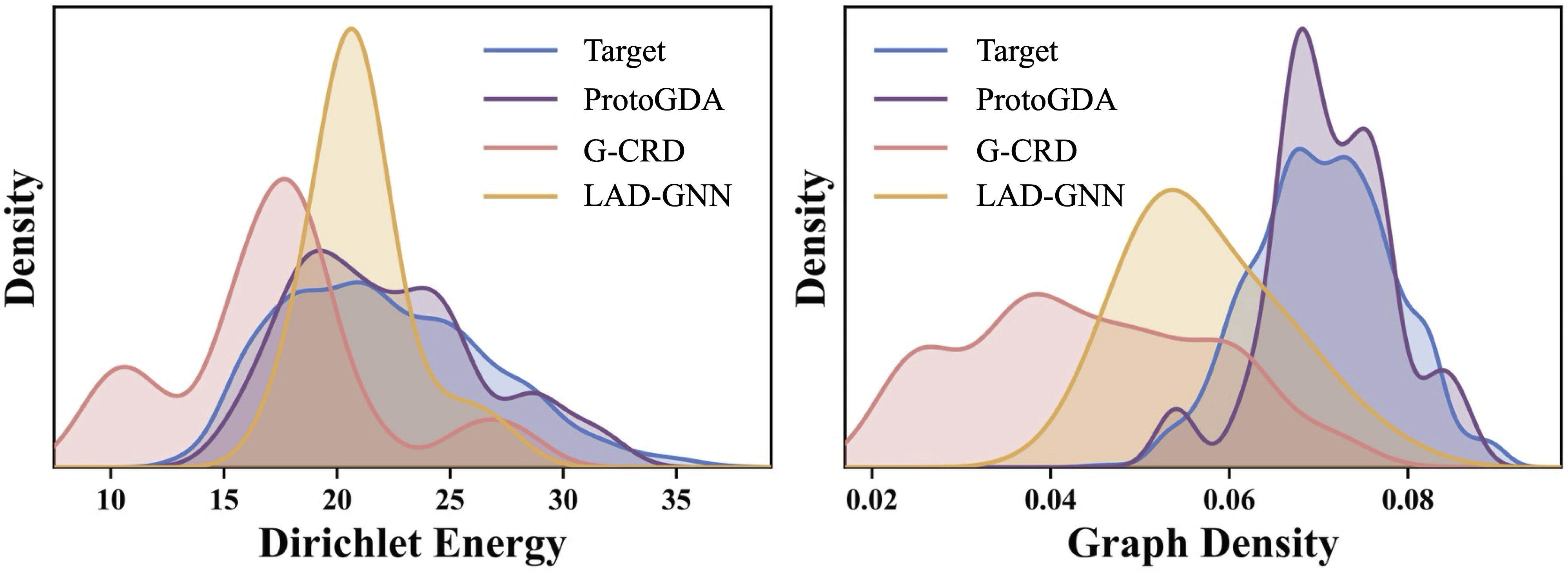}
    \vspace{-16pt}
    \caption{Distributions of Dirichlet energy (left) and graph density (right) for the target graphs and for the training graphs used by \method{}, G-CRD, and LAD-GNN.}
    \label{fig:distribution}
    \vspace{-10pt}
\end{wrapfigure}

To examine whether the alignment objectives reshape the statistics of the training graphs, Fig.~\ref{fig:distribution} compares the distributions of Dirichlet energy and graph density of the target graphs with those of the training graphs used by \method{}, G-CRD, and LAD-GNN. The prototypes of \method{} follow the range and shape of the target energy distribution and overlap with the target density distribution in the high-density region. In comparison, G-CRD places more mass at low energies and spreads widely in density, and LAD-GNN concentrates around a single energy peak and shifts toward lower densities. Although the alignment objectives constrain only the collection means, the full distributions of the prototypes move toward those of the target graphs. This behavior arises because the semantic objective keeps prototypes of different classes distinct, so the mean constraints shift a structurally diverse set jointly and do not collapse it onto the target mean. The distillation baselines, in contrast, train on graphs that retain the source statistics, and the structural mismatch with the target domain persists throughout training.

\section{Conclusion}
\label{sec:conclusion}

In this paper, we revisited graph domain adaptation from a data-centric perspective and proposed \method{}, which treats the training graphs themselves as an adaptation variable. \method{} distills source supervision into a compact set of graph prototypes and conditions them on unlabeled target graphs through correspondence-free structural and spectral statistics, so that a predictor trained on the prototypes transfers directly to the target domain. Extensive experiments on graph and image benchmarks under diverse domain shifts demonstrate the effectiveness of \method{}. We further showed that the target risk of the deployed predictor is controlled by the prototype fitting error, the alignment loss, and a label-consistency term governed by the semantic objective. Future work includes developing richer structural descriptors, designing memory-efficient bilevel solvers for larger graphs, and disentangling structural, attribute, and label shifts within a unified framework.

\bibliographystyle{ACM-Reference-Format}
\bibliography{reference}

@String { ACMMM        = {Proceedings of the ACM International Conference on Multimedia} }

@String { ICLR         = {Proceedings of the International Conference on Learning Representations} }

@String { ICML         = {Proceedings of the International Conference on Machine Learning} }

@String { NIPS         = {Proceedings of the Conference on Neural Information Processing Systems} }

@String { TNNLS        = {IEEE Transactions on Neural Networks and Learning Systems}}

@String { AAAI         = {Proceedings of the AAAI Conference on Artificial Intelligence}}

@String { IJCAI         = {Proceedings of the International Joint Conference on Artificial Intelligence}}

@String { WWW        = {Proceedings of the ACM Web Conference}}

@String { TKDE        = {IEEE Transactions on Knowledge and Data Engineering}}

@String { KDD        = {Proceedings of the International ACM SIGKDD Conference on Knowledge Discovery \& Data Mining}}

@article{di2025graph,
  title={Graph Out-of-Distribution Generalization Based on Structural-Entropy-Guided Information Bottleneck},
  author={Di, Zijun and Zheng, Peng and Lu, Bin and Guan, Kai and Fu, Luoyi and Jin, Ningdi and Chen, Ye and Gan, Xiaoying and Zhou, Lei and Wang, Xinbing and others},
  journal={ACM Transactions on Knowledge Discovery from Data},
  volume={20},
  number={1},
  pages={1--34},
  year={2025},
  publisher={ACM, New York, NY}
}

@inproceedings{yao2023improving,
  title={Improving the expressiveness of k-hop message-passing gnns by injecting contextualized substructure information},
  author={Yao, Tianjun and Wang, Yingxu and Zhang, Kun and Liang, Shangsong},
  booktitle=KDD,
  pages={3070--3081},
  year={2023}
}

@inproceedings{wang2026nested,
  title={Nested graph pseudo-label refinement for noisy label domain adaptation learning},
  author={Wang, Yingxu and Wang, Mengzhu and Huang, Zhichao and Liu, Suyu and Yin, Nan},
  booktitle=AAAI,
  volume={40},
  number={31},
  pages={26697--26705},
  year={2026}
}

@inproceedings{wang2026usbd,
  title={Usbd: Universal structural basis distillation for source-free graph domain adaptation},
  author={Wang, Yingxu and Zhang, Kunyu and Wang, Mengzhu and Gao, Siyang and Yin, Nan},
  booktitle=KDD,
  pages={5125--5136},
  year={2026}
}

@article{wang2026disrfm,
  title={DisRFM: Polar Riemannian Flow Matching for Structure-Preserving Graph Domain Adaptation},
  author={Wang, Yingxu and Liu, Xinwang and Wang, Mengzhu and Gao, Siyang and Yin, Nan},
  journal={arXiv preprint arXiv:2602.00656},
  year={2026}
}

@article{wang2026learning,
  title={Learning Propagation Geometry from Message-Passing Feedback},
  author={Wang, Yingxu and Zhang, Kunyu and Liu, Xinwang and Wang, Mengzhu and Gao, Siyang and Tang, Chang and Yin, Nan},
  journal={arXiv preprint arXiv:2609.34711},
  year={2026}
}

@article{wang2026safe,
  title={Safe-Subspace Pseudo-Label Refinement for Source-Free Graph Domain Adaptation},
  author={Wang, Yingxu and Liu, Xinwang and Gao, Siyang and Yin, Nan},
  journal={arXiv preprint arXiv:2606.00808},
  year={2026}
}

@article{wang2026cross,
  title={Cross-Resolution Semantic Learning for Graph Domain Adaptation},
  author={Wang, Yingxu and Huang, Haoze and Zheng, Zhongkai and Liang, Shangsong},
  journal={arXiv preprint arXiv:2607.29365},
  year={2026}
}

@article{wang2026sgac,
  title={SGAC: a graph neural network framework for imbalanced and structure-aware AMP classification},
  author={Wang, Yingxu and Liang, Victor and Yin, Nan and Liu, Siwei and Segal, Eran},
  journal={Briefings in Bioinformatics},
  volume={27},
  number={1},
  pages={bbag038},
  year={2026},
  publisher={Oxford University Press}
}

@article{wang2025protomol,
  title={Protomol: enhancing molecular property prediction via prototype-guided multimodal learning},
  author={Wang, Yingxu and Zhang, Kunyu and Huang, Jiaxin and Yin, Nan and Liu, Siwei and Segal, Eran},
  journal={Briefings in Bioinformatics},
  volume={26},
  number={6},
  pages={bbaf629},
  year={2025},
  publisher={Oxford University Press}
}

@article{zhao2025spot,
  title={Spot: Spectral optimal transport for graph domain generalization},
  author={Zhao, Yusheng and Luo, Xiao and Luo, Junyu and Ju, Wei and Gu, Zhonghui and Xiao, Zhiping and Hua, Xian-Sheng and Zhang, Ming},
  journal={ACM Transactions on Knowledge Discovery from Data},
  volume={20},
  number={1},
  pages={1--22},
  year={2025},
  publisher={ACM, New York, NY}
}

@article{zhang2026survey,
  title={A survey of deep graph learning under distribution shifts: from graph out-of-distribution generalization to adaptation},
  author={Zhang, Kexin and Liu, Shuhan and Wang, Song and Shi, Weili and Chen, Chen and Li, Pan and Li, Sheng and Li, Jundong and Ding, Kaize},
  journal={ACM Transactions on Knowledge Discovery from Data},
  volume={20},
  number={2},
  pages={1--38},
  year={2026},
  publisher={ACM, New York, NY}
}

@article{yin2025dream,
  title={Dream: a dual variational framework for unsupervised graph domain adaptation},
  author={Yin, Nan and Shen, Li and Wang, Mengzhu and Liu, Xinwang and Chen, Chong and Hua, Xian-Sheng},
  journal={IEEE Transactions on Pattern Analysis and Machine Intelligence},
  year={2025},
  publisher={IEEE}
}

@inproceedings{yin2022deal,
  title={Deal: An unsupervised domain adaptive framework for graph-level classification},
  author={Yin, Nan and Shen, Li and Li, Baopu and Wang, Mengzhu and Luo, Xiao and Chen, Chong and Luo, Zhigang and Hua, Xian-Sheng},
  booktitle=ACMMM,
  pages={3470--3479},
  year={2022}
}

@inproceedings{yin2023coco,
  title={{CoCo}: A Coupled Contrastive Framework for Unsupervised Domain Adaptive Graph Classification},
  author={Yin, Nan and Shen, Li and Wang, Mengzhu and Lan, Long and Ma, Zeyu and Chen, Chong and Hua, Xian-Sheng and Luo, Xiao},
  booktitle=ICML,
  pages={40040--40053},
  year={2023},
  organization={PMLR}
}

@inproceedings{pang2023sa,
  title={{SA-GDA}: Spectral Augmentation for Graph Domain Adaptation},
  author={Pang, Jinhui and Wang, Zixuan and Tang, Jiliang and Xiao, Mingyan and Yin, Nan},
  booktitle=ACMMM,
  pages={309--318},
  year={2023}
}

@article{ben2010theory,
  title={A theory of learning from different domains},
  author={Ben-David, Shai and Blitzer, John and Crammer, Koby and Kulesza, Alex and Pereira, Fernando and Vaughan, Jennifer Wortman},
  journal={Machine Learning},
  volume={79},
  number={1--2},
  pages={151--175},
  year={2010},
  publisher={Springer}
}

@inproceedings{wu2020unsupervised,
  title={Unsupervised domain adaptive graph convolutional networks},
  author={Wu, Man and Pan, Shirui and Zhou, Chuan and Chang, Xiaojun and Zhu, Xingquan},
  booktitle=WWW,
  pages={1457--1467},
  year={2020}
}

@inproceedings{qiao2023semi,
  title={Semi-supervised domain adaptation in graph transfer learning},
  author={Qiao, Ziyue and Luo, Xiao and Xiao, Meng and Dong, Hao and Zhou, Yuanchun and Xiong, Hui},
  booktitle=IJCAI,
  pages={2279--2287},
  year={2023}
}

@inproceedings{liu2023structural,
  title={Structural re-weighting improves graph domain adaptation},
  author={Liu, Shikun and Li, Tianchun and Feng, Yongbin and Tran, Nhan and Zhao, Han and Qiu, Qiang and Li, Pan},
  booktitle=ICML,
  pages={21778--21793},
  year={2023},
  organization={PMLR}
}

@inproceedings{liu2024rethinking,
  title={Rethinking propagation for unsupervised graph domain adaptation},
  author={Liu, Meihan and Fang, Zeyu and Zhang, Zhen and Gu, Ming and Zhou, Sheng and Wang, Xin and Bu, Jiajun},
  booktitle=AAAI,
  volume={38},
  number={12},
  pages={13963--13971},
  year={2024}
}

@inproceedings{liu2024pairwise,
  title={Pairwise alignment improves graph domain adaptation},
  author={Liu, Shikun and Zou, Deyu and Zhao, Han and Li, Pan},
  booktitle=ICML,
  pages={32552--32575},
  year={2024}
}

@inproceedings{fang2025benefits,
  title={On the benefits of attribute-driven graph domain adaptation},
  author={Fang, Ruiyi and Li, Bingheng and Kang, Zhao and Zeng, Qiuhao and Dashtbayaz, Nima Hosseini and Pu, Ruizhi and Wang, Boyu and Ling, Charles},
  booktitle=ICLR,
  year={2025}
}

@inproceedings{chen2025smoothness,
  title={Smoothness really matters: A simple yet effective approach for unsupervised graph domain adaptation},
  author={Chen, Wei and Ye, Guo and Wang, Yakun and Zhang, Zhao and Zhang, Libang and Wang, Daixin and Zhang, Zhiqiang and Zhuang, Fuzhen},
  booktitle=AAAI,
  volume={39},
  pages={15875--15883},
  year={2025}
}

@article{wu2024graph,
  title={Graph learning under distribution shifts: A comprehensive survey on domain adaptation, out-of-distribution, and continual learning},
  author={Wu, Man and Zheng, Xin and Zhang, Qin and Shen, Xiao and Luo, Xiong and Zhu, Xingquan and Pan, Shirui},
  journal={arXiv preprint arXiv:2402.16374},
  year={2024}
}

@article{cai2024graph,
  title={Graph domain adaptation: A generative view},
  author={Cai, Ruichu and Wu, Fengzhu and Li, Zijian and Wei, Pengfei and Yi, Lingling and Zhang, Kun},
  journal={ACM Transactions on Knowledge Discovery from Data},
  volume={18},
  number={3},
  pages={1--24},
  year={2024},
  publisher={ACM New York, NY}
}

@article{dai2022graph,
  title={Graph transfer learning via adversarial domain adaptation with graph convolution},
  author={Dai, Quanyu and Wu, Xiao-Ming and Xiao, Jiaren and Shen, Xiao and Wang, Dan},
  journal=TKDE,
  volume={35},
  number={5},
  pages={4908--4922},
  year={2023},
  doi={10.1109/TKDE.2022.3144250},
  publisher={IEEE}
}

@inproceedings{xiao2023spa,
  title={{SPA}: A Graph Spectral Alignment Perspective for Domain Adaptation},
  author={Xiao, Zhiqing and Wang, Haobo and Jin, Ying and Feng, Lei and Chen, Gang and Huang, Fei and Zhao, Junbo},
  booktitle=NIPS,
  volume={36},
  pages={37252--37272},
  year={2023},
  doi={10.52202/075280-1619}
}

@inproceedings{yang2025disentangled,
  title={Disentangled graph spectral domain adaptation},
  author={Yang, Liang and Chen, Xin and Zhuo, Jiaming and Jin, Di and Wang, Chuan and Cao, Xiaochun and Wang, Zhen and Guo, Yuanfang},
  booktitle=ICML,
  year={2025}
}

@inproceedings{you2023graph,
  title={Graph domain adaptation via theory-grounded spectral regularization},
  author={You, Yuning and Chen, Tianlong and Wang, Zhangyang and Shen, Yang},
  booktitle=ICLR,
  year={2023}
}

@inproceedings{ganin2015unsupervised,
  title={Unsupervised domain adaptation by backpropagation},
  author={Ganin, Yaroslav and Lempitsky, Victor},
  booktitle=ICML,
  pages={1180--1189},
  year={2015},
  organization={PMLR}
}

@inproceedings{huang2024can,
  title={Can modifying data address graph domain adaptation?},
  author={Huang, Renhong and Xu, Jiarong and Jiang, Xin and An, Ruichuan and Yang, Yang},
  booktitle={Proceedings of the 30th ACM SIGKDD Conference on Knowledge Discovery and Data Mining},
  pages={1131--1142},
  year={2024},
  doi={10.1145/3637528.3672058}
}

@inproceedings{liu2024revisiting,
  title={Revisiting, benchmarking and understanding unsupervised graph domain adaptation},
  author={Liu, Meihan and Zhang, Zhen and Tang, Jiachen and Bu, Jiajun and He, Bingsheng and Zhou, Sheng},
  booktitle=NIPS,
  volume={37},
  pages={89408--89436},
  year={2024},
  doi={10.52202/079017-2838}
}

@article{joshi2022representation,
  title={On representation knowledge distillation for graph neural networks},
  author={Joshi, Chaitanya K and Liu, Fayao and Xun, Xu and Lin, Jie and Foo, Chuan Sheng},
  journal=TNNLS,
  volume={35},
  number={4},
  pages={4656--4667},
  year={2024},
  doi={10.1109/TNNLS.2022.3223018},
  publisher={IEEE}
}

@inproceedings{yao2024mugsi,
  title={{MuGSI}: Distilling {GNNs} with Multi-Granularity Structural Information for Graph Classification},
  author={Yao, Tianjun and Sun, Jiaqi and Cao, Defu and Zhang, Kun and Chen, Guangyi},
  booktitle=WWW,
  pages={709--720},
  year={2024}
}

@article{wu2024teacher,
  title={A teacher-free graph knowledge distillation framework with dual self-distillation},
  author={Wu, Lirong and Lin, Haitao and Gao, Zhangyang and Zhao, Guojiang and Li, Stan Z},
  journal=TKDE,
  volume={36},
  number={9},
  pages={4375--4385},
  year={2024},
  publisher={IEEE}
}

@inproceedings{hong2024label,
  title={Label attentive distillation for GNN-based graph classification},
  author={Hong, Xiaobin and Li, Wenzhong and Wang, Chaoqun and Lin, Mingkai and Lu, Sanglu},
  booktitle=AAAI,
  volume={38},
  number={8},
  pages={8499--8507},
  year={2024}
}

@inproceedings{lu2024adagmlp,
  title={{AdaGMLP}: {AdaBoosting} {GNN}-to-{MLP} Knowledge Distillation},
  author={Lu, Weigang and Guan, Ziyu and Zhao, Wei and Yang, Yaming},
  booktitle=KDD,
  pages={2060--2071},
  year={2024}
}

@inproceedings{lai2025simple,
  title={Simple yet effective graph distillation via clustering},
  author={Lai, Yurui and Zhang, Taiyan and Yang, Renchi},
  booktitle=KDD,
  pages={1229--1240},
  year={2025}
}

@article{gao2025graph,
  title={Graph condensation: A survey},
  author={Gao, Xinyi and Yu, Junliang and Chen, Tong and Ye, Guanhua and Zhang, Wentao and Yin, Hongzhi},
  journal=TKDE,
  volume={37},
  number={4},
  pages={1819--1837},
  year={2025},
  publisher={IEEE}
}

@inproceedings{wang2024self,
  title={Self-supervised learning for graph dataset condensation},
  author={Wang, Yuxiang and Yan, Xiao and Jin, Shiyu and Huang, Hao and Xu, Quanqing and Zhang, Qingchen and Du, Bo and Jiang, Jiawei},
  booktitle=KDD,
  pages={3289--3298},
  year={2024}
}

@inproceedings{jin2022condensing,
  title={Condensing graphs via one-step gradient matching},
  author={Jin, Wei and Tang, Xianfeng and Jiang, Haoming and Li, Zheng and Zhang, Danqing and Tang, Jiliang and Yin, Bing},
  booktitle=KDD,
  pages={720--730},
  year={2022}
}

@inproceedings{baek2021accurate,
  title={Accurate learning of graph representations with graph multiset pooling},
  author={Baek, Jinheon and Kang, Minki and Hwang, Sung Ju},
  booktitle=ICLR,
  year={2021}
}

@inproceedings{kipf2016semi,
  title={Semi-supervised classification with graph convolutional networks},
  author={Kipf, Thomas N and Welling, Max},
  booktitle=ICLR,
  year={2017}
}

@article{levie2021transferability,
  title={Transferability of spectral graph convolutional neural networks},
  author={Levie, Ron and Huang, Wei and Bucci, Lorenzo and Bronstein, Michael and Kutyniok, Gitta},
  journal={Journal of Machine Learning Research},
  volume={22},
  number={272},
  pages={1--59},
  year={2021}
}

@article{muller1997integral,
  title={Integral probability metrics and their generating classes of functions},
  author={M{\"u}ller, Alfred},
  journal={Advances in applied probability},
  volume={29},
  number={2},
  pages={429--443},
  year={1997},
  publisher={Cambridge University Press}
}

@article{lecun2002gradient,
  title={Gradient-based learning applied to document recognition},
  author={LeCun, Yann and Bottou, L{\'e}on and Bengio, Yoshua and Haffner, Patrick},
  journal={Proceedings of the IEEE},
  volume={86},
  number={11},
  pages={2278--2324},
  year={1998},
  doi={10.1109/5.726791},
  publisher={Ieee}
}

@techreport{krizhevsky2009learning,
  title={Learning multiple layers of features from tiny images},
  author={Krizhevsky, Alex},
  institution={University of Toronto},
  year={2009}
}

@article{dobson2003distinguishing,
  title={Distinguishing enzyme structures from non-enzymes without alignments},
  author={Dobson, Paul D and Doig, Andrew J},
  journal={Journal of Molecular Biology},
  pages={771--783},
  year={2003},
  publisher={Elsevier}
}

@article{kazius2005derivation,
  title={Derivation and validation of toxicophores for mutagenicity prediction},
  author={Kazius, Jeroen and McGuire, Ross and Bursi, Roberta},
  journal={Journal of Medicinal Chemistry},
  volume={48},
  number={1},
  pages={312--320},
  year={2005},
  doi={10.1021/jm040835a},
  publisher={ACS Publications}
}

@article{wale2008comparison,
  title={Comparison of descriptor spaces for chemical compound retrieval and classification},
  author={Wale, Nikil and Watson, Ian A and Karypis, George},
  journal={Knowledge and Information Systems},
  volume={14},
  pages={347--375},
  year={2008},
  publisher={Springer}
}

@inproceedings{orsini2015graph,
  title={Graph invariant kernels},
  author={Orsini, Francesco and Frasconi, Paolo and De Raedt, Luc},
  booktitle=IJCAI,
  volume={2015},
  pages={3756--3762},
  year={2015},
}

@article{hu2021ogblsc,
  title={Ogb-lsc: A large-scale challenge for machine learning on graphs},
  author={Hu, Weihua and Fey, Matthias and Ren, Hongyu and Nakata, Maho and Dong, Yuxiao and Leskovec, Jure},
  journal={arXiv preprint arXiv:2103.09430},
  year={2021}
}

@inproceedings{wu2022discovering,
  title={Discovering invariant rationales for graph neural networks},
  author={Wu, Ying-Xin and Wang, Xiang and Zhang, An and He, Xiangnan and Chua, Tat-Seng},
  booktitle=ICLR,
  year={2022}
}

@inproceedings{xu2018how,
  title={How powerful are graph neural networks?},
  author={Xu, Keyulu and Hu, Weihua and Leskovec, Jure and Jegelka, Stefanie},
  booktitle=ICLR,
  year={2019}
}

@article{dwivedi2023benchmarking,
  title={Benchmarking graph neural networks},
  author={Dwivedi, Vijay Prakash and Joshi, Chaitanya K and Luu, Anh Tuan and Laurent, Thomas and Bengio, Yoshua and Bresson, Xavier},
  journal={Journal of Machine Learning Research},
  volume={24},
  number={43},
  pages={1--48},
  year={2023}
}

@article{gretton2012kernel,
  title={A kernel two-sample test},
  author={Gretton, Arthur and Borgwardt, Karsten M. and Rasch, Malte J. and Sch{\"o}lkopf, Bernhard and Smola, Alexander},
  journal={Journal of Machine Learning Research},
  volume={13},
  number={25},
  pages={723--773},
  year={2012}
}

@incollection{mcdiarmid1989method,
  title={On the method of bounded differences},
  author={McDiarmid, Colin},
  booktitle={Surveys in Combinatorics},
  series={London Mathematical Society Lecture Note Series},
  volume={141},
  pages={148--188},
  publisher={Cambridge University Press},
  year={1989}
}

@book{cover2006elements,
  title={Elements of Information Theory},
  author={Cover, Thomas M. and Thomas, Joy A.},
  edition={2nd},
  publisher={Wiley-Interscience},
  year={2006}
}

@inproceedings{hu2020open,
  title={Open graph benchmark: Datasets for machine learning on graphs},
  author={Hu, Weihua and Fey, Matthias and Zitnik, Marinka and Dong, Yuxiao and Ren, Hongyu and Liu, Bowen and Catasta, Michele and Leskovec, Jure},
  booktitle=NIPS,
  volume={33},
  pages={22118--22133},
  year={2020}
}

@inproceedings{jin2022graph,
  title={Graph condensation for graph neural networks},
  author={Jin, Wei and Zhao, Lingxiao and Zhang, Shichang and Liu, Yozen and Tang, Jiliang and Shah, Neil},
  booktitle=ICLR,
  year={2022}
}

@inproceedings{xu2023kernel,
  title={Kernel ridge regression-based graph dataset distillation},
  author={Xu, Zhe and Chen, Yuzhong and Pan, Menghai and Chen, Huiyuan and Das, Mahashweta and Yang, Hao and Tong, Hanghang},
  booktitle={Proceedings of the ACM SIGKDD Conference on Knowledge Discovery and Data Mining},
  pages={2850--2861},
  year={2023}
}

@inproceedings{zhao2021dataset,
  title={Dataset condensation with gradient matching},
  author={Zhao, Bo and Mopuri, Konda Reddy and Bilen, Hakan},
  booktitle=ICLR,
  year={2021}
}

@inproceedings{cazenavette2022dataset,
  title={Dataset distillation by matching training trajectories},
  author={Cazenavette, George and Wang, Tongzhou and Torralba, Antonio and Efros, Alexei A. and Zhu, Jun-Yan},
  booktitle={Proceedings of the IEEE/CVF Conference on Computer Vision and Pattern Recognition},
  pages={4750--4759},
  year={2022}
}

@article{wang2018dataset,
  title={Dataset distillation},
  author={Wang, Tongzhou and Zhu, Jun-Yan and Torralba, Antonio and Efros, Alexei A.},
  journal={arXiv preprint arXiv:1811.10959},
  year={2018}
}

@inproceedings{franceschi2018bilevel,
  title={Bilevel programming for hyperparameter optimization and meta-learning},
  author={Franceschi, Luca and Frasconi, Paolo and Salzo, Saverio and Grazzi, Riccardo and Pontil, Massimiliano},
  booktitle=ICML,
  pages={1568--1577},
  year={2018},
  organization={PMLR}
}

@article{hinton2015distilling,
  title={Distilling the knowledge in a neural network},
  author={Hinton, Geoffrey and Vinyals, Oriol and Dean, Jeff},
  journal={arXiv preprint arXiv:1503.02531},
  year={2015}
}

\clearpage
\appendix
\section{Algorithm}
\begin{algorithm}[t]
\caption{Learning framework of \method{}}
\label{alg:overall}
\small
\begin{algorithmic}[1]
\Require Labeled source graphs \(\mathcal{D}_S\), unlabeled target graphs \(\mathcal{D}_T\), number of prototypes \(K\), optimization budgets \(T_{\mathrm{out}},T_{\mathrm{in}},T_{\mathrm{inf}}\), learning rates \(\eta_{\mathrm{in}},\eta_{\mathrm{out}},\eta_{\mathrm{inf}}\), and alignment weights \(\lambda_{\mathrm{geo}},\lambda_{\mathrm{spec}}\).
\Ensure Target predictions \(\hat{Y}_T\).
\State Initialize \(\{\mathbf Z_k,\widetilde{\mathbf X}_k,\widetilde{\mathbf y}_k\}_{k=1}^{K}\) and the proxy starting state \(\theta^{(0)}\).
\State Construct \(\widetilde{\mathbf A}_k\) from \(\mathbf Z_k\) by Eq.~\eqref{eq:adj_param} and form \(\mathcal B=\{(\widetilde G_k,\widetilde{\mathbf y}_k)\}_{k=1}^{K}\).
\State Compute \(\bar{\boldsymbol\psi}_{\mathrm{geo}}^{T}\) and \(\bar\Omega^{T}\) from \(\mathcal D_T\) by Eqs.~\eqref{eq:geo_stats} and~\eqref{eq:spec_stats}.
\For{\(r=1,\ldots,T_{\mathrm{out}}\)}
    \State Compute \(\widehat\theta_{\mathcal B}\gets\operatorname{Opt}_{\theta}^{T_{\mathrm{in}}}(\theta^{(0)};\mathcal L_{\mathrm{in}}(\theta;\mathcal B))\) with Adam at learning rate \(\eta_{\mathrm{in}}\) and \(\mathcal B\) fixed, as in Eq.~\eqref{eq:semantic_bilevel}.
    \State Compute \(\mathcal L_{\mathrm{sem}}(\widehat\theta_{\mathcal B};\mathcal D_S)\) by Eq.~\eqref{eq:semantic_loss}.
    \State Compute \(\mathcal L_{\mathrm{geo}}(\mathcal B;\mathcal D_T)\) and \(\mathcal L_{\mathrm{spec}}(\mathcal B;\mathcal D_T)\) by Eqs.~\eqref{eq:geo_loss} and~\eqref{eq:spec_loss}.
    \State Combine the three losses into \(\mathcal L_{\mathrm{total}}\) by Eq.~\eqref{eq:outer_meta}.
    \State Update \(\{\mathbf Z_k,\widetilde{\mathbf X}_k,\widetilde{\mathbf y}_k\}_{k=1}^{K}\) with one Adam step on \(\mathcal L_{\mathrm{total}}\) at learning rate \(\eta_{\mathrm{out}}\), differentiating \(\mathcal L_{\mathrm{sem}}\) through the \(T_{\mathrm{in}}\) inner updates.
    \State Recompute \(\widetilde{\mathbf A}_k\) from the updated \(\mathbf Z_k\) by Eq.~\eqref{eq:adj_param} and update \(\mathcal B\).
\EndFor
\State Freeze the optimized prototypes as \(\mathcal B^*\gets\mathcal B\) and discard the proxy GNN.
\State Initialize \(\phi^{(0)}\) independently of the proxy.
\State Compute \(\phi^*\gets\operatorname{Opt}_{\phi}^{T_{\mathrm{inf}}}(\phi^{(0)};\mathcal L_{\mathrm{inf}}(\phi;\mathcal B^*))\) with Adam at learning rate \(\eta_{\mathrm{inf}}\) and \(\mathcal B^*\) fixed, as in Eq.~\eqref{eq:inference_train}.
\For{each target graph \(G_j^t\in\mathcal{D}_T\)}
    \State Compute \(\hat{y}_j^t\) by Eq.~\eqref{eq:target_prediction}.
\EndFor
\State \Return \(\hat{Y}_T=\{\hat{y}_j^t\}_{j=1}^{N_T}\).
\end{algorithmic}
\end{algorithm}

Algorithm~\ref{alg:overall} summarizes the overall training procedure of \method{}.

\section{Notation Summary}
\label{app:notation}

\begin{table}[h!]
    \centering
    \caption{Summary of key notations.}
    \label{tab:notation}
    \small
    \setlength{\tabcolsep}{6pt}
    \renewcommand{\arraystretch}{1.10}
    \begin{tabularx}{0.98\linewidth}{@{}lX@{}}
        \toprule
        \textbf{Symbol} & \textbf{Description} \\
        \midrule
        $\mathcal D_S,\mathcal D_T,N_S,N_T$
        & Labeled source domain, unlabeled target domain, and their numbers of graphs. \\
        $P_S,P_T,P_T^{\mathbb G}$
        & Source and target joint distributions, and the target marginal over graphs. \\
        $G=(\mathbf A,\mathbf X),N,d$
        & Attributed graph with adjacency and features, node count, and feature dimension. \\
        $\mathcal Y,C,\ell_{\mathrm{CE}}$
        & Label space, number of classes, and cross-entropy loss. \\
        \midrule
        $\mathcal B,K$
        & Prototype set and number of prototypes. \\
        $\widetilde G_k=(\widetilde{\mathbf A}_k,\widetilde{\mathbf X}_k),\widetilde{\mathbf y}_k,\widetilde N_k$
        & $k$-th prototype with adjacency and features, its soft label, and its node count. \\
        $\mathbf Z_k,\mathbf P_k,\mathbf S_k$
        & Learnable score matrix, its sigmoid relaxation, and its symmetrized relaxation. \\
        $f_\theta,\theta^{(0)},\widehat\theta_{\mathcal B}$
        & Proxy GNN, its starting state, and its state after $T_{\mathrm{in}}$ inner updates on $\mathcal B$. \\
        $\mathcal L_{\mathrm{in}},\mathcal L_{\mathrm{sem}}$
        & Inner loss on the prototypes and semantic loss of the proxy on source graphs. \\
        \midrule
        $\mathbf d,\mu_d,\sigma_d,\rho,\tau$
        & Degree vector, mean degree, degree deviation, density, and triangle intensity. \\
        $\boldsymbol\psi_{\mathrm{geo}},\bar{\boldsymbol\psi}_{\mathrm{geo}}^{B},\bar{\boldsymbol\psi}_{\mathrm{geo}}^{T},\mathcal L_{\mathrm{geo}}$
        & Geometric descriptor, its prototype and target means, and its alignment loss. \\
        $\widehat{\mathbf A},\widehat{\mathbf D},\widehat{\mathbf L}$
        & Self-looped adjacency, its degree matrix, and its normalized Laplacian. \\
        $\Omega,\bar\Omega^{B},\bar\Omega^{T},\mathcal L_{\mathrm{spec}}$
        & Dirichlet energy, its prototype and target means, and its alignment loss. \\
        $\lambda_{\mathrm{geo}},\lambda_{\mathrm{spec}},\mathcal L_{\mathrm{tgt}},\mathcal L_{\mathrm{total}}$
        & Alignment weights, label-free target term, and complete outer objective. \\
        $\Phi,\mathbf M,\bar\Phi^{B},\bar\Phi^{T}$
        & Joint five-dimensional descriptor, diagonal weight matrix, and collection means. \\
        $\mathbf b_k,\mathbf J_{\theta,k},\mathbf J_{\Phi,k}$
        & Variables of prototype $k$ and the Jacobians of $\widehat\theta_{\mathcal B}$ and $\Phi(\widetilde G_k)$ with respect to them. \\
        \midrule
        $\mathcal B^*,f_\phi,\phi^{(0)},\phi^*,\mathcal L_{\mathrm{inf}}$
        & Frozen prototypes, deployed predictor, initialization, trained state, and loss. \\
        $T_{\mathrm{out}},T_{\mathrm{in}},T_{\mathrm{inf}}$
        & Numbers of outer iterations, inner updates, and deployment updates. \\
        $\eta_{\mathrm{in}},\eta_{\mathrm{out}},\eta_{\mathrm{inf}}$
        & Learning rates of the inner, outer, and deployment optimizers. \\
        $\mathbb G,N_{\max},x_{\max},\mathcal H$
        & Bounded graph space, its node and feature-norm limits, and the predictor class. \\
        $\ell_{\mathrm{TV}},R_T,R_{\mathcal B},\gamma_{\mathcal B}$
        & Total-variation loss, target risk, prototype risk, and joint optimal risk. \\
        $\kappa,\varepsilon_\Phi,R,\delta$
        & Descriptor sensitivity and residual, descriptor bound, and confidence level. \\
        \bottomrule
    \end{tabularx}
\end{table}

As shown in Table~\ref{tab:notation}, we summarize the key notations of this paper.

\section{Proof of Proposition~\ref{prop:conditioning}}
\label{app:proof_conditioning}

\textbf{Proposition~\ref{prop:conditioning} (Correspondence-Free and Label-Free Target Conditioning)} \textit{For an attributed graph $G=(\mathbf A,\mathbf X)$, let $\Phi(G)=[\boldsymbol\psi_{\mathrm{geo}}(\mathbf A)^{\top},\Omega(\mathbf A,\mathbf X)]^{\top}\in\mathbb R^{5}$ collect its geometric descriptor and Dirichlet energy, let $\mathbf M=\operatorname{diag}(\lambda_{\mathrm{geo}}\mathbf I_4,\lambda_{\mathrm{spec}})$ with $\lambda_{\mathrm{geo}},\lambda_{\mathrm{spec}}\ge0$, and let $\bar\Phi^{B}$ and $\bar\Phi^{T}$ denote the means of $\Phi$ over the prototype set $\mathcal B$ and the target collection $\mathcal D_T$. Then $\Phi(\boldsymbol\Pi\mathbf A\boldsymbol\Pi^{\top},\boldsymbol\Pi\mathbf X)=\Phi(\mathbf A,\mathbf X)$ for every permutation matrix $\boldsymbol\Pi$ of matching size, and
\begin{equation*}
\mathcal L_{\mathrm{tgt}}(\mathcal B;\mathcal D_T)=\big\|\mathbf M^{1/2}\big(\bar\Phi^{B}-\bar\Phi^{T}\big)\big\|_2^2 .
\end{equation*}
Moreover, if the inner optimizer is differentiable, then for the $k$-th prototype with topology and feature variables $\mathbf b_k=(\mathbf Z_k,\widetilde{\mathbf X}_k)$ and soft label $\widetilde{\mathbf y}_k$, the gradients of the outer objective decompose as
\begin{equation*}
\begin{aligned}
\nabla_{\mathbf b_k}\mathcal L_{\mathrm{total}}
&=\mathbf J_{\theta,k}^{\top}\nabla_{\widehat\theta}\mathcal L_{\mathrm{sem}}
+\frac{2}{K}\mathbf J_{\Phi,k}^{\top}\mathbf M\big(\bar\Phi^{B}-\bar\Phi^{T}\big),\\
\nabla_{\widetilde{\mathbf y}_k}\mathcal L_{\mathrm{total}}
&=\Big(\frac{\partial\widehat\theta_{\mathcal B}}{\partial\widetilde{\mathbf y}_k}\Big)^{\top}\nabla_{\widehat\theta}\mathcal L_{\mathrm{sem}},
\end{aligned}
\end{equation*}
where $\mathbf J_{\theta,k}=\partial\widehat\theta_{\mathcal B}/\partial\mathbf b_k$ and $\mathbf J_{\Phi,k}=\partial\Phi(\widetilde G_k)/\partial\mathbf b_k$.}

\textbf{\textit{Proof.}}\par\nopagebreak
We begin with the invariance of $\Phi$ to node relabeling. Let $\mathbf A'=\boldsymbol\Pi\mathbf A\boldsymbol\Pi^{\top}$ and $\mathbf X'=\boldsymbol\Pi\mathbf X$, and note that $\mathbf A'$ remains symmetric with a zero diagonal. Since $\boldsymbol\Pi^{\top}\mathbf 1=\mathbf 1$, $\mathbf 1^{\top}\boldsymbol\Pi=\mathbf 1^{\top}$, and $\boldsymbol\Pi^{\top}\boldsymbol\Pi=\mathbf I$, the degree vector of the relabeled graph is a permutation of the original one,
\begin{equation}
\mathbf d'=\boldsymbol\Pi\mathbf A\boldsymbol\Pi^{\top}\mathbf 1=\boldsymbol\Pi\mathbf d .
\end{equation}
A permutation changes neither the mean of the degrees nor their deviation from the mean, so $\mu_d(\mathbf A')=\frac1N\mathbf 1^{\top}\boldsymbol\Pi\mathbf d=\mu_d(\mathbf A)$ and
\begin{equation}
\big\|\mathbf d'-\mu_d(\mathbf A')\mathbf 1\big\|_2=\big\|\boldsymbol\Pi\big(\mathbf d-\mu_d(\mathbf A)\mathbf 1\big)\big\|_2=\big\|\mathbf d-\mu_d(\mathbf A)\mathbf 1\big\|_2,
\end{equation}
which gives $\sigma_d(\mathbf A')=\sigma_d(\mathbf A)$. The density and the triangle intensity are preserved as well, because $\mathbf 1^{\top}\mathbf A'\mathbf 1=\mathbf 1^{\top}\mathbf A\mathbf 1$ and, by the cyclic property of the trace, $\operatorname{Tr}(\mathbf A'^{3})=\operatorname{Tr}(\boldsymbol\Pi\mathbf A^{3}\boldsymbol\Pi^{\top})=\operatorname{Tr}(\mathbf A^{3})$. The same reasoning carries over to the Dirichlet energy. The self-looped adjacency satisfies $\widehat{\mathbf A}'=\boldsymbol\Pi\widehat{\mathbf A}\boldsymbol\Pi^{\top}$, and its degree matrix has strictly positive diagonal entries because of the self-loops. Permuting a vector permutes the diagonal of the corresponding diagonal matrix, so $\widehat{\mathbf D}'=\boldsymbol\Pi\widehat{\mathbf D}\boldsymbol\Pi^{\top}$ and $\widehat{\mathbf D}'^{-1/2}=\boldsymbol\Pi\widehat{\mathbf D}^{-1/2}\boldsymbol\Pi^{\top}$. The normalized Laplacian is then conjugated by the same permutation,
\begin{equation}
\widehat{\mathbf L}'
=\mathbf I-\boldsymbol\Pi\widehat{\mathbf D}^{-1/2}\boldsymbol\Pi^{\top}\boldsymbol\Pi\widehat{\mathbf A}\boldsymbol\Pi^{\top}\boldsymbol\Pi\widehat{\mathbf D}^{-1/2}\boldsymbol\Pi^{\top}
=\boldsymbol\Pi\widehat{\mathbf L}\boldsymbol\Pi^{\top},
\end{equation}
and the permutation cancels inside the trace,
\begin{equation}
\Omega(\mathbf A',\mathbf X')
=\operatorname{Tr}\big(\mathbf X^{\top}\boldsymbol\Pi^{\top}\boldsymbol\Pi\widehat{\mathbf L}\boldsymbol\Pi^{\top}\boldsymbol\Pi\mathbf X\big)
=\Omega(\mathbf A,\mathbf X).
\end{equation}
Since $\boldsymbol\psi_{\mathrm{geo}}$ and $\Omega$ map a graph with any number of nodes to a value of fixed dimension, $\Phi(G)\in\mathbb R^{5}$ is defined for every graph regardless of its size and node ordering.

With $\Phi$ in place, the target term can be written as a single weighted distance. The first four coordinates of $\bar\Phi^{B}-\bar\Phi^{T}$ equal $\bar{\boldsymbol\psi}_{\mathrm{geo}}^{B}-\bar{\boldsymbol\psi}_{\mathrm{geo}}^{T}$ and carry weight $\lambda_{\mathrm{geo}}$ in $\mathbf M$, and the fifth coordinate equals $\bar\Omega^{B}-\bar\Omega^{T}$ and carries weight $\lambda_{\mathrm{spec}}$. Since $\mathbf M$ is diagonal with non-negative entries, $\mathbf M^{1/2}$ is well defined and
\begin{equation}
\big\|\mathbf M^{1/2}\big(\bar\Phi^{B}-\bar\Phi^{T}\big)\big\|_2^2
=\lambda_{\mathrm{geo}}\big\|\bar{\boldsymbol\psi}_{\mathrm{geo}}^{B}-\bar{\boldsymbol\psi}_{\mathrm{geo}}^{T}\big\|_2^2+\lambda_{\mathrm{spec}}\big(\bar\Omega^{B}-\bar\Omega^{T}\big)^2
=\lambda_{\mathrm{geo}}\mathcal L_{\mathrm{geo}}+\lambda_{\mathrm{spec}}\mathcal L_{\mathrm{spec}},
\label{eq:app_distance_form}
\end{equation}
whose right-hand side is $\mathcal L_{\mathrm{tgt}}$ by definition. To relate this distance to the maximum mean discrepancy, consider the kernel $k_\Phi(G,G')=\Phi(G)^{\top}\mathbf M\,\Phi(G')$, whose feature map is $G\mapsto\mathbf M^{1/2}\Phi(G)$. For the empirical distributions of $\mathcal B$ and $\mathcal D_T$, the squared (biased) empirical MMD under $k_\Phi$ expands as
\begin{equation}
\frac{1}{K^2}\sum_{k,k'}k_\Phi(\widetilde G_k,\widetilde G_{k'})
-\frac{2}{KN_T}\sum_{k,j}k_\Phi(\widetilde G_k,G_j^t)
+\frac{1}{N_T^2}\sum_{j,j'}k_\Phi(G_j^t,G_{j'}^t)
=\big\|\mathbf M^{1/2}\big(\bar\Phi^{B}-\bar\Phi^{T}\big)\big\|_2^2 ,
\end{equation}
because the kernel is bilinear in the two descriptors~\cite{gretton2012kernel}. The target term is exactly this discrepancy.

The distance form also determines how the target term enters the gradients. The descriptor $\Phi(\widetilde G_k)$ is differentiable with respect to $\mathbf b_k$, since $\widetilde{\mathbf A}_k$ is a smooth function of $\mathbf Z_k$, $\sigma_d$ contains the stabilizer $\epsilon>0$, and the self-loops keep $\widehat{\mathbf D}^{-1/2}$ smooth. For a differentiable inner optimizer, the proxy state $\widehat\theta_{\mathcal B}$ obtained after $T_{\mathrm{in}}$ updates is differentiable with respect to every prototype variable. The outer objective is $\mathcal L_{\mathrm{total}}=\mathcal L_{\mathrm{sem}}+\mathcal L_{\mathrm{tgt}}$, and the semantic loss depends on the prototype variables only through $\widehat\theta_{\mathcal B}$, so the chain rule gives
\begin{equation}
\nabla_{\mathbf b_k}\mathcal L_{\mathrm{sem}}=\mathbf J_{\theta,k}^{\top}\nabla_{\widehat\theta}\mathcal L_{\mathrm{sem}},
\qquad
\nabla_{\widetilde{\mathbf y}_k}\mathcal L_{\mathrm{sem}}=\Big(\frac{\partial\widehat\theta_{\mathcal B}}{\partial\widetilde{\mathbf y}_k}\Big)^{\top}\nabla_{\widehat\theta}\mathcal L_{\mathrm{sem}} .
\label{eq:app_sem_grad}
\end{equation}
By Eq.~\eqref{eq:app_distance_form}, the target term depends on the prototypes only through $\bar\Phi^{B}$, while $\bar\Phi^{T}$ is fixed. Only the $k$-th summand of $\bar\Phi^{B}=\frac1K\sum_{k'=1}^{K}\Phi(\widetilde G_{k'})$ depends on $\mathbf b_k$, so $\partial\bar\Phi^{B}/\partial\mathbf b_k=\frac1K\mathbf J_{\Phi,k}$ and
\begin{equation}
\nabla_{\mathbf b_k}\mathcal L_{\mathrm{tgt}}
=\Big(\frac{\partial\bar\Phi^{B}}{\partial\mathbf b_k}\Big)^{\top}\,2\mathbf M\big(\bar\Phi^{B}-\bar\Phi^{T}\big)
=\frac{2}{K}\mathbf J_{\Phi,k}^{\top}\mathbf M\big(\bar\Phi^{B}-\bar\Phi^{T}\big).
\label{eq:app_tgt_grad}
\end{equation}
The descriptor $\Phi(\widetilde G_k)$ is computed from $\widetilde{\mathbf A}_k$ and $\widetilde{\mathbf X}_k$ alone, neither of which involves $\widetilde{\mathbf y}_k$, so the target term contributes no gradient to the soft labels. Adding Eqs.~\eqref{eq:app_sem_grad} and~\eqref{eq:app_tgt_grad} gives the claimed decomposition.
\hfill$\square$

Two consequences of the proof support the discussion in the main text. First, $\boldsymbol\psi_{\mathrm{geo}}(\widetilde{\mathbf A}_k)$ does not depend on $\widetilde{\mathbf X}_k$, so the block of $\mathbf J_{\Phi,k}$ associated with $\widetilde{\mathbf X}_k$ is zero in its first four rows, and the target residual reaches the node features only through the energy coordinate weighted by $\lambda_{\mathrm{spec}}$. Second, $\mathcal L_{\mathrm{tgt}}$ is invariant to any permutation of the class indices in $\{\widetilde{\mathbf y}_k\}_{k=1}^{K}$ and contributes no gradient to them, so the soft labels are updated by $\mathcal L_{\mathrm{sem}}$ alone.

\section{Proof of Proposition~\ref{prop:risk}}
\label{app:proof_risk}

\textbf{Proposition~\ref{prop:risk} (Target Risk of the Deployed Predictor)} \textit{Let $\mathbb G$ be the space of attributed graphs with at most $N_{\max}$ nodes, symmetric adjacency matrices in $[0,1]^{N\times N}$ with zero diagonal, and node features of $\ell_2$ norm at most $x_{\max}$, and let $\mathcal H$ be a class of predictors $h\colon\mathbb G\to\Delta^{C-1}$. For a prototype set $\mathcal B\subset\mathbb G$ and $h\in\mathcal H$, define the target risk $R_T(h)=\mathbb E_{(G,y)\sim P_T}[\ell_{\mathrm{TV}}(h(G),\mathbf e_y)]$, the prototype risk $R_{\mathcal B}(h)=\frac1K\sum_{k=1}^{K}\ell_{\mathrm{TV}}(h(\widetilde G_k),\widetilde{\mathbf y}_k)$, and the joint optimal risk $\gamma_{\mathcal B}=\inf_{h'\in\mathcal H}[R_{\mathcal B}(h')+R_T(h')]$. With $\Phi(G)=[\boldsymbol\psi_{\mathrm{geo}}(\mathbf A)^{\top},\Omega(\mathbf A,\mathbf X)]^{\top}$ and $\mathbf M=\operatorname{diag}(\lambda_{\mathrm{geo}}\mathbf I_4,\lambda_{\mathrm{spec}})$ for $\lambda_{\mathrm{geo}},\lambda_{\mathrm{spec}}>0$, suppose that (i) for every $h,h'\in\mathcal H$, the disagreement $g(G)=\ell_{\mathrm{TV}}(h(G),h'(G))$ admits $\mathbf w\in\mathbb R^{5}$ and $c\in\mathbb R$ with $\|\mathbf M^{-1/2}\mathbf w\|_2\le\kappa$ and $\sup_{G\in\mathbb G}|g(G)-\mathbf w^{\top}\Phi(G)-c|\le\varepsilon_\Phi$, (ii) $\mathcal D_T$ is drawn i.i.d.\ from the target marginal $P_T^{\mathbb G}$, and (iii) $\|\mathbf M^{1/2}\Phi(G)\|_2\le R$ holds $P_T^{\mathbb G}$-almost surely. Then for any $\delta\in(0,1)$, with probability at least $1-\delta$, every prototype set $\mathcal B\subset\mathbb G$, including one optimized on $\mathcal D_T$, and every $h\in\mathcal H$ satisfy
\begin{equation*}
R_T(h)\le R_{\mathcal B}(h)+\kappa\sqrt{\mathcal L_{\mathrm{tgt}}(\mathcal B;\mathcal D_T)}+\gamma_{\mathcal B}+2\varepsilon_\Phi+\frac{\kappa R\big(1+\sqrt{2\ln(1/\delta)}\big)}{\sqrt{N_T}}.
\end{equation*}
Moreover, on the same event, every predictor $f_{\phi}\in\mathcal H$ evaluated on the frozen prototypes $\mathcal B^*$ with the cross-entropy loss $\mathcal L_{\mathrm{inf}}(\phi;\mathcal B^*)=\frac1K\sum_{k=1}^{K}\ell_{\mathrm{CE}}(f_{\phi}(\widetilde G_k^*),\widetilde{\mathbf y}_k^*)$ satisfies
\begin{equation*}
R_T(f_{\phi})\le\sqrt{\tfrac12\mathcal L_{\mathrm{inf}}(\phi;\mathcal B^*)}+\kappa\sqrt{\mathcal L_{\mathrm{tgt}}(\mathcal B^*;\mathcal D_T)}+\gamma_{\mathcal B^*}+2\varepsilon_\Phi+\frac{\kappa R\big(1+\sqrt{2\ln(1/\delta)}\big)}{\sqrt{N_T}} .
\end{equation*}}

We first establish a concentration inequality for the empirical target statistics.

\begin{lemma}[Concentration of Bounded Empirical Means]
\label{lem:concentration}
Let $\mathbf z_1,\ldots,\mathbf z_n$ be i.i.d.\ random vectors in $\mathbb R^m$ with $\|\mathbf z_j\|_2\le R$ almost surely, mean $\boldsymbol\mu=\mathbb E[\mathbf z_1]$, and empirical mean $\bar{\mathbf z}=\frac1n\sum_{j=1}^{n}\mathbf z_j$. For any $\delta\in(0,1)$, with probability at least $1-\delta$,
\begin{equation*}
\|\bar{\mathbf z}-\boldsymbol\mu\|_2\le\frac{R\big(1+\sqrt{2\ln(1/\delta)}\big)}{\sqrt n}.
\end{equation*}
\end{lemma}

\textbf{\textit{Proof.}}\par\nopagebreak
Let $F(\mathbf z_1,\ldots,\mathbf z_n)=\|\bar{\mathbf z}-\boldsymbol\mu\|_2$. Replacing $\mathbf z_j$ by $\mathbf z_j'$ changes $F$ by at most $\|\mathbf z_j-\mathbf z_j'\|_2/n\le 2R/n$ by the triangle inequality. McDiarmid's inequality~\cite{mcdiarmid1989method} then gives
\begin{equation}
\Pr\big(F-\mathbb E F\ge t\big)\le\exp\Big(-\frac{2t^2}{n(2R/n)^2}\Big)=\exp\Big(-\frac{nt^2}{2R^2}\Big),
\end{equation}
so $F\le\mathbb E F+R\sqrt{2\ln(1/\delta)/n}$ with probability at least $1-\delta$. By Jensen's inequality and the independence of the $\mathbf z_j$,
\begin{equation}
\mathbb E F\le\big(\mathbb E F^2\big)^{1/2}=\Big(\frac1n\mathbb E\|\mathbf z_1-\boldsymbol\mu\|_2^2\Big)^{1/2}=\Big(\frac1n\big(\mathbb E\|\mathbf z_1\|_2^2-\|\boldsymbol\mu\|_2^2\big)\Big)^{1/2}\le\frac{R}{\sqrt n}.
\end{equation}
Combining the two inequalities proves the claim.
\hfill$\square$

\textbf{\textit{Proof of Proposition~\ref{prop:risk}.}}\par\nopagebreak
Throughout, $\ell_{\mathrm{TV}}(\mathbf p,\mathbf q)=\frac12\|\mathbf p-\mathbf q\|_1$ denotes the total variation distance on $\Delta^{C-1}$, and $\bar\Phi^{B}=\frac1K\sum_{k=1}^{K}\Phi(\widetilde G_k)$ and $\bar\Phi^{T}=\frac1{N_T}\sum_{j=1}^{N_T}\Phi(G_j^t)$ denote the mean descriptors of the prototypes and of the target sample. Fix a prototype set $\mathcal B\subset\mathbb G$, a predictor $h\in\mathcal H$, and $\eta>0$, and choose $h'\in\mathcal H$ with $R_{\mathcal B}(h')+R_T(h')\le\gamma_{\mathcal B}+\eta$. Let $g(G)=\ell_{\mathrm{TV}}(h(G),h'(G))$ and let $P_{\mathcal B}^{\mathbb G}=\frac1K\sum_{k=1}^{K}\delta_{\widetilde G_k}$ denote the prototype distribution. We relate $h$ to the near-optimal joint predictor $h'$ on both distributions and then control how much their disagreement $g$ changes between the target graphs and the prototypes.

Since $\ell_{\mathrm{TV}}$ is a metric on $\Delta^{C-1}$, the inequality $\ell_{\mathrm{TV}}(h(G),\mathbf e_y)\le g(G)+\ell_{\mathrm{TV}}(h'(G),\mathbf e_y)$ holds for every labeled graph $(G,y)$, and taking the expectation over $P_T$ gives
\begin{equation}
R_T(h)\le\mathbb E_{P_T^{\mathbb G}}[g]+R_T(h').
\label{eq:app_step1}
\end{equation}
The same inequality on each prototype, $g(\widetilde G_k)\le\ell_{\mathrm{TV}}(h(\widetilde G_k),\widetilde{\mathbf y}_k)+\ell_{\mathrm{TV}}(\widetilde{\mathbf y}_k,h'(\widetilde G_k))$, averages to $\mathbb E_{P_{\mathcal B}^{\mathbb G}}[g]\le R_{\mathcal B}(h)+R_{\mathcal B}(h')$. Adding and subtracting $\mathbb E_{P_{\mathcal B}^{\mathbb G}}[g]$ in Eq.~\eqref{eq:app_step1} and using the choice of $h'$, we obtain
\begin{equation}
R_T(h)\le R_{\mathcal B}(h)+\gamma_{\mathcal B}+\eta+\big|\mathbb E_{P_T^{\mathbb G}}[g]-\mathbb E_{P_{\mathcal B}^{\mathbb G}}[g]\big| .
\label{eq:app_step2}
\end{equation}

The last term measures how differently the two predictors behave on the target graphs and on the prototypes, and condition (i) reduces it to a comparison of mean descriptors. Writing $g=\mathbf w^{\top}\Phi+c+r$ with $\|\mathbf M^{-1/2}\mathbf w\|_2\le\kappa$ and $\sup_{G\in\mathbb G}|r(G)|\le\varepsilon_\Phi$, and letting $\boldsymbol\mu_T=\mathbb E_{P_T^{\mathbb G}}[\Phi(G)]$, the constant $c$ cancels in the difference of expectations. Because the target graphs and the prototypes both lie in $\mathbb G$, the residual $r$ is bounded by $\varepsilon_\Phi$ on each distribution and contributes at most $2\varepsilon_\Phi$. Since $\lambda_{\mathrm{geo}},\lambda_{\mathrm{spec}}>0$, the matrix $\mathbf M^{1/2}$ is invertible, and the Cauchy--Schwarz inequality gives
\begin{equation}
\big|\mathbb E_{P_T^{\mathbb G}}[g]-\mathbb E_{P_{\mathcal B}^{\mathbb G}}[g]\big|
\le\big|(\mathbf M^{-1/2}\mathbf w)^{\top}\mathbf M^{1/2}(\boldsymbol\mu_T-\bar\Phi^{B})\big|+2\varepsilon_\Phi
\le\kappa\big\|\mathbf M^{1/2}(\boldsymbol\mu_T-\bar\Phi^{B})\big\|_2+2\varepsilon_\Phi .
\label{eq:app_step3}
\end{equation}
The population mean $\boldsymbol\mu_T$ is unknown, but the triangle inequality separates the distance into a part that the alignment objective controls and a part that depends only on the target sample. By Eq.~\eqref{eq:app_distance_form},
\begin{equation}
\big\|\mathbf M^{1/2}(\boldsymbol\mu_T-\bar\Phi^{B})\big\|_2
\le\big\|\mathbf M^{1/2}(\bar\Phi^{T}-\bar\Phi^{B})\big\|_2+\big\|\mathbf M^{1/2}(\boldsymbol\mu_T-\bar\Phi^{T})\big\|_2
=\sqrt{\mathcal L_{\mathrm{tgt}}(\mathcal B;\mathcal D_T)}+\Delta_T,
\label{eq:app_step4}
\end{equation}
where $\Delta_T=\|\mathbf M^{1/2}(\boldsymbol\mu_T-\bar\Phi^{T})\|_2$ involves neither $\mathcal B$ nor $h$. Under conditions (ii) and (iii), the vectors $\mathbf z_j=\mathbf M^{1/2}\Phi(G_j^t)$ are i.i.d.\ with $\|\mathbf z_j\|_2\le R$ almost surely and mean $\mathbf M^{1/2}\boldsymbol\mu_T$, so Lemma~\ref{lem:concentration} yields, with probability at least $1-\delta$,
\begin{equation}
\Delta_T\le\frac{R\big(1+\sqrt{2\ln(1/\delta)}\big)}{\sqrt{N_T}} .
\label{eq:app_step5}
\end{equation}
This event concerns the target sample alone, so Eqs.~\eqref{eq:app_step2}--\eqref{eq:app_step4} hold on it for every prototype set in $\mathbb G$ and every predictor, including those optimized on $\mathcal D_T$. Combining Eqs.~\eqref{eq:app_step2}--\eqref{eq:app_step5} and letting $\eta\to0$ proves the first bound.

The bound for any predictor on the frozen prototypes follows once its prototype risk is related to its cross-entropy loss. For $f_\phi\in\mathcal H$, write $\mathbf p_k=f_{\phi}(\widetilde G_k^{*})\in\Delta^{C-1}$. Pinsker's inequality~\cite{cover2006elements} gives $\ell_{\mathrm{TV}}(\mathbf p_k,\widetilde{\mathbf y}_k^{*})\le\sqrt{\frac12\mathrm{KL}(\widetilde{\mathbf y}_k^{*}\,\|\,\mathbf p_k)}$, and the cross-entropy decomposes as $\ell_{\mathrm{CE}}(\mathbf p_k,\widetilde{\mathbf y}_k^{*})=H(\widetilde{\mathbf y}_k^{*})+\mathrm{KL}(\widetilde{\mathbf y}_k^{*}\,\|\,\mathbf p_k)$ with a non-negative Shannon entropy $H(\widetilde{\mathbf y}_k^{*})$, so $\ell_{\mathrm{TV}}(\mathbf p_k,\widetilde{\mathbf y}_k^{*})\le\sqrt{\frac12\ell_{\mathrm{CE}}(\mathbf p_k,\widetilde{\mathbf y}_k^{*})}$. Averaging over $k$ and applying Jensen's inequality to the concave square root,
\begin{equation}
R_{\mathcal B^*}(f_{\phi})
\le\frac1K\sum_{k=1}^{K}\sqrt{\tfrac12\ell_{\mathrm{CE}}(\mathbf p_k,\widetilde{\mathbf y}_k^{*})}
\le\sqrt{\frac{1}{2K}\sum_{k=1}^{K}\ell_{\mathrm{CE}}(\mathbf p_k,\widetilde{\mathbf y}_k^{*})}
=\sqrt{\tfrac12\mathcal L_{\mathrm{inf}}(\phi;\mathcal B^*)} ,
\end{equation}
and substituting this inequality into the first bound with $\mathcal B=\mathcal B^*$ and $h=f_{\phi}$ proves the second. The argument uses no property of the training procedure, so the bound applies to the deployed predictor at any stage of training.
\hfill$\square$

The uniformity over $\mathcal B$ is what allows the prototypes to be optimized on the same target graphs that enter the bound. The only random quantity is $\Delta_T$, which compares the empirical target mean with the population target mean and does not depend on the prototypes, so no union bound over prototype sets is required and the estimation term does not grow with $K$. The prototype adjacencies are symmetric with entries in $[0,1]$ and zero diagonal by construction, so the requirement $\mathcal B\subset\mathbb G$ reduces to bounded prototype features, which we assume in the analysis and which can be enforced by projecting the features onto the $\ell_2$ ball of radius $x_{\max}$. The constant $R$ is finite on $\mathbb G$, as shown in Appendix~\ref{app:descriptors}.

\section{Properties of the Target Descriptors}
\label{app:descriptors}

This section verifies the differentiability and boundedness of the descriptor $\Phi=[\boldsymbol\psi_{\mathrm{geo}}^{\top},\Omega]^{\top}$, which underlie the gradient decomposition and the risk bound.

\textbf{Differentiability.} The prototype adjacency $\widetilde{\mathbf A}_k$ is a smooth function of $\mathbf Z_k$. The mean degree $\mu_d$ and the density $\rho$ are linear in $\mathbf A$, and the triangle intensity $\tau$ is a cubic polynomial in its entries. The degree standard deviation $\sigma_d$ is the square root of a quadratic function of $\mathbf A$ plus $\epsilon>0$, so it remains smooth even when all degrees coincide. For the Dirichlet energy, the self-loops ensure that every diagonal entry of $\widehat{\mathbf D}$ is at least one, so $\widehat{\mathbf D}^{-1/2}$, $\widehat{\mathbf L}$, and $\Omega$ are smooth in $(\mathbf A,\mathbf X)$. $\Phi(\widetilde G_k)$ is then differentiable with respect to $(\mathbf Z_k,\widetilde{\mathbf X}_k)$.

\textbf{Boundedness.} For a graph in $\mathbb G$ with $N\le N_{\max}$ nodes, every degree lies in $[0,N-1]$, so $\mu_d\le N_{\max}-1$, and Popoviciu's inequality gives $\sigma_d\le\big((N_{\max}-1)^2/4+\epsilon\big)^{1/2}$. The density satisfies $\rho\le1$ because $\mathbf 1^{\top}\mathbf A\mathbf 1\le N(N-1)$. Since the diagonal of $\mathbf A$ vanishes, $\operatorname{Tr}(\mathbf A^{3})$ sums over ordered triples of distinct nodes with summands in $[0,1]$, so $\tau\le(N_{\max}-1)(N_{\max}-2)/6$. The eigenvalues of $\widehat{\mathbf L}$ lie in $[0,2]$, so $0\le\Omega\le2\|\mathbf X\|_F^2\le2N_{\max}x_{\max}^2$. Every $G\in\mathbb G$ then satisfies
\begin{equation}
\big\|\mathbf M^{1/2}\Phi(G)\big\|_2^2\le\lambda_{\mathrm{geo}}\Big(\tfrac54(N_{\max}-1)^2+\epsilon+1+\tfrac1{36}(N_{\max}-1)^2(N_{\max}-2)^2\Big)+4\lambda_{\mathrm{spec}}N_{\max}^2x_{\max}^4,
\end{equation}
and the square root of the right-hand side is a valid choice of $R$ in Proposition~\ref{prop:risk}.

% Keep only if Section 3.3 still states the scaling and additivity of Omega.
\textbf{Scaling of the Dirichlet energy.} $\Omega(\mathbf A,c\mathbf X)=c^2\Omega(\mathbf A,\mathbf X)$ follows directly from the trace form. For a disjoint union, the adjacency, degree, and normalized Laplacian matrices are block diagonal, so $\Omega(G_1\sqcup G_2)=\Omega(G_1)+\Omega(G_2)$ and $\Omega(G^{\sqcup r})=r\,\Omega(G)$. The $r$-fold union leaves $\mu_d$, $\sigma_d$, and $\tau$ unchanged, so the geometric coordinates describe connectivity largely independently of graph size, whereas the energy coordinate scales with the number of nodes and the feature magnitude.

\begin{figure}[t]
    \centering
    \includegraphics[width=\linewidth]{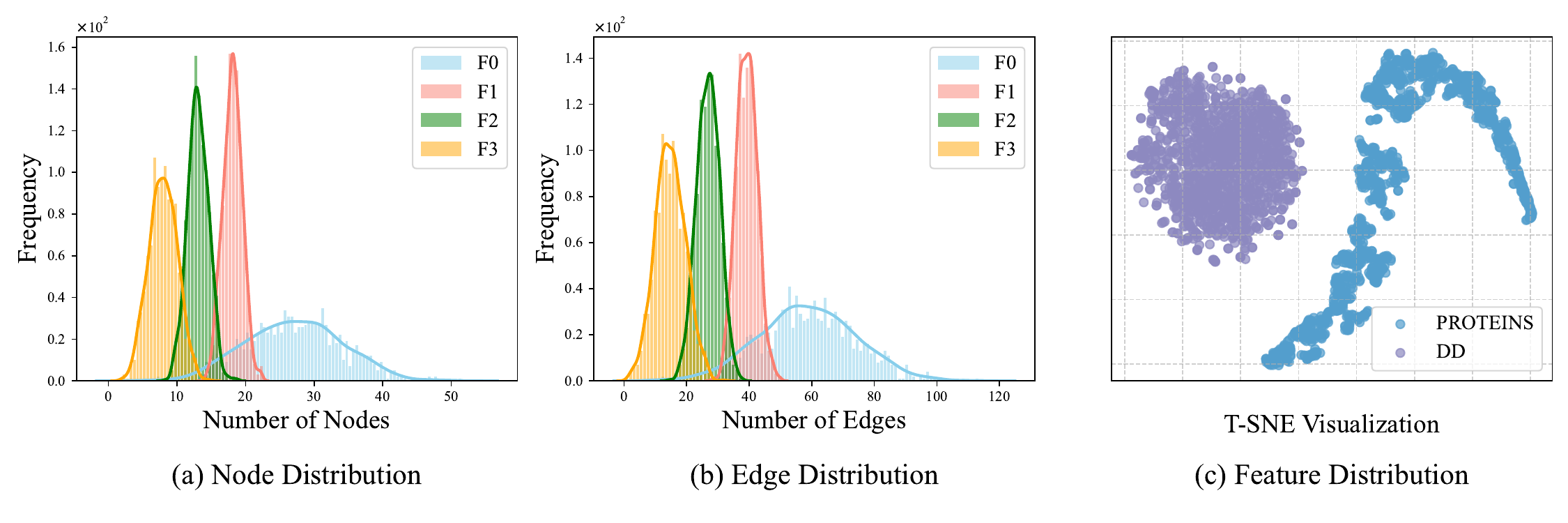}
    \caption{Visualization of domain shifts across different types. (a) Node distribution shift between sub-datasets of FRANKENSTEIN. (b) Edge distribution shift between sub-datasets of FRANKENSTEIN. (c) Feature distribution shift between PROTEINS and DD datasets.}
    \label{fig:shift}
    
\end{figure}

\section{Datasets}
\label{app:datasets}

We conduct extensive experiments on a variety of datasets. The statistics of the datasets are summarized in Table \ref{tab:dataset}. Visualization of domain shifts of these datasets is in Figure~\ref{fig:shift}. The detailed descriptions of these datasets are provided as follows:

(1) For structure-based domain shifts:

\begin{itemize}

\item \textbf{NCI1.} The NCI1 dataset~\cite{wale2008comparison} consists of 4,110 molecular graphs for binary graph classification. In each graph, nodes denote atoms and edges represent chemical bonds, capturing the molecular structure of compounds. Each molecule is labeled according to its ability to inhibit cancer cell growth. Following~\cite{yin2023coco}, we also partition the dataset into four subsets, denoted as N0, N1, N2, and N3, based on graph-level statistics, specifically node density and edge density.

\item \textbf{Mutagenicity.} 
The Mutagenicity dataset~\cite{kazius2005derivation} consists of 4,337 molecular graphs for binary classification. In each graph, nodes represent atoms and edges denote chemical bonds, characterizing the structural composition of chemical compounds. Each molecule is labeled according to its mutagenic effect. Following~\cite{yin2023coco}, we divide the dataset into four subsets, denoted as M0, M1, M2, and M3, based on graph-level statistics, specifically node density and edge density.

\item \textbf{FRANKENSTEIN.} The FRANKENSTEIN dataset~\cite{orsini2015graph} comprises 4,337 molecular graphs for graph classification. In each graph, nodes correspond to atoms and edges represent chemical bonds, encoding the structural properties of chemical compounds. Each molecule is labeled according to its biological activity. Following~\cite{yin2023coco}, we split the dataset into four subsets, denoted as F0, F1, F2, and F3, based on graph-level statistics, specifically node density and edge density.

\item \textbf{ogbg-molhiv.} The ogbg-molhiv dataset~\cite{hu2021ogblsc} consists of 41,127 molecular graphs for binary classification. In each graph, nodes denote atoms and edges represent chemical bonds, capturing the structural characteristics of chemical compounds. Each molecule is labeled according to whether it exhibits HIV inhibitory activity. Following~\cite{yin2023coco}, we split the dataset into four subsets, denoted as H0, H1, H2, and H3, based on graph-level statistics, specifically node density and edge density.

\item \textbf{MNIST.} The MNIST dataset~\cite{lecun2002gradient} contains 70,000 grayscale handwritten digit images (60,000 training and 10,000 test samples) spanning 10 classes. We transform each image into a graph, where nodes represent superpixels and edges encode spatial adjacency relationships. To introduce structural distribution shifts, we partition the resulting graphs into three subsets, denoted as S0, S1, and S2, based on their edge density.

\item \textbf{CIFAR-10.} The CIFAR-10 dataset~\cite{krizhevsky2009learning} consists of 60,000 color images (50,000 training and 10,000 test samples) across 10 object categories. We construct graph representations for each image, where nodes correspond to superpixels and edges encode spatial adjacency relationships. To induce structural distribution shifts, we partition the resulting graphs into three subsets, denoted as C0, C1, and C2, according to their edge density.

\end{itemize}

(2) For feature-based domain shifts:

\begin{itemize}

\item \textbf{PROTEINS.} The PROTEINS dataset~\citep{dobson2003distinguishing} comprises 1,113 protein graphs for graph classification. Each graph models a protein structure, where nodes represent amino acids with categorical attributes and edges indicate spatial or chemical proximity.

\item \textbf{DD.} The DD dataset consists of 1,178 protein graphs, each annotated with a binary label indicating whether the protein is an enzyme. In each graph, nodes represent amino acids and edges connect pairs of amino acids whose spatial distance is within 6~\AA{}. Compared to PROTEINS, graphs in DD are generally larger and denser, exhibiting higher structural complexity while sharing similar semantic labels.

\item \textbf{COX2.} The COX2 dataset comprises 467 molecular graphs for binary classification, while COX2\_MD contains 303 structurally modified variants derived from COX2. In both datasets, nodes denote atoms and edges represent chemical bonds, encoding molecular structures. COX2\_MD introduces systematic structural perturbations relative to COX2 to simulate distribution shifts while preserving the original label semantics.

\item \textbf{BZR.} The BZR dataset consists of 405 molecular graphs for binary classification, and BZR\_MD includes 306 structurally modified counterparts derived from BZR. In both cases, nodes correspond to atoms and edges denote chemical bonds. The BZR\_MD dataset applies controlled structural modifications to emulate domain shifts while maintaining consistent label semantics.

\end{itemize}

(3) For correlation shifts:

\begin{itemize}

\item \textbf{Spurious-Motif.} The Spurious-Motif dataset~\cite{wu2022discovering} is a synthetic benchmark designed to assess robustness under structural correlation shifts. Each graph is generated by attaching a label-determining motif (e.g., house or cycle) to a base graph (e.g., tree or ladder) that defines the environmental context. Node features are deliberately uninformative, requiring models to rely solely on topological cues. We construct two variants, each containing 3,000 graphs: a biased version (Spurious-Motif\_bias), where specific motifs are strongly correlated with particular base graphs to induce spurious structural shortcuts, and an unbiased version (Spurious-Motif), where motifs and base graphs are paired uniformly at random.

\end{itemize}

\subsubsection{Data Processing}
For datasets from TUDataset~\footnote{https://chrsmrrs.github.io/datasets/} (e.g., Mutagenicity and FRANKENSTEIN), we employ the standard preprocessing and normalization pipeline provided by PyTorch Geometric~\footnote{https://pyg.org/}. For MNIST and CIFAR-10, we reconstruct graph structures using $k$-nearest-neighbor (KNN) graphs based on node spatial coordinates. To induce controlled structural shifts, we regenerate edges using domain-specific neighborhood sizes, with $k\in\{5,7,9\}$ for MNIST and $k\in\{6,9,12\}$ for CIFAR-10~\cite{dwivedi2023benchmarking}. For datasets from the Open Graph Benchmark (OGB)~\footnote{https://ogb.stanford.edu/}, such as ogbg-molhiv, we follow the official OGB preprocessing and normalization protocols. For the synthetic Spurious-Motif dataset, we adhere to the generation and processing procedures described in~\cite{wu2022discovering}, converting the generated graphs into standard PyTorch Geometric data objects. Node features are initialized as constant one-dimensional vectors to eliminate attribute information, thereby enforcing reliance on structural topology. The dataset is organized into biased and unbiased variants, serving as source and target domains, respectively.

\section{Baselines}
\label{app:baselines}

We compare \method{} against a comprehensive set of competitive methods: 

(1) \textbf{General graph neural networks.} We compare \method{} with three general graph neural networks: 

\begin{itemize}
    \item \textbf{GCN}: GCN~\cite{kipf2016semi} introduces a graph convolutional architecture that propagates and transforms node features through normalized neighborhood aggregation, effectively capturing local structural dependencies and learning expressive representations.

    \item \textbf{GIN}: GIN~\cite{xu2018how} introduces an expressive graph neural architecture that aggregates neighboring node features through sum pooling and multilayer perceptrons, matching the discriminative power of the Weisfeiler--Lehman test and learning effective representations.

    \item \textbf{GMT}: GMT~\cite{baek2021accurate} introduces a Graph Multiset Transformer that formulates graph pooling as multiset encoding and employs multi-head attention to capture structural dependencies among nodes, producing expressive and permutation-invariant representations for downstream graph prediction tasks.
\end{itemize}

(2) \textbf{Graph distillation methods.} We compare \method{} with six graph distillation methods: 

\begin{itemize}
\item \textbf{G-CRD}: G-CRD~\cite{joshi2022representation} presents a graph contrastive representation distillation method, where a student aligns with a teacher by maximizing cross-model consistency across augmented graph views, effectively transferring structural and semantic knowledge to improve generalization and efficiency.

\item \textbf{MuGSI}: MuGSI~\cite{yao2024mugsi} introduces a multi-granularity structural distillation framework for graph classification, where a student GNN aligns with a teacher across node-, subgraph-, and graph-level representations, facilitating rich structural knowledge transfer to improve predictive performance and efficiency.

\item \textbf{TGS}: TGS~\cite{wu2024teacher} proposes a teacher-free graph knowledge distillation framework with dual self-distillation, where a single GNN collaboratively distills knowledge across multiple views and hierarchical representations, enabling mutual refinement without an external teacher to improve robustness and classification performance.

\item \textbf{LAD-GNN}: LAD-GNN~\cite{hong2024label} proposes a label-attentive distillation framework for GNN-based graph classification, where the student emphasizes class-relevant structural patterns under teacher guidance, leveraging label-aware attention to selectively transfer discriminative knowledge and improve classification accuracy.

\item \textbf{AdaGMLP}: AdaGMLP~\cite{lu2024adagmlp} proposes an AdaBoost-based GNN-to-MLP knowledge distillation framework, where multiple boosted MLP students are sequentially trained to mimic a GNN teacher, reweighting hard samples to progressively refine decision boundaries and achieve strong performance without graph-structured inference.

\item \textbf{ClustGDD}: ClustGDD~\cite{lai2025simple} synthesizes a compact set of training graphs through graph clustering, providing a dataset distillation baseline that does not use target-domain information.
\end{itemize}
(3) \textbf{Graph Domain Adaptation methods.} We compare \method{} with six graph domain adaptation methods:

\begin{itemize}
        \item \textbf{SGDA}:  SGDA~\citep{qiao2023semi} addresses graph domain adaptation with partially labeled source graphs, introducing adaptive shift parameters trained adversarially to align source and target embedding distributions, together with centroid-aware pseudo-labeling to alleviate label scarcity. 
        \item \textbf{StruRW}: StruRW \citep{liu2023structural} proposes structural re-weighting for graph domain adaptation, where graph components are adaptively reweighted according to their transferability, emphasizing domain-invariant structures while suppressing domain-specific noise to enhance cross-domain adaptation.
        \item \textbf{A2GNN}: A2GNN \citep{liu2024rethinking} rethinks message propagation for unsupervised graph domain adaptation, identifying propagation-induced domain bias and introducing a refined propagation mechanism that decouples structural smoothing from domain-specific signals to learn more transferable graph representations.
        \item \textbf{PA-BOTH}~\cite{liu2024pairwise} is the full variant of pairwise alignment for graph domain adaptation, which decomposes graph structure shift into conditional structure shift and label shift, recalibrating the influence among neighboring nodes with edge weights to handle the former and reweighting the classification loss with label weights to handle the latter.

\item \textbf{GAA}: GAA~\citep{fang2025benefits} investigates attribute-driven graph domain adaptation, demonstrating that leveraging informative node attributes to guide cross-domain alignment improves representation transferability, mitigates structural bias, and enhances target-domain performance under distribution shift.
        \item \textbf{TDSS}: TDSS~\citep{chen2025smoothness} shows that enforcing representation smoothness is critical for unsupervised graph domain adaptation, introducing a simple regularization strategy that promotes local consistency across graphs to enhance transferability and improve target-domain performance.
\end{itemize}

\section{Implementation Details}
\label{app:implementation}

We implement \method{} in PyTorch and run all methods on NVIDIA A100 GPUs. Both the proxy GNN and the deployed predictor are three-layer GINs with hidden dimension 128 and dropout 0.2. All optimization uses Adam with weight decay $10^{-12}$. Prototype distillation runs $T_{\mathrm{out}}=20$ outer iterations, each with $T_{\mathrm{in}}=30$ inner updates, using $\eta_{\mathrm{in}}=10^{-3}$, $\eta_{\mathrm{out}}=10^{-4}$, and gradient clipping at 1.0 for the outer updates. The deployed predictor is trained on the frozen prototypes for $T_{\mathrm{inf}}=20$ updates with $\eta_{\mathrm{inf}}=10^{-3}$. By default, we set $K=30$, $\lambda_{\mathrm{geo}}=0.7$, and $\lambda_{\mathrm{spec}}=0.5$. Baselines use the same optimizer, learning rate, hidden dimension, and GNN depth where applicable. We report accuracy on the TUDataset, image, and Spurious-Motif benchmarks and ROC-AUC on ogbg-molhiv, as the mean and standard deviation over five independent runs.

\section{Complexity Analysis}
\label{app:complexity}

Here we analyze the computational complexity of the proposed \method{}, which is dominated by the bilevel prototype distillation. Let $K$ denote the number of prototypes, $n_b$ the number of nodes per prototype, $N_S$ the number of source graphs, $d_h$ the hidden dimension, $d$ the input feature dimension, and $L$ the number of GNN layers. For a graph with $n$ nodes and $e$ edges, one forward pass of the GNN takes $\mathcal O\big(L(ed_h+nd_h^2)\big)$ time. Since the prototype adjacencies are dense, the inner optimization over $T_{\mathrm{in}}$ steps takes $\mathcal O\big(T_{\mathrm{in}}KL(n_b^2d_h+n_bd_h^2)\big)$ time, and back-propagating through the unrolled updates preserves this order up to a constant factor. The semantic loss on the source graphs takes $\mathcal O\big(L(|\mathcal E_S|d_h+|\mathcal V_S|d_h^2)\big)$ time, where $|\mathcal V_S|$ and $|\mathcal E_S|$ denote the total numbers of nodes and edges in the source graphs, and the prototype statistics introduce an additional term of $\mathcal O\big(K(n_b^3+n_b^2d)\big)$ for the triangle intensity and the Dirichlet energy. As a result, the overall computational complexity of \method{} over $T_{\mathrm{out}}$ outer iterations is $\mathcal O\big(T_{\mathrm{out}}\big(T_{\mathrm{in}}KL(n_b^2d_h+n_bd_h^2)+L(|\mathcal E_S|d_h+|\mathcal V_S|d_h^2)+K(n_b^3+n_b^2d)\big)\big)$. The target statistics are computed only once, and the deployed predictor is trained on the $K$ prototypes, so both add a negligible one-time cost.

\section{More Results}
\label{app:more_results}

\subsection{More experimental results}
\label{app:full_results}

We extend the comparison to all ordered source--target pairs under node density and edge density shifts on FRANKENSTEIN, ogbg-molhiv, Mutagenicity, and NCI1, covering the reverse directions of the transfers in the main text. As shown in Tables~\ref{tab:frankenstein_node}--\ref{tab:nci1_idx}, \method{} achieves the highest average score in every table, and its advantage is most pronounced on FRANKENSTEIN and ogbg-molhiv. The baseline groups behave differently across benchmarks. On FRANKENSTEIN, distillation and GDA methods improve over source-only GNNs, whereas on the other benchmarks model-level adaptation often fails to improve over source-only training, indicating negative transfer when the adaptation mechanism does not match the shift. \method{} adapts the training graphs while the semantic objective anchors them to source supervision, which limits such negative transfer. Reversing a transfer changes both the labeled source and the target statistics, so the two directions form distinct adaptation problems, and \method{} leads in both directions for most pairs. The remaining cases mainly arise when the target domain contains large and heterogeneous graphs, whose broad descriptor distributions make mean matching less informative. The partition rule also matters, since node-based partitions couple differences in graph size with feature--topology statistics, whereas edge-based partitions mainly alter connectivity. On ogbg-molhiv, which is evaluated by ROC-AUC under a strongly imbalanced label distribution, the gains of \method{} carry over from accuracy to a ranking-based metric.

\begin{figure}[t]
    \centering
    \begin{subfigure}[b]{0.24\linewidth}
        \centering
        \includegraphics[width=\linewidth]{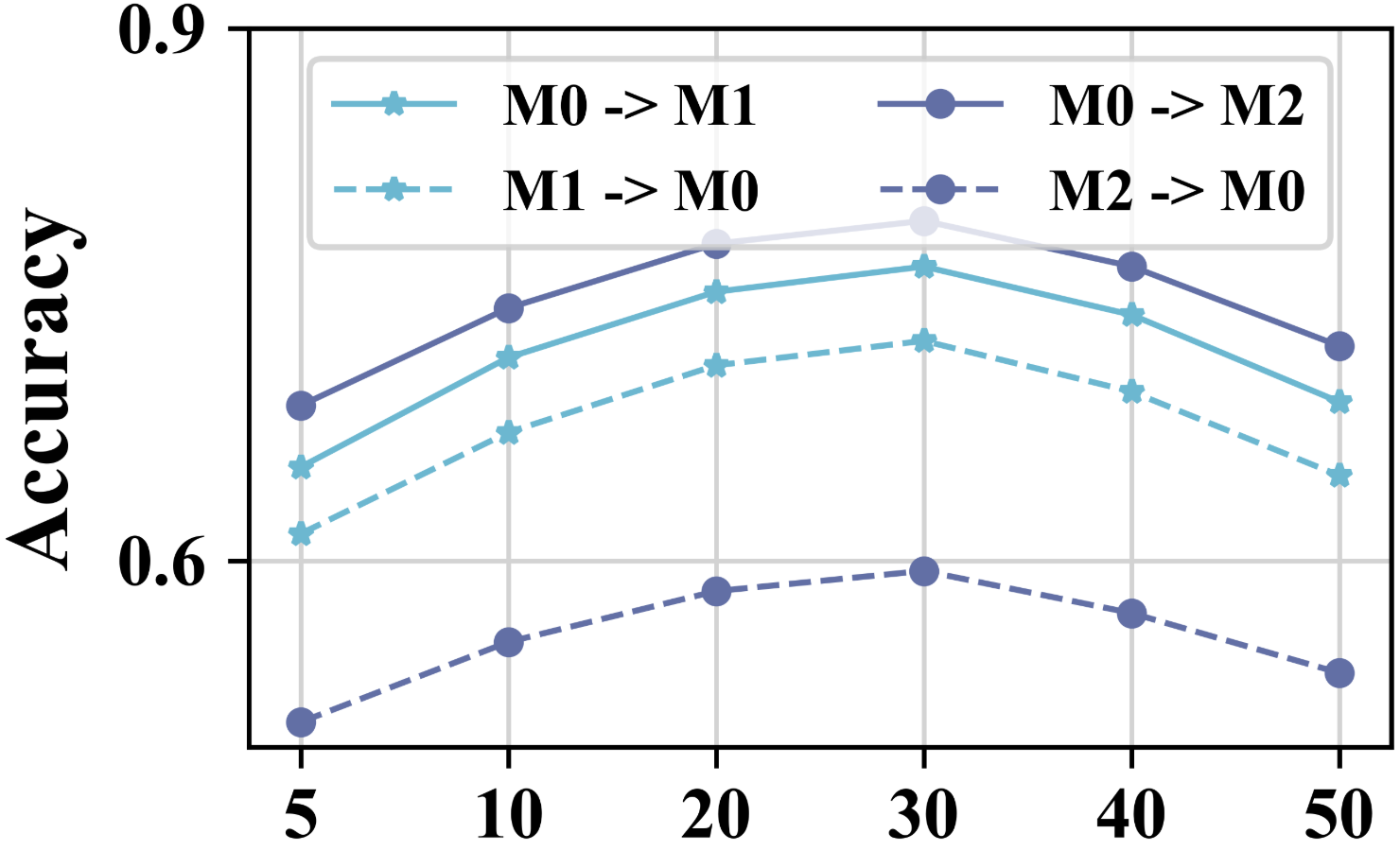}
        \caption{Mutagenicity}
    \end{subfigure}\hfill
    \begin{subfigure}[b]{0.24\linewidth}
        \centering
        \includegraphics[width=\linewidth]{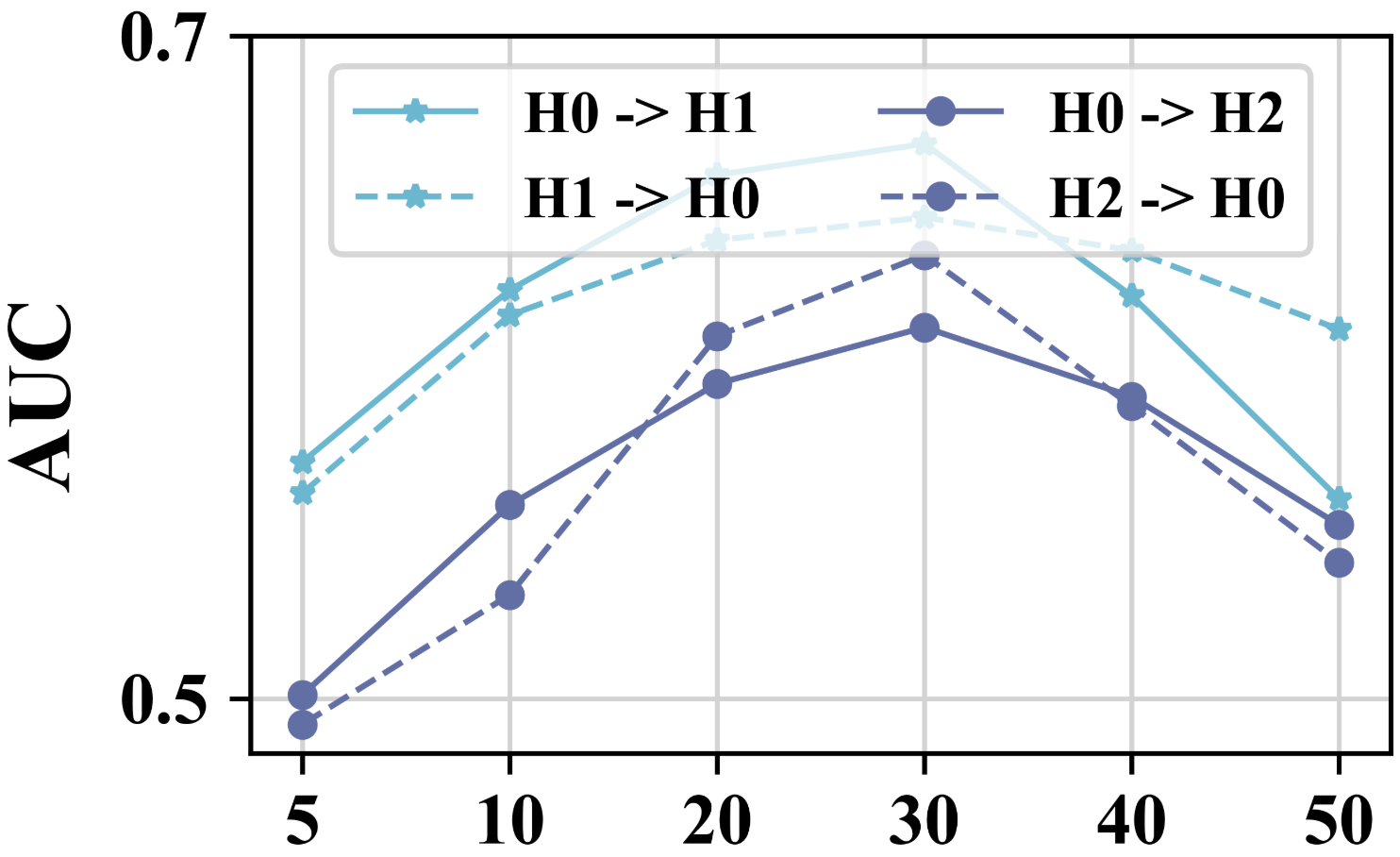}
        \caption{ogbg-molhiv}
    \end{subfigure}\hfill
    \begin{subfigure}[b]{0.24\linewidth}
        \centering
        \includegraphics[width=\linewidth]{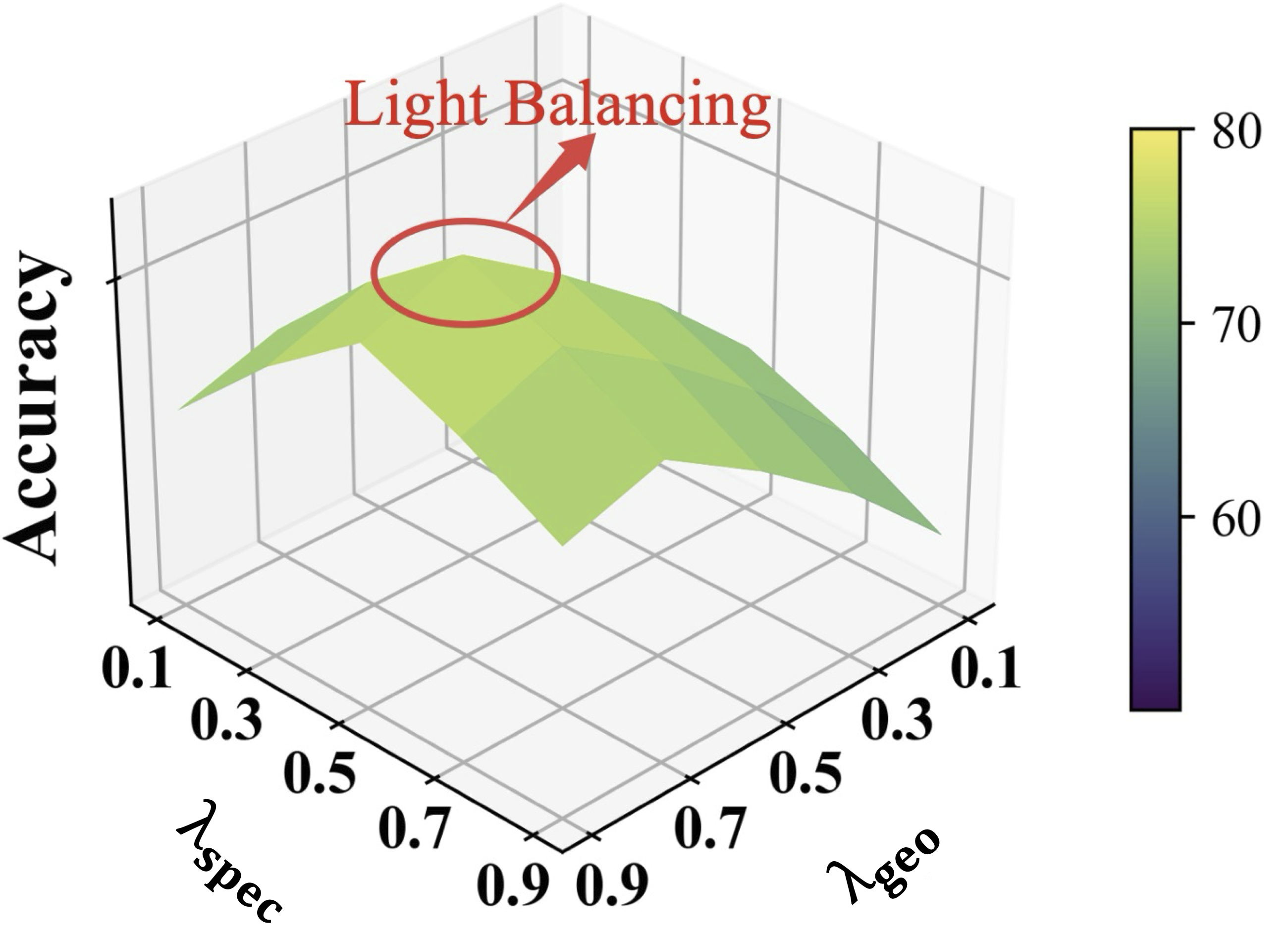}
        \caption{Mutagenicity}
    \end{subfigure}\hfill
    \begin{subfigure}[b]{0.24\linewidth}
        \centering
        \includegraphics[width=\linewidth]{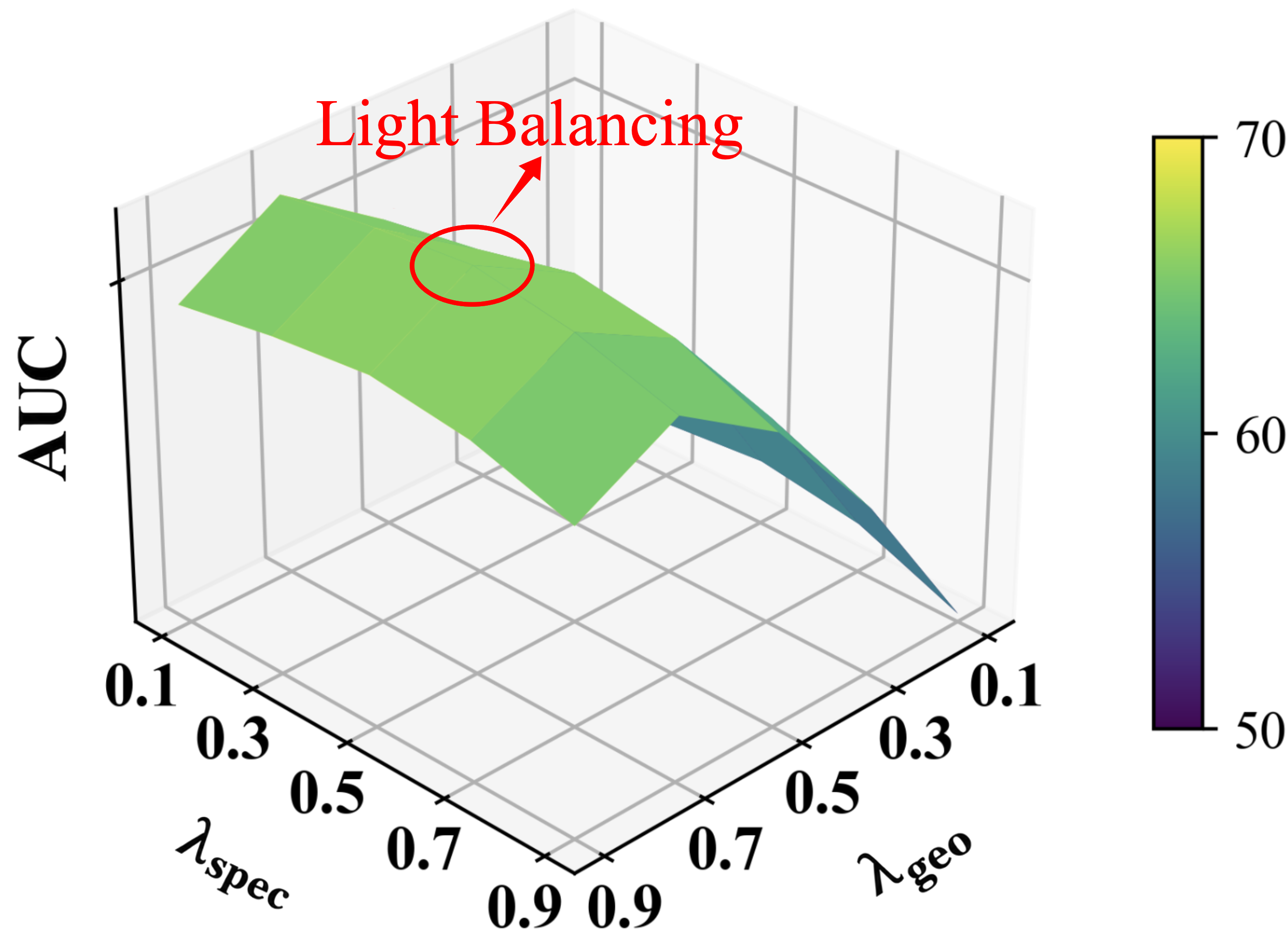}
        \caption{ogbg-molhiv}
    \end{subfigure}
    \caption{Sensitivity to the number of prototypes $K$ (a,b) and to the alignment weights $\lambda_{\mathrm{geo}}$ and $\lambda_{\mathrm{spec}}$ (c,d) on Mutagenicity and ogbg-molhiv. The axes $\lambda_1$ and $\lambda_2$ in (d) denote $\lambda_{\mathrm{geo}}$ and $\lambda_{\mathrm{spec}}$.}
    \label{fig:sensitivity_appendix}
    \Description{Four plots in one row. The first two show accuracy and ROC-AUC against the number of prototypes for Mutagenicity and ogbg-molhiv, with a peak at thirty prototypes. The last two show surfaces of performance over the two alignment weights for the same datasets.}
\end{figure}

\subsection{More Ablation Study}
\label{app:ablation_study}

We further report the ablation results on NCI1, ogbg-molhiv, and Mutagenicity under edge density shift in Tables~\ref{tab:ablation_nci1}--\ref{tab:ablation_mutag}, and make the following observations.
(1) The full model outperforms all variants in every transfer, and the variants follow the same average ranking as on FRANKENSTEIN, with w/o SE performing worst, followed by w/o GE, w/o SP, and w/o TG. Removing the semantic objective causes by far the largest degradation, since the target statistics carry no class information and source supervision is indispensable for keeping the prototypes discriminative.
(2) Geometric alignment generally contributes more than energy alignment, which matches the nature of edge density shifts that act mainly on connectivity. This ordering holds in every transfer on NCI1 and Mutagenicity, while on ogbg-molhiv energy alignment becomes more important in some transfers, suggesting that feature--topology coupling carries more domain-specific information on this larger molecular benchmark. The two alignment terms are complementary, and removing either one consistently degrades performance.
(3) The deployed predictor improves over the terminal proxy in every transfer, with gains of similar magnitude across benchmarks and metrics. This consistent gap reflects a general property of the bilevel optimization, in which the proxy is always fitted to the prototypes before the last outer update, whereas the deployed predictor is trained on the final prototypes.

\subsection{More Sensitivity Study}
\label{app:sensitive_analysis}

We further conduct the sensitivity analysis on Mutagenicity and ogbg-molhiv over the same ranges of $K$, $\lambda_{\mathrm{geo}}$, and $\lambda_{\mathrm{spec}}$. As shown in Fig.~\ref{fig:sensitivity_appendix}(a,b), performance first rises as $K$ grows toward a moderate value and then declines, consistent with the trend on FRANKENSTEIN and NCI1. The curves on ogbg-molhiv are steeper for small $K$, since a very small prototype set cannot represent both classes of this imbalanced benchmark together with the target statistics. As shown in Fig.~\ref{fig:sensitivity_appendix}(c,d), performance peaks at moderate alignment weights on both datasets, and ogbg-molhiv remains stable over a broader region of the grid than Mutagenicity. In both cases, the weakest settings lie on the boundary of the grid, where either the alignment terms or the semantic objective dominates the outer objective, confirming that \method{} benefits from balancing target conditioning with source supervision and does not require precise tuning within a moderate range.

\subsection{Visualization of Learned Prototypes}
\label{app:vis}

Fig.~\ref{fig:vis_graph} shows the complete set of prototypes learned by \method{}, arranged approximately in increasing order of Dirichlet energy. The lowest-energy prototypes form compact polycyclic structures with fused rings. Prototypes with intermediate energy mostly connect cyclic substructures through bridges or chains, yielding diverse shapes such as a three-armed star (j), a ladder-like structure (k), and an H-shaped structure (n). The highest-energy prototypes are dominated by elongated chains with terminal rings, branched trees, and rings carrying peripheral branches, as in (v) and (IV). Prototypes with similar energy can still differ in topology, as in (k) and (l), because the energy depends on node features as well as structure. Overall, the prototypes cover a broad range of molecule-like motifs, confirming that the collection-level alignment objectives do not collapse the prototypes into a single template.

\begin{figure}[t]
\centering
\includegraphics[width=0.88\linewidth]{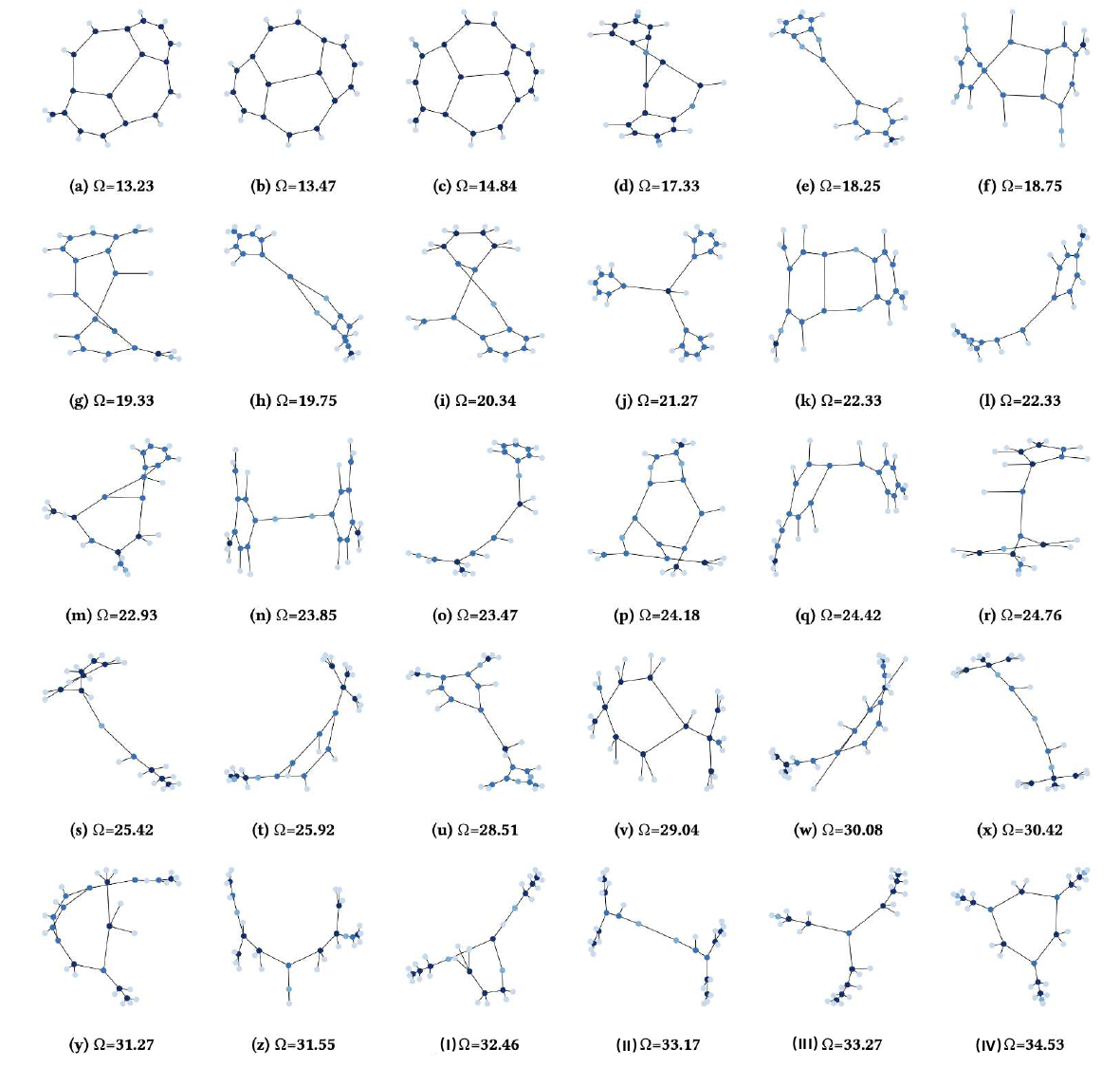}
\caption{All prototypes learned by \method{}, labeled with their Dirichlet energy $\Omega$.}
\label{fig:vis_graph}
\Description{A gallery of the learned graph prototypes. Each graph is annotated with its Dirichlet energy, with examples ranging from sparse chains to compact cyclic structures.}
\end{figure}

\clearpage

\begin{table*}[ht]
\small
\centering
\caption{Ablation results on NCI1 (source $\rightarrow$ target). \textbf{Bold} indicates the highest reported mean.}
\resizebox{1.0\textwidth}{!}{
\begin{tabular}{l|c|c|c|c|c|c|c|c|c|c|c|c}
\toprule
\textbf{Methods}
& N0$\rightarrow$N1 & N1$\rightarrow$N0
& N0$\rightarrow$N2 & N2$\rightarrow$N0
& N0$\rightarrow$N3 & N3$\rightarrow$N0
& N1$\rightarrow$N2 & N2$\rightarrow$N1
& N1$\rightarrow$N3 & N3$\rightarrow$N1
& N2$\rightarrow$N3 & N3$\rightarrow$N2 \\
\midrule 
\method{} w/o SE  & 50.3 & 63.8 & 44.9 & 55.1 & 30.5 & 51.2 & 49.0 & 41.6 & 33.1 & 46.2 & 60.5 & 48.4 \\
\method{} w/o SP & 57.8 & 71.1 & 51.5 & 62.4 & 36.8 & 59.3 & 56.2 & 48.4 & 38.6 & 51.2 & 69.1 & 54.0 \\
\method{} w/o GE  & 56.1 & 67.4 & 50.2 & 61.8 & 34.2 & 57.5 & 52.8 & 47.7 & 36.5 & 49.0 & 65.7 & 51.9 \\
\method{} w/o TG & 61.2 & 72.5 & 52.8 & 64.9 & 37.1 & 62.0 & 58.6 & 51.1 & 40.2 & 54.5 & 70.8 & 56.4 \\
\midrule
\method{} & \textbf{63.6} & \textbf{75.8} & \textbf{55.6} & \textbf{67.9} & \textbf{39.2} & \textbf{64.3} & \textbf{60.4} & \textbf{53.3} & \textbf{42.5} & \textbf{56.7} & \textbf{73.9} & \textbf{58.2} \\
\bottomrule
\end{tabular}
}
\label{tab:ablation_nci1}
\end{table*}

\begin{table*}[ht]
\small
\centering
\caption{Ablation results on ogbg-molhiv (source $\rightarrow$ target). \textbf{Bold} indicates the highest reported mean.}
\resizebox{1.0\textwidth}{!}{
\begin{tabular}{l|c|c|c|c|c|c|c|c|c|c|c|c}
\toprule
\textbf{Methods}
& H0$\rightarrow$H1 & H1$\rightarrow$H0

& H0$\rightarrow$H2 & H2$\rightarrow$H0
& H0$\rightarrow$H3 & H3$\rightarrow$H0
& H1$\rightarrow$H2 & H2$\rightarrow$H1
& H1$\rightarrow$H3 & H3$\rightarrow$H1
& H2$\rightarrow$H3 & H3$\rightarrow$H2 \\
\midrule 
\method{} w/o SE  & 53.4 & 52.1 & 49.3 & 51.8 & 53.2 & 50.5 & 51.7 & 52.6 & 48.9 & 51.2 & 54.3 & 52.8 \\
\method{} w/o SP & 62.1 & 59.8 & 51.2 & 59.2 & 60.4 & 57.1 & 58.6 & 62.8 & 57.8 & 56.4 & 61.2 & 62.3 \\
\method{} w/o GE  & 58.6 & 56.4 & 54.8 & 55.7 & 57.1 & 54.8 & 56.2 & 55.4 & 52.3 & 58.1 & 58.7 & 57.9 \\
\method{} w/o TG & 64.2 & 61.7 & 59.1 & 61.5 & 63.8 & 60.2 & 62.1 & 58.1 & 60.5 & 61.0 & 64.5 & 60.1 \\
\midrule
\method{} & \textbf{66.8} & \textbf{64.5} & \textbf{61.2} & \textbf{63.4} & \textbf{65.1} & \textbf{62.7} & \textbf{63.9} & \textbf{66.5} & \textbf{62.1} & \textbf{62.8} & \textbf{66.2} & \textbf{65.4} \\
\bottomrule
\end{tabular}
}
\label{tab:ablation_molhiv}
\end{table*}

\begin{table*}[ht]
\small
\centering
\caption{Ablation results on Mutagenicity (source $\rightarrow$ target). \textbf{Bold} indicates the highest reported mean.}
\resizebox{1.0\textwidth}{!}{
\begin{tabular}{l|c|c|c|c|c|c|c|c|c|c|c|c}
\toprule
\textbf{Methods}
& M0$\rightarrow$M1 & M1$\rightarrow$M0
& M0$\rightarrow$M2 & M2$\rightarrow$M0
& M0$\rightarrow$M3 & M3$\rightarrow$M0
& M1$\rightarrow$M2 & M2$\rightarrow$M1
& M1$\rightarrow$M3 & M3$\rightarrow$M1
& M2$\rightarrow$M3 & M3$\rightarrow$M2 \\
\midrule 
\method{} w/o SE  & 64.2 & 61.5 & 65.8 & 50.3 & 53.4 & 54.2 & 65.1 & 69.5 & 48.2 & 62.8 & 50.9 & 48.7 \\
\method{} w/o SP & 71.5 & 68.1 & 74.2 & 55.6 & 56.5 & 60.1 & 72.5 & 76.7 & 53.4 & 69.2 & 56.2 & 53.8 \\
\method{} w/o GE  & 69.5 & 66.4 & 71.6 & 53.8 & 54.2 & 57.1 & 69.8 & 74.2 & 51.5 & 65.7 & 54.1 & 51.6 \\
\method{} w/o TG & 72.4 & 70.6 & 73.9 & 57.1 & 58.6 & 62.3 & 74.1 & 78.5 & 54.9 & 71.4 & 59.4 & 55.5 \\
\midrule
\method{} & \textbf{76.5} & \textbf{72.3} & \textbf{79.1} & \textbf{59.4} & \textbf{60.8} & \textbf{64.5} & \textbf{76.4} & \textbf{81.2} & \textbf{56.8} & \textbf{73.6} & \textbf{60.3} & \textbf{57.2} \\
\bottomrule
\end{tabular}
}
\label{tab:ablation_mutag}
\end{table*}

\begin{table*}[t]
\centering
\caption{The classification results (in \%) on the FRANKENSTEIN dataset under node density domain shift (source $\rightarrow$ target).  F0, F1, F2, and F3 denote the sub-datasets partitioned with node density. \textbf{Bold} results indicate the best performance.}
\resizebox{1.0\textwidth}{!}{
\begin{tabular}{l|c|c|c|c|c|c|c|c|c|c|c|c}
\toprule
\textbf{Methods} &F0$\rightarrow$F1 &F1$\rightarrow$F0 &F0$\rightarrow$F2 &F2$\rightarrow$F0 &F0$\rightarrow$F3 &F3$\rightarrow$F0 &F1$\rightarrow$F2 &F2$\rightarrow$F1 &F1$\rightarrow$F3 &F3$\rightarrow$F1 &F2$\rightarrow$F3 &F3$\rightarrow$F2 \\ 
\midrule
GCN & 32.7\scriptsize{$\pm$2.0} & 43.2\scriptsize{$\pm$2.1} & 42.2\scriptsize{$\pm$1.7} & 43.2\scriptsize{$\pm$1.5} & 54.5\scriptsize{$\pm$1.8} & 43.2\scriptsize{$\pm$1.9} & 42.2\scriptsize{$\pm$2.0} & 32.7\scriptsize{$\pm$1.5} & 56.5\scriptsize{$\pm$1.2} & 32.7\scriptsize{$\pm$1.8} & 55.5\scriptsize{$\pm$1.8} & 42.2\scriptsize{$\pm$2.0}\\
GIN & 39.7\scriptsize{$\pm$1.8} & 51.4\scriptsize{$\pm$1.7} & 54.7\scriptsize{$\pm$1.3} & 51.4\scriptsize{$\pm$1.7} & 58.2\scriptsize{$\pm$1.9} & 41.8\scriptsize{$\pm$2.0} & 54.7\scriptsize{$\pm$2.3} & 47.1\scriptsize{$\pm$1.6} & 52.1\scriptsize{$\pm$1.2} & 34.3\scriptsize{$\pm$1.5} & 50.8\scriptsize{$\pm$1.7} & 43.8\scriptsize{$\pm$2.0}\\
GMT & 47.3\scriptsize{$\pm$2.2} & 56.8\scriptsize{$\pm$1.9} & 57.8\scriptsize{$\pm$1.4} & 56.8\scriptsize{$\pm$1.2} & 39.5\scriptsize{$\pm$2.3} & 44.1\scriptsize{$\pm$1.2} & 57.8\scriptsize{$\pm$2.0} & 58.3\scriptsize{$\pm$1.5} & 39.5\scriptsize{$\pm$1.6} & 39.3\scriptsize{$\pm$1.4} & 39.5\scriptsize{$\pm$2.0} & 44.1\scriptsize{$\pm$1.6} \\

\midrule
G-CRD & 57.6{\scriptsize$\pm 1.7$} & 60.0{\scriptsize$\pm 2.2$} & 56.2{\scriptsize$\pm 0.9$} & 49.3{\scriptsize$\pm 1.1$} & 55.5{\scriptsize$\pm 1.9$} & 55.8{\scriptsize$\pm 2.9$} & 59.4{\scriptsize$\pm 0.5$} & 57.1{\scriptsize$\pm 2.2$} & 59.9{\scriptsize$\pm 3.3$} & 60.4{\scriptsize$\pm 2.2$} & 57.6{\scriptsize$\pm 1.1$} & 69.1{\scriptsize$\pm 2.7$} \\
\bottomrule
\end{tabular}
}
\label{tab:frankenstein_node}
\end{table*}

\begin{table*}[t]
\centering
\caption{The classification results (in \%) on the FRANKENSTEIN dataset under edge density domain shift (source $\rightarrow$ target).  F0, F1, F2, and F3 denote the sub-datasets partitioned with edge density. \textbf{Bold} results indicate the best performance.}
\resizebox{1.0\textwidth}{!}{
\begin{tabular}{l|c|c|c|c|c|c|c|c|c|c|c|c}
\toprule
\textbf{Methods}
& F0$\rightarrow$F1
& F1$\rightarrow$F0
& F0$\rightarrow$F2
& F2$\rightarrow$F0
& F0$\rightarrow$F3
& F3$\rightarrow$F0
& F1$\rightarrow$F2
& F2$\rightarrow$F1
& F1$\rightarrow$F3
& F3$\rightarrow$F1
& F2$\rightarrow$F3
& F3$\rightarrow$F2 \\

\midrule
GCN & 33.5\scriptsize{$\pm$2.1} & 40.9\scriptsize{$\pm$1.5} & 44.6\scriptsize{$\pm$1.3} & 40.1\scriptsize{$\pm$1.6} & 57.4\scriptsize{$\pm$2.0} & 43.1\scriptsize{$\pm$1.3} & 44.6\scriptsize{$\pm$2.0} & 33.5\scriptsize{$\pm$1.2} & 56.4\scriptsize{$\pm$1.1} & 36.6\scriptsize{$\pm$1.9} & 57.4\scriptsize{$\pm$2.0} & 47.8\scriptsize{$\pm$1.7}\\
GIN & 40.1\scriptsize{$\pm$1.2} & 52.4\scriptsize{$\pm$1.8} & 52.2\scriptsize{$\pm$2.2} & 52.7\scriptsize{$\pm$1.8} & 58.3\scriptsize{$\pm$1.5} & 42.8\scriptsize{$\pm$2.0} & 53.2\scriptsize{$\pm$2.3} & 47.1\scriptsize{$\pm$1.6} & 52.1\scriptsize{$\pm$2.2} & 39.1\scriptsize{$\pm$2.0} & 49.9\scriptsize{$\pm$1.8} & 43.5\scriptsize{$\pm$1.9}\\
GMT & 51.5\scriptsize{$\pm$1.0} 
& \textbf{57.0}\scriptsize{$\pm$1.3} 
& 55.4\scriptsize{$\pm$1.6} 
& \textbf{59.0}\scriptsize{$\pm$1.2} 
& 39.6\scriptsize{$\pm$1.6} 
& 44.1\scriptsize{$\pm$2.0} 
& 55.4\scriptsize{$\pm$2.1} 
& 66.5\scriptsize{$\pm$1.6} 
& 39.6\scriptsize{$\pm$1.9} 
& 36.5\scriptsize{$\pm$1.5} 
& 39.6\scriptsize{$\pm$1.4} 
& 42.8\scriptsize{$\pm$1.4}\\
\midrule

G-CRD
& 71.0{\scriptsize$\pm 0.8$}
& 55.2{\scriptsize$\pm 2.9$}
& 57.6{\scriptsize$\pm 2.6$}
& 49.3{\scriptsize$\pm 2.9$}
& 40.1{\scriptsize$\pm 2.5$}
& 52.6{\scriptsize$\pm 1.5$}
& 57.6{\scriptsize$\pm 2.1$}
& 70.5{\scriptsize$\pm 1.7$}
& 47.5{\scriptsize$\pm 1.8$}
& 71.0{\scriptsize$\pm 3.0$}
& 59.4{\scriptsize$\pm 1.9$}
& 57.6{\scriptsize$\pm 1.1$} \\
\bottomrule
\end{tabular}
}
\label{tab:frankenstein_edge}
\end{table*}

\begin{table*}[t]
\centering
\caption{The ROC-AUC results (in \%) on the ogbg-molhiv dataset under node density domain shift (source $\rightarrow$ target).  H0, H1, H2, and H3 denote the sub-datasets partitioned with node density. \textbf{Bold} results indicate the best performance.}
\resizebox{1.0\textwidth}{!}{
\begin{tabular}{l|c|c|c|c|c|c|c|c|c|c|c|c}
\toprule
\textbf{Methods}    & H0$\rightarrow$H1 & H1$\rightarrow$H0 & H0$\rightarrow$H2 & H2$\rightarrow$H0 & H0$\rightarrow$H3 & H3$\rightarrow$H0 & H1$\rightarrow$H2 & H2$\rightarrow$H1 & H1$\rightarrow$H3 & H3$\rightarrow$H1 & H2$\rightarrow$H3 & H3$\rightarrow$H2 \\ 
\midrule
GCN         & 57.8\scriptsize{$\pm$1.9} & 57.4\scriptsize{$\pm$2.1} & 56.5\scriptsize{$\pm$1.3} & 54.2\scriptsize{$\pm$1.2} & 54.5\scriptsize{$\pm$1.7} & 51.9\scriptsize{$\pm$1.7} & 61.8\scriptsize{$\pm$2.0} & 55.2\scriptsize{$\pm$1.3} & 49.8\scriptsize{$\pm$1.4} & \textbf{56.6}\scriptsize{$\pm$1.3} & 61.5\scriptsize{$\pm$1.1} & 56.8\scriptsize{$\pm$1.7}\\
GIN         & 59.1\scriptsize{$\pm$1.6} & 53.3\scriptsize{$\pm$1.6} & 50.0\scriptsize{$\pm$1.7} & 54.8\scriptsize{$\pm$1.8} & 56.8\scriptsize{$\pm$1.6} & 50.7\scriptsize{$\pm$1.4} & 55.2\scriptsize{$\pm$1.3} & 58.7\scriptsize{$\pm$1.4} & 57.0\scriptsize{$\pm$1.8} & 49.6\scriptsize{$\pm$1.7} & 53.1\scriptsize{$\pm$1.3} & 55.3\scriptsize{$\pm$1.6}\\
GMT         & 51.1\scriptsize{$\pm$1.6} & 57.2\scriptsize{$\pm$1.7} & 50.6\scriptsize{$\pm$1.2} & 55.9\scriptsize{$\pm$1.9} & 58.4\scriptsize{$\pm$1.4} & 53.4\scriptsize{$\pm$1.7} & 56.1\scriptsize{$\pm$1.6} & 57.6\scriptsize{$\pm$1.3} & 60.9\scriptsize{$\pm$1.7} & 54.5\scriptsize{$\pm$2.0} & 55.4\scriptsize{$\pm$1.6} & 51.2\scriptsize{$\pm$1.8}\\
\midrule
G-CRD & 55.5{\scriptsize$\pm 1.2$} & 59.8{\scriptsize$\pm 2.9$} & 57.6{\scriptsize$\pm 3.2$} & 47.4{\scriptsize$\pm 0.6$} & 60.3{\scriptsize$\pm 1.4$} & 42.8{\scriptsize$\pm 3.4$} & 48.1{\scriptsize$\pm 2.1$} & 52.0{\scriptsize$\pm 2.4$} & 45.5{\scriptsize$\pm 0.5$} & 46.4{\scriptsize$\pm 2.0$} & 53.4{\scriptsize$\pm 3.5$} & 56.6{\scriptsize$\pm 1.6$} \\
\bottomrule
\end{tabular}
}
\label{tab:molhiv_node_results}
\end{table*}

\begin{table*}[t]
\centering
\caption{The ROC-AUC results (in \%) on the ogbg-molhiv dataset under edge density domain shift (source $\rightarrow$ target).  H0, H1, H2, and H3 denote the sub-datasets partitioned with edge density. \textbf{Bold} results indicate the best performance.}
\resizebox{1.0\textwidth}{!}{
\begin{tabular}{l|c|c|c|c|c|c|c|c|c|c|c|c}
\toprule
\textbf{Methods}    & H0$\rightarrow$H1 & H1$\rightarrow$H0 & H0$\rightarrow$H2 & H2$\rightarrow$H0 & H0$\rightarrow$H3 & H3$\rightarrow$H0 & H1$\rightarrow$H2 & H2$\rightarrow$H1 & H1$\rightarrow$H3 & H3$\rightarrow$H1 & H2$\rightarrow$H3 & H3$\rightarrow$H2 \\ 
\midrule
GCN         & 57.5\scriptsize{$\pm$1.8} & 58.0\scriptsize{$\pm$1.5} & 56.1\scriptsize{$\pm$2.1} & 54.3\scriptsize{$\pm$1.2} & 54.1\scriptsize{$\pm$1.4} & 50.9\scriptsize{$\pm$1.4} & 61.2\scriptsize{$\pm$1.6} & 54.8\scriptsize{$\pm$1.9} & 54.9\scriptsize{$\pm$2.1} & 56.6\scriptsize{$\pm$1.4} & 65.4\scriptsize{$\pm$2.0} & 58.1\scriptsize{$\pm$1.6}\\
GIN         & 58.1\scriptsize{$\pm$1.1} & 55.3\scriptsize{$\pm$1.2} & 49.7\scriptsize{$\pm$1.5} & 55.2\scriptsize{$\pm$2.2} & 55.5\scriptsize{$\pm$1.9} & 50.5\scriptsize{$\pm$2.0} & 55.1\scriptsize{$\pm$1.3} & 58.8\scriptsize{$\pm$1.2} & 56.6\scriptsize{$\pm$1.6} & 50.0\scriptsize{$\pm$1.1} & 56.8\scriptsize{$\pm$1.8} & 56.2\scriptsize{$\pm$1.6} \\
GMT         & 50.7\scriptsize{$\pm$1.0} & 57.5\scriptsize{$\pm$1.7} & 51.1\scriptsize{$\pm$2.1} & 58.8\scriptsize{$\pm$1.5} & 59.8\scriptsize{$\pm$1.4} & 52.0\scriptsize{$\pm$1.4} & 57.0\scriptsize{$\pm$1.8} & 55.4\scriptsize{$\pm$1.0} & 54.0\scriptsize{$\pm$1.7} & 54.9\scriptsize{$\pm$1.0} & 58.1\scriptsize{$\pm$1.3} & 52.0\scriptsize{$\pm$1.3} \\
\midrule
G-CRD & 61.1{\scriptsize$\pm 3.1$} & 52.6{\scriptsize$\pm 0.6$} & 58.9{\scriptsize$\pm 2.0$} & 51.3{\scriptsize$\pm 0.7$} & 48.5{\scriptsize$\pm 3.5$} & 33.2{\scriptsize$\pm 1.2$} & 61.5{\scriptsize$\pm 1.6$} & 37.3{\scriptsize$\pm 1.1$} & 40.9{\scriptsize$\pm 0.8$} & 53.1{\scriptsize$\pm 1.1$} & 54.6{\scriptsize$\pm 1.4$} & 54.6{\scriptsize$\pm 2.4$} \\
\bottomrule
\end{tabular}
}
\label{tab:molhiv_edge}
\end{table*}

\begin{table*}[t]
\centering
\caption{The classification results (in \%) on the Mutagenicity dataset under node density domain shift (source $\rightarrow$ target). M0, M1, M2, and M3 denote the sub-datasets partitioned with node density. \textbf{Bold} results indicate the best performance.}
\resizebox{1.0\textwidth}{!}{
\begin{tabular}{l|c|c|c|c|c|c|c|c|c|c|c|c}
\toprule
\textbf{Methods} &M0$\rightarrow$M1 &M1$\rightarrow$M0 &M0$\rightarrow$M2 &M2$\rightarrow$M0 &M0$\rightarrow$M3 &M3$\rightarrow$M0 &M1$\rightarrow$M2 &M2$\rightarrow$M1 &M1$\rightarrow$M3 &M3$\rightarrow$M1 &M2$\rightarrow$M3 &M3$\rightarrow$M2 \\ 
\midrule

GCN 
& 64.1\scriptsize{$\pm$1.4} 
& 62.6\scriptsize{$\pm$2.3} 
& 65.5\scriptsize{$\pm$2.0} 
& 62.3\scriptsize{$\pm$1.1} 
& 56.9\scriptsize{$\pm$2.1} 
& 53.7\scriptsize{$\pm$1.8} 
& 66.1\scriptsize{$\pm$1.6} 
& 70.8\scriptsize{$\pm$2.1} 
& 47.7\scriptsize{$\pm$1.9} 
& 49.2\scriptsize{$\pm$2.1} 
& 52.1\scriptsize{$\pm$1.7} 
& 54.4\scriptsize{$\pm$2.1}\\

GIN 
& \textbf{66.5}\scriptsize{$\pm$2.1} 
& 53.2\scriptsize{$\pm$2.3} 
& 52.0\scriptsize{$\pm$1.7} 
& 62.6\scriptsize{$\pm$2.3} 
& 53.7\scriptsize{$\pm$1.7} 
& 48.6\scriptsize{$\pm$2.3} 
& 71.5\scriptsize{$\pm$2.5} 
& 70.0\scriptsize{$\pm$2.1} 
& 52.2\scriptsize{$\pm$2.3} 
& 64.3\scriptsize{$\pm$2.1} 
& 50.9\scriptsize{$\pm$1.4} 
& 52.4\scriptsize{$\pm$1.5} \\

GMT 
& 65.7\scriptsize{$\pm$1.8} 
& \textbf{67.1}\scriptsize{$\pm$1.9} 
& 62.1\scriptsize{$\pm$2.1} 
& \textbf{66.4}\scriptsize{$\pm$1.2} 
& 59.0\scriptsize{$\pm$2.0} 
& 58.0\scriptsize{$\pm$2.7} 
& 69.9\scriptsize{$\pm$1.4} 
& 68.4\scriptsize{$\pm$2.3} 
& 60.4\scriptsize{$\pm$1.3} 
& 48.2\scriptsize{$\pm$2.0} 
& 59.1\scriptsize{$\pm$1.3} 
& 60.6\scriptsize{$\pm$1.7} \\

\midrule

G-CRD 
& 60.4{\scriptsize$\pm 0.5$} 
& 55.9{\scriptsize$\pm 3.2$} 
& 55.3{\scriptsize$\pm 1.1$} 
& 54.8{\scriptsize$\pm 1.4$} 
& 59.6{\scriptsize$\pm 3.2$} 
& 56.8{\scriptsize$\pm 1.3$} 
& 51.2{\scriptsize$\pm 2.2$} 
& 63.6{\scriptsize$\pm 2.9$} 
& 62.2{\scriptsize$\pm 1.7$} 
& 59.4{\scriptsize$\pm 2.7$} 
& 63.1{\scriptsize$\pm 1.0$} 
& 72.3{\scriptsize$\pm 3.1$} \\
\bottomrule
\end{tabular}}
\label{tab:mutag_node}
\end{table*}
\begin{table*}[t]
\centering
\caption{The classification results (in \%) on the Mutagenicity dataset under edge density domain shift (source $\rightarrow$ target). M0, M1, M2, and M3 denote the sub-datasets partitioned with edge density. \textbf{Bold} results indicate the best performance.}
\resizebox{1.0\textwidth}{!}{
\begin{tabular}{l|c|c|c|c|c|c|c|c|c|c|c|c}
\toprule
\textbf{Methods} 
&M0$\rightarrow$M1 
&M1$\rightarrow$M0 
&M0$\rightarrow$M2 
&M2$\rightarrow$M0 
&M0$\rightarrow$M3 
&M3$\rightarrow$M0 
&M1$\rightarrow$M2 
&M2$\rightarrow$M1 
&M1$\rightarrow$M3 
&M3$\rightarrow$M1 
&M2$\rightarrow$M3 
&M3$\rightarrow$M2 \\
\midrule

GCN 
& 66.3\scriptsize{$\pm$1.7} 
& 65.1\scriptsize{$\pm$2.2} 
& 63.6\scriptsize{$\pm$1.4} 
& 67.6\scriptsize{$\pm$1.6} 
& 56.0\scriptsize{$\pm$1.4} 
& 49.6\scriptsize{$\pm$1.9} 
& 66.1\scriptsize{$\pm$2.0} 
& 71.8\scriptsize{$\pm$1.8} 
& 48.0\scriptsize{$\pm$2.3} 
& 47.8\scriptsize{$\pm$1.7} 
& 51.6\scriptsize{$\pm$1.8} 
& 56.6\scriptsize{$\pm$2.2} \\

GIN 
& 67.1\scriptsize{$\pm$1.7} 
& 56.1\scriptsize{$\pm$1.5} 
& 54.2\scriptsize{$\pm$2.6} 
& 65.6\scriptsize{$\pm$1.6} 
& 55.4\scriptsize{$\pm$1.9} 
& 50.8\scriptsize{$\pm$2.1} 
& 67.9\scriptsize{$\pm$2.1} 
& 69.4\scriptsize{$\pm$2.0} 
& 53.4\scriptsize{$\pm$1.6} 
& 66.7\scriptsize{$\pm$2.0} 
& 50.4\scriptsize{$\pm$1.1} 
& 55.4\scriptsize{$\pm$1.2} \\

GMT 
& 67.9\scriptsize{$\pm$1.3} 
& 63.9\scriptsize{$\pm$1.8} 
& 61.5\scriptsize{$\pm$1.8} 
& \textbf{68.5}\scriptsize{$\pm$1.5} 
& 58.2\scriptsize{$\pm$2.4} 
& 54.3\scriptsize{$\pm$1.8} 
& 66.7\scriptsize{$\pm$1.9} 
& 76.1\scriptsize{$\pm$1.7} 
& \textbf{58.8}\scriptsize{$\pm$1.5} 
& 44.5\scriptsize{$\pm$1.5} 
& 59.0\scriptsize{$\pm$1.4} 
& \textbf{63.7}\scriptsize{$\pm$1.7} \\

\midrule

G-CRD
& 73.7{\scriptsize$\pm 0.6$}
& 68.7{\scriptsize$\pm 3.0$}
& 55.3{\scriptsize$\pm 1.0$}
& 55.6{\scriptsize$\pm 1.4$}
& 58.4{\scriptsize$\pm 2.5$}
& 53.0{\scriptsize$\pm 0.8$}
& 67.3{\scriptsize$\pm 2.5$}
& 62.2{\scriptsize$\pm 3.2$}
& 50.2{\scriptsize$\pm 2.6$}
& 70.5{\scriptsize$\pm 1.8$}
& 57.5{\scriptsize$\pm 2.8$}
& 54.0{\scriptsize$\pm 0.8$} \\
\bottomrule
\end{tabular}
}
\label{tab:mutag_edge}
\end{table*}
\begin{table*}[t]
\centering
\caption{The classification results (in \%) on the NCI1 dataset under node density domain shift (source $\rightarrow$ target).  N0, N1, N2, and N3 denote the sub-datasets partitioned with node density. \textbf{Bold} results indicate the best performance.}
\resizebox{1.0\textwidth}{!}{
\begin{tabular}{l|c|c|c|c|c|c|c|c|c|c|c|c}
\toprule
\textbf{Methods} &N0$\rightarrow$N1 &N1$\rightarrow$N0 &N0$\rightarrow$N2 &N2$\rightarrow$N0 &N0$\rightarrow$N3 &N3$\rightarrow$N0 &N1$\rightarrow$N2 &N2$\rightarrow$N1 &N1$\rightarrow$N3 &N3$\rightarrow$N1 &N2$\rightarrow$N3 &N3$\rightarrow$N2 \\ 
\midrule
GCN & 54.0\scriptsize{$\pm$2.0} & \textbf{66.5}\scriptsize{$\pm$1.9} & 47.3\scriptsize{$\pm$2.2} & 37.4\scriptsize{$\pm$1.8} & 38.4\scriptsize{$\pm$2.3} & 27.6\scriptsize{$\pm$1.5} & 59.0\scriptsize{$\pm$1.7} & 56.9\scriptsize{$\pm$2.0} & 56.5\scriptsize{$\pm$1.2} & 51.0\scriptsize{$\pm$1.7} & 60.5\scriptsize{$\pm$1.9} & 57.4\scriptsize{$\pm$1.6}\\
GIN & 57.0\scriptsize{$\pm$1.9} & 38.7\scriptsize{$\pm$2.3} & 45.7\scriptsize{$\pm$1.8} & 34.4\scriptsize{$\pm$1.9} & 43.6\scriptsize{$\pm$2.1} & 27.0\scriptsize{$\pm$2.0} & 54.7\scriptsize{$\pm$1.7} & 52.8\scriptsize{$\pm$2.1} & \textbf{59.0}\scriptsize{$\pm$1.8} & 49.1\scriptsize{$\pm$2.5} & 59.4\scriptsize{$\pm$2.0} & 57.4\scriptsize{$\pm$1.3}\\
GMT & 53.0\scriptsize{$\pm$1.9} & 63.7\scriptsize{$\pm$1.6} & 48.9\scriptsize{$\pm$2.2} & 32.8\scriptsize{$\pm$2.0} & 41.5\scriptsize{$\pm$1.7} & 27.8\scriptsize{$\pm$1.6} & 56.4\scriptsize{$\pm$1.9} & \textbf{57.2}\scriptsize{$\pm$1.7} & 54.1\scriptsize{$\pm$1.9} & 49.8\scriptsize{$\pm$1.7} & 61.6\scriptsize{$\pm$2.1} & 58.5\scriptsize{$\pm$1.1}\\

\midrule
G-CRD & 56.3{\scriptsize$\pm 3.0$} & 50.5{\scriptsize$\pm 1.1$} & 46.6{\scriptsize$\pm 2.6$} & 34.0{\scriptsize$\pm 2.5$} & \textbf{51.5}{\scriptsize$\pm 3.5$} & 51.5{\scriptsize$\pm 1.9$} & 31.1{\scriptsize$\pm 2.7$} & 55.3{\scriptsize$\pm 2.9$} & 46.6{\scriptsize$\pm 1.3$} & \textbf{64.1}{\scriptsize$\pm 1.0$} & 58.3{\scriptsize$\pm 2.6$} & 58.3{\scriptsize$\pm 3.2$} \\
\bottomrule
\end{tabular}
}
\label{tab:nci1_node}
\end{table*}

\begin{table*}[t]
\centering
\caption{The classification results (in \%) on the NCI1 dataset under edge density domain shift (source $\rightarrow$ target).  N0, N1, N2, and N3 denote the sub-datasets partitioned with edge density. \textbf{Bold} results indicate the best performance.}
\resizebox{1.0\textwidth}{!}{
\begin{tabular}{l|c|c|c|c|c|c|c|c|c|c|c|c}
\toprule
\textbf{Methods} &N0$\rightarrow$N1 &N1$\rightarrow$N0 &N0$\rightarrow$N2 &N2$\rightarrow$N0 &N0$\rightarrow$N3 &N3$\rightarrow$N0 &N1$\rightarrow$N2 &N2$\rightarrow$N1 &N1$\rightarrow$N3 &N3$\rightarrow$N1 &N2$\rightarrow$N3 &N3$\rightarrow$N2 \\ 
\midrule
GCN & 53.7\scriptsize{$\pm$2.1} & 68.2\scriptsize{$\pm$1.9} & 50.6\scriptsize{$\pm$1.9} & 45.8\scriptsize{$\pm$1.7} & 35.5\scriptsize{$\pm$2.4} & 28.2\scriptsize{$\pm$2.2} & 55.9\scriptsize{$\pm$1.7} & 56.8\scriptsize{$\pm$1.9} & 55.7\scriptsize{$\pm$1.5} & 53.2\scriptsize{$\pm$2.1} & 60.4\scriptsize{$\pm$1.8} & 54.8\scriptsize{$\pm$0.9} \\
GIN & 51.4\scriptsize{$\pm$1.7} & 47.5\scriptsize{$\pm$2.1} & 50.4\scriptsize{$\pm$1.6} & 32.7\scriptsize{$\pm$1.6} & \textbf{47.0}\scriptsize{$\pm$2.1} & 29.4\scriptsize{$\pm$2.3} & 52.4\scriptsize{$\pm$1.9} & 55.7\scriptsize{$\pm$2.0} & \textbf{64.3}\scriptsize{$\pm$1.8} & 51.1\scriptsize{$\pm$1.2} & 60.3\scriptsize{$\pm$1.3} & 53.5\scriptsize{$\pm$1.6}\\
GMT & 53.0\scriptsize{$\pm$1.9} & 58.3\scriptsize{$\pm$1.7} & 49.0\scriptsize{$\pm$1.9} & 34.3\scriptsize{$\pm$1.6} & 39.7\scriptsize{$\pm$1.5} & 27.5\scriptsize{$\pm$2.4} & 55.8\scriptsize{$\pm$1.6} & \textbf{60.4}\scriptsize{$\pm$2.1} & 51.3\scriptsize{$\pm$1.0} & 51.5\scriptsize{$\pm$1.6} & 60.0\scriptsize{$\pm$2.6} & 54.5\scriptsize{$\pm$1.4}\\

\midrule
G-CRD & 50.5{\scriptsize$\pm 3.3$} & 52.4{\scriptsize$\pm 3.0$} & 53.0{\scriptsize$\pm 0.6$} & 65.0{\scriptsize$\pm 1.4$} & 36.5{\scriptsize$\pm 2.6$} & 31.1{\scriptsize$\pm 3.5$} & 52.4{\scriptsize$\pm 1.6$} & 49.5{\scriptsize$\pm 2.3$} & 38.9{\scriptsize$\pm 1.0$} & 54.0{\scriptsize$\pm 1.2$} & 45.6{\scriptsize$\pm 1.5$} & 54.4{\scriptsize$\pm 2.0$} \\
\bottomrule
\end{tabular}
}
\label{tab:nci1_idx}
\end{table*}

\end{document}